\documentclass{article} %

\usepackage[nonatbib, final]{neurips_2022}

\usepackage[utf8]{inputenc} %
\usepackage[T1]{fontenc}    %
\usepackage{algorithm}
\usepackage{algpseudocode}
\usepackage{hyperref}       %
\usepackage{url}            %
\usepackage{booktabs}       %
\usepackage{amsfonts}       %
\usepackage{amsmath}       %
\usepackage{amssymb}  %
\usepackage{amsthm}
\usepackage{euscript}
\usepackage{mathrsfs}
\usepackage{nicefrac}       %
\usepackage{microtype}      %
\usepackage{wrapfig}
\usepackage{xcolor}         %
\usepackage{graphicx}
\usepackage{subcaption}

\usepackage{custom}

\hypersetup{citecolor=cyan, colorlinks=true, linkcolor=blue!75!black}

\newtheorem{lem}{Lemma} 

\newtheorem{Definition}{Definition}
\newtheorem{Theorem}{Theorem} 

\newtheorem{proposition}{Proposition}
\newtheorem{claim}{Claim}

\makeatletter
\renewcommand{\paragraph}{%
  \@startsection{paragraph}{4}%
  {\z@}{0.05ex \@plus 1ex \@minus .2ex}{-1em}%
  {\normalfont\normalsize\bfseries}%
}
\makeatother

\newcommand{\cca}{\textsc{cca}}
\newcommand{\svcca}{\textsc{svcca}}
\newcommand{\pwcca}{\textsc{pwcca}}
\newcommand{\procrustes}{\textsc{procrustes}}
\newcommand{\cka}{\textsc{cka}}

\newcommand{\gulp}{\textsc{gulp}}

\title{\gulp: a prediction-based metric between representations}

\author{
Enric Boix-Adser\`{a} \\
MIT\\ %
\texttt{eboix@mit.edu} \\ 
\And
Hannah Lawrence \\
MIT\\ %
\texttt{hanlaw@mit.edu} \\ 
\And
George Stepaniants \\
MIT\\ %
\texttt{gstepan@mit.edu} \\
\And
Philippe Rigollet \\
MIT\\ %
\texttt{rigollet@math.mit.edu} \\

}

\begin{document}

\maketitle

\begin{abstract}
Comparing the representations learned by different neural networks has recently emerged as a key tool to understand various architectures and ultimately optimize them. In this work, we introduce \gulp{}, a family of   distance measures between representations that is explicitly motivated by  downstream predictive tasks. By construction, \gulp{} provides uniform control over the difference in prediction performance between two representations, with respect to regularized linear prediction tasks. Moreover, it satisfies several desirable structural properties, such as the triangle inequality and invariance under orthogonal transformations, and thus lends itself to data embedding and visualization. 
We extensively evaluate \gulp{} relative to other methods, and demonstrate that it correctly differentiates between architecture families, converges over the course of training, and captures generalization performance on downstream linear tasks.

\end{abstract}

\section{Introduction}

The spectacular success of deep neural networks (DNN) witnessed over the past decade has been largely attributed to their ability to generate good representations of the data~\cite{Bengiorep} . But \emph{what makes a representation good?} Answering this question is a necessary step towards a principled theory of DNN design. This fundamental question  calls for a \emph{metric over representations} as a basic primitive. Indeed, embedding representations into a metric space enables comparison, modifications and ultimately optimization of DNN architectures~\cite{luoarchiopt}; see Figure~\ref{fig:teaser}.

    In light of the practical impact of a meaningful metric over representations, this question has recently garnered significant attention, leading to a myriad of propositions %
such as \cca{}, \cka{}, and \procrustes{}. Their relative pros and cons are currently the subject of a lively debate~\cite{ding2021grounding,davari2022inadequacy} whose resolution calls for a theoretically grounded notion of metric. 

\begin{figure}[t] %
  \centering
    \includegraphics[width=\textwidth]{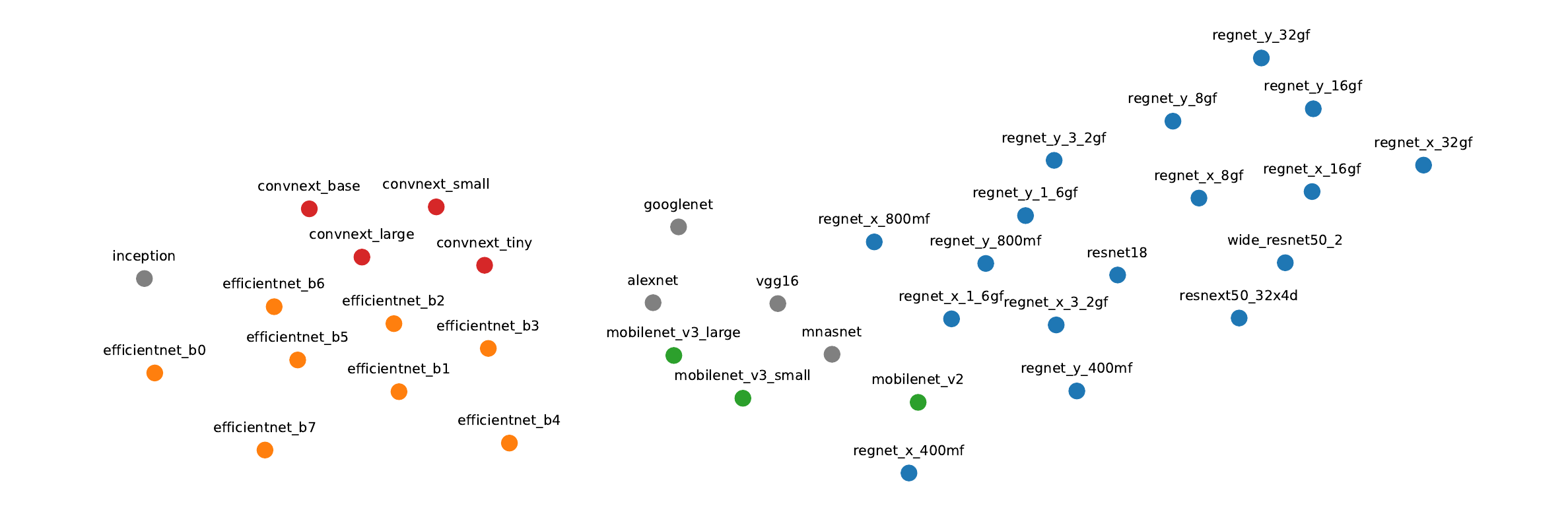}
        \caption{t-SNE embedding of various pretrained DNN representations of the ImageNet~\cite{krizhevsky2012imagenet} dataset with \gulp{} distance ($\lambda=10^{-2}$), colored by architecture type (gray denotes architectures that do not belong to a family). The embedding shows a good clustering of various architectures (ResNets, EfficientNets etc.), indicating that \gulp{} captures intrinsic aspects of representations shared within an architecture family.}
    \label{fig:teaser}
\end{figure}

\paragraph{Our contributions.} In this work, we %
define a new family of metrics\footnote{More specifically, we define pseudo-metrics rather than metrics. However, these can be readily turned into a metric using metric identification. This amounts to allowing equivalence classes of representations.}, called \gulp{}\footnote{GULP is Uniform Linear Probing.}, over the space of representations. 
Our construction rests on a functional notion of what makes two representations similar: namely, that two representations are similar if and only if they are equally useful as inputs to downstream, linear transfer learning tasks. This idea is partially inspired by feature-based transfer learning, in which simple models adapt pretrained representations, such as Inceptionv3 \cite{szegedy2016rethinking}, CLIP \cite{radford2021learning}, and ELMo \cite{peters2018deep}, for specific tasks %
\cite{RazavianSharif2014}; indeed, this is a key use for pretrained representations. Moreover, our application of \emph{linear} transfer learning is reminiscent of \emph{linear probes}, which were introduced by \cite{AlaBen18} as a tool to compare internal layers of a DNN in terms of prediction accuracy. Linear probes play a central role in the literature on hidden representations. They have been used not only to study the information captured by hidden representations \cite{ravichander2020probing}, but also to themselves define desiderata of distances between representations \cite{ding2021grounding}. However, previous applications of linear probing required hand-selecting the task on which prediction accuracy is measured, %
whereas our \gulp{} distance provides a uniform bound over \emph{all} norm-bounded tasks.

We establish various theoretical properties of the \gulp{} pseudo-metric, including the triangle inequality (Thm~\ref{thm:gulp_dist}), sample complexity (Thm~\ref{thm:plug-in-concentration}), and vanishing cases. In particular, we show that akin to the \procrustes{} pseudo-metric,  \gulp{} is invariant under orthogonal transformations (Thm~\ref{thm:invariance}) and vanishes precisely when the two representations are related by an orthogonal transformation (Thm~\ref{thm:gulp_dist}).

In turn, we use \gulp{} to produce low-dimensional embeddings of various DNNs that provide new insights on the relationship between various architectures (Figures~\ref{fig:teaser}, \ref{fig:mnist_width_depth_small}, and \ref{fig:imagenet_pretrained_dendogram}). Moreover, in Figure~\ref{fig:dists_during_training_one_plot}, we showcase a numerical experiment to demonstrate that the $\gulp{}$ distance between two independent networks decreases during training on the same dataset.

\medskip

\paragraph{Related work.}  This contribution is part of a growing body of work that aims at providing tools to understand and quantify the metric space of representations~\cite{RagGilYos17, MorRagBen18, KorNorLee19, AlaBen18, AroLiaMa17, ConLamRan18, LaaCot00, LenVed15, LiYosClu15, LiaLiLi19,miglani2019comparing, SmiTurHam17, WanHuGu18, ding2021grounding, davari2022inadequacy,cui2022deconfounded}. Several of these measures, such as \svcca{}~\cite{RagGilYos17} and \pwcca{}~\cite{MorRagBen18}, are based on a classical canonical correlation analysis (\cca{}) from multivariate analysis~\cite{And84}. More recently, centered kernel alignment \cka{}~\cite{cristianini2001kernel, cortes2012algorithms, KorNorLee19,davari2022inadequacy} has emerged as a popular measure; see  Section~\ref{sec:pb_statement} for more details on these methods.
The orthogonal procrustes metric (\procrustes{}) is a classical tool of shape analysis~\cite{DryMar16} to compute the distance between labelled point clouds. Though not as conspicuous as \cka{}-based methods in the context of DNN representations, it was recently presented under a favorable light in~\cite{ding2021grounding}. 

Various desirable properties of a similarity measure between representations have been put forward. These include structural  properties such as invariance or equivariance~\cite{LaaCot00, KorNorLee19}, as well as sanity checks such as specificity against random initialization~\cite{ding2021grounding}, for example. Such desiderata can serve as diagnostics for existing similarity measures, but fall short of providing concrete design guidelines.

\paragraph{Outline} The rest of the paper proceeds as follows. Section \ref{sec:pb_statement} lays out the derivation of \gulp{}, as well as important theoretical properties: conditions under which it is zero, and limiting cases in terms of the regularization parameter $\lambda$, demonstrating that it interpolates between $\cca{}$ and a version of $\cka{}$. Section \ref{sec:plug-in} establishes concentration results for the finite-sample version, justifying its use in practice. In Section \ref{sec:experiments} we validate \gulp{} through extensive experiments\footnote{Our code is available at \href{https://github.com/sgstepaniants/GULP}{https://github.com/sgstepaniants/GULP}.}. Finally, we conclude in Section \ref{sec:conclusion}.

\section{The $\gulp{}$ distance} \label{sec:pb_statement} 
As stated in the introduction, the goal of this paper is to develop a pseudo-metric over the space of representations of a given dataset. Unlike previous approaches, which work with finite datasets, we take a statistical perspective and formulate the population version of our problem. We defer statistical questions arising from finite sample size to Section~\ref{sec:plug-in}. 

Let $X \in \R^d$ be a random input with distribution $P_X$ and let $f: \R^d \to \R^k$ denote a \emph{representation map}, such as a trained DNN.
The random vector $f(X)\in \R^k$ is the \emph{representation of $X$ by $f$}. %
We assume throughout that a representation map is centered and normalized, so that
$\E[f(X)]=0$ and $\E\|f(X)\|^2=1$. 
In particular, this normalization allows us to identify (unnormalized) %
representation maps $\phi, \psi$ that are related by  $\psi(x)=a\phi(x) + b$, $P_X$-a.s. for $a \in \R$ and $b \in \R^d$, down to a single representation of $X$ (after normalizing), which is a well-known requirement for distances between representations~\cite[Sec. 2.3]{KorNorLee19}.

We are now in a position to define the \gulp{} distance between representations; the terminology ``distance" is justified in Theorem~\ref{thm:gulp_dist}. To that end, let  $\phi:\R^d \to \R^k$ and $\psi: \R^d \to \R^\ell$ be two representation maps, where $\ell$ may differ from $k$. Let $(X,Y) \in \R^d\times\R$ be a random pair and let $\eta(x)=\E[Y|X=x]$ denote the regression function of $Y$ onto $X$. Moreover, for any $\lambda>0$, let $\beta_\lambda$ denote the population ridge regression solution given by
\begin{align*}
    \beta_{\lambda} = \argmin_\beta\E[(\beta^\top \phi(X) - Y)^2] + \lambda \|\beta\|^2
\end{align*}
and similarly for $\gamma_{\lambda}$ with respect to $\psi(\cdot)$. Since we use squared error, these only depend the distribution of $Y$ through the regression function $\eta$.

\begin{Definition}
Fix $\lambda>0$. The \gulp{} distance between representations $\phi(X)$ and $\psi(X)$ is given by
\begin{align*}
    d_{\lambda}(\phi, \psi) := \sup_{\eta}\Big(\E(\beta_\lambda^\top \phi(X)-\gamma_\lambda^\top \psi(X))^2\Big)^\frac{1}{2}\,,
\end{align*}
where the supremum is taken over all regression functions $\eta$ such that $\|\eta\|_{L^2(P_X)} \leq 1$. 
\end{Definition}
The $\gulp{}$ distance measures the discrepancy between the prediction of an optimal ridge regression estimator based on $\phi$, and its counterpart based on $\psi$, uniformly over all regression tasks. While this notion of distance is intuitive and motivated by a clear regression task, it is unclear how to compute it \emph{a priori}. The next proposition provides an equivalent formulation of \gulp{}, which is amenable to accurate and efficient estimation; see Section~\ref{sec:plug-in}. %
It is based on the following covariance matrices:
\begin{equation}
    \label{eq:covs}
    \Sigma_\phi=\cov(\phi(X))=\E[\phi(X)\phi(X)^\top]\, \qquad  \Sigma_\psi=\cov(\psi(X))=\E[\psi(X)\psi(X)^\top]
\end{equation}
We implicitly used the centering assumption in the above definition, and the normalization condition implies that the covariance matrices have unit trace. Throughout, we assume these matrices are invertible, which is without loss of generality by projecting onto the image of the representation map. We also define the regularized inverses: 
$$
\Sigma_\phi^{-\lambda}:=(\Sigma_\phi + \lambda I_k)^{-1}\,, \qquad \Sigma_\psi^{-\lambda}:=(\Sigma_\psi + \lambda I_\ell)^{-1}
$$
as well as the cross-covariance matrices $\Sigma_{\phi\psi}$ and $\Sigma_{\psi\phi}$ as follows:
\begin{equation}
    \label{eq:cross_cov}
\Sigma_{\phi\psi}=\E[\phi(X)\psi(X)^\top]= \Sigma_{\psi\phi}^\top\,.
\end{equation}

\begin{proposition}\label{prop:closed-form}
Fix $\lambda\ge 0$. The \gulp{} distance between representations $\phi(X)$ and $\psi(X)$ satisfies
\begin{align}
    \boxed{\begin{aligned}
    d^2_\lambda(\phi, \psi)= \tr(\Sigma_{\phi}^{-\lambda} \Sigma_{\phi} \Sigma_{\phi}^{-\lambda} \Sigma_{\phi})
+ \tr(\Sigma_{\psi}^{-\lambda} \Sigma_{\psi} \Sigma_{\psi}^{-\lambda} \Sigma_{\psi}) - 2\tr(\Sigma_{\phi}^{-\lambda} \Sigma_{\phi \psi}\Sigma_{\psi}^{-\lambda} \Sigma_{\phi \psi}^\top)        \label{gulp2}
 \end{aligned}}
 \end{align}
\end{proposition}
\begin{proof}
See Appendix~\ref{app:closed-form}.
\end{proof}

\subsection{Structural properties}
In this section, we show that \gulp{} is invariant under orthogonal transformations and that it is a valid metric on the space of representations. We begin by establishing a third characterization of \gulp{} that is useful for the purposes of this section; the proof can be found in Appendix~\ref{app:closed-form}.

\begin{lem}
\label{lem:gulp3}
Fix $\lambda\ge 0$. The \gulp{} distance $d_\lambda(\phi, \psi)$ between the representations $\phi(X)$ and $\psi(X)$ satisfies
$$
d_\lambda^2(\phi, \psi)=\E (\phi(X)^\top \Sigma_\phi^{-\lambda}\phi(X') - \psi(X)^\top  \Sigma_\psi^{-\lambda}\psi(X'))^2\,,
$$
where $X'$ is an independent copy of $X$.
\end{lem}

We are now in a position to state our main structural results. We begin with a key invariance result.

\begin{Theorem}
\label{thm:invariance}
Fix $\lambda\ge 0$. The \gulp{} distance $d_\lambda(\phi, \psi)$ between the representations $\phi(X) \in \R^k$ and $\psi(X)\in \R^\ell$ is invariant under orthogonal transformations: for any orthogonal transformations $U: \R^{k}\to  \R^{k}$ and $V :\R^\ell \to \R^\ell$, it holds
$$
d_\lambda(U\circ\phi, V\circ\psi)=d_\lambda(\phi, \psi)
$$
\end{Theorem}
\vspace{-2em}
\begin{proof}
We slightly abuse notation by identifying any orthogonal transformation $W$ to a matrix $W$ such that $W(x)=W\cdot x$. Note that for any representation map, we have $\Sigma_{W\circ f}=W\Sigma_f W^{\top}$ and
$$
\Sigma_{W\circ f}^{-\lambda}=(W\Sigma_{f}W^\top + \lambda WW^\top)^{-1}=W(\Sigma_f^\top + \lambda I)W^\top=W\Sigma_{f}^{-\lambda}W^\top\,.
$$
Hence, using Lemma~\ref{lem:gulp3}, we get that
\begin{align*}
    d_\lambda^2(U\circ\phi, V\circ\psi)&=\E (\phi(X)^\top U^\top U\Sigma_\phi^{-\lambda}U^\top U\phi(X') - \psi(X)^\top  V^\top V\Sigma_\psi^{-\lambda}V^\top V\psi(X'))^2\\
    &=\E (\phi(X)^\top \Sigma_\phi^{-\lambda}\phi(X') - \psi(X)^\top \Sigma_\psi^{-\lambda}\psi(X'))^2=d_\lambda^2 (\phi, \psi)\,,
\end{align*}
where we used the fact that $U^\top U=I_k$ and $V^\top V=I_\ell$.
\end{proof}

Next, we show that \gulp{} satisfies the axioms of a metric.

\begin{Theorem}
\label{thm:gulp_dist}
Fix $\lambda> 0$. The \gulp{} distance $d_\lambda(\phi, \psi)$ satisfies the axioms of a pseudometric, namely for all representation maps $\phi, \psi, \varphi$, it holds
\begin{align*}
    d_\lambda(\phi, \phi)=0, \quad\ \  d_\lambda(\phi, \psi)=d_\lambda(\psi,\phi), \quad \mbox{and} \quad d_\lambda(\phi, \psi)\le d_\lambda(\phi, \varphi) + d_\lambda (\varphi, \psi)
\end{align*}
Moreover, $d_{\lambda}(\phi, \psi)=0$ if and only if $k=\ell$ and there exists an orthogonal transformation $U$ such that $\phi(X)=U\psi(X)$ a.s.
\end{Theorem}
\vspace{-2em}
\begin{proof}
Lemma~\ref{lem:gulp3}  provides  an isometric embedding of representations $f \mapsto f(X)\Sigma_f^{-\lambda} f(X')$ into the Hilbert space $L^2(P_X^{\otimes 2})$. It readily yields that $d_\lambda$ is a pseudometric. It remains to identify for which $\phi, \psi$ it holds that $d_\lambda(\phi, \psi)=0$.

The ``easy'' direction follows from the invariance property of Theorem~\ref{thm:invariance}: if $\phi$ and $\psi$ satisfy $\phi(X)=U\psi(X)$ almost surely, then $d_\lambda(\phi, \psi)=d_\lambda(U\psi, \psi)=0$. We sketch the proof of the other direction, and defer the full proof to Appendix~\ref{app:distance-proofs}. Define $\tilde{\phi} = (\Sigma_{\phi} + \lambda I)^{-1/2} \phi$ and $\tilde{\psi} = (\Sigma_{\psi} + \lambda I)^{-1/2} \psi$. By Lemma~\ref{lem:gulp3}, the condition that $d_{\lambda}(\phi,\psi) = 0$ is equivalent to $\tilde{\phi}(X)^{\top} \tilde{\phi}(X') = \tilde{\psi}(X)^{\top} \tilde{\psi}(X)$ almost surely over $X,X'$. 
So if $d_{\lambda}(\phi,\psi) = 0$, then we can leverage a classical fact that the Gram matrix of a set of vectors determines the vectors up to an isometry \cite{horn2012matrix}, to prove that there is an orthogonal transformation $U \in \R^{k \times k}$ such that $\tilde{\phi}(X) = U \tilde{\psi}(X)$ almost surely over $X$. Finally, via analyzing a homogeneous Sylvester equation, this implies that $\phi(X) = U \psi(X)$ almost surely.
\end{proof}
Note that when $\lambda=0$, the conclusion of this theorem fails to hold: %
$d_0$ still satisfies the axioms of a pseudo-distance, but the cases for which $d_0(\phi,\psi)=0$ are different. This point is illustrated in the next section where we establish that $d_0$ is the \cca{} distance commonly employed in the literature.

\subsection{Comparison with \cca{}, ridge-\cca{}, \cka{}, and \procrustes{}}

Throughout this section, we assume that $k=\ell$ for simplicity.

\paragraph{Ridge-\textsc{cca}.} Our distance is most closely related to ridge-\cca{}, introduced by \cite{VINOD1976147} as a regularized version of Canonical Covariance Analysis (\cca{}) when the covariance matrices $\Sigma_\phi$ or $\Sigma_\psi$ are close to singular. More specifically, for any $\lambda \ge 0$, define the matrix $C_\lambda:=\Sigma_\phi^{-\lambda}\Sigma_{\phi\psi}\Sigma_\psi^{-\lambda}\Sigma_{\psi\phi}$; the ridge-\cca{} similarity measure is defined as $\rho_{\lambda-\cca{}}=\tr(C_\lambda)$. Hence, we readily see from Proposition~\ref{prop:closed-form} that \gulp{} and ridge-\cca{} are describing the same geometry over representations. To see this, recall that Lemma~\ref{lem:gulp3} provides an isometric embedding $f \mapsto f(X) \Sigma_{f}^{-\lambda} f(X')$ of representation maps into $L^2(P_X^{\otimes 2})$. While \gulp{} is the distance on this Hilbert space, ridge-\cca{} is the inner product. 

Ridge-\cca{} was briefly considered in the seminal work \cite{KorNorLee19} but discarded because of (i) its lack of interpretability and (ii) the absence of a rule to select $\lambda$. We argue that in fact, our prediction-driven derivation of \gulp{} gives a clear and compelling interpretation of this geometry (as well as suggests several extensions; see Section \ref{sec:conclusion}). Moreover, we show that tunability of $\lambda$ is, in fact, a desirable feature that allows to represent the space of representations at various resolutions, giving various levels of information; for example, in Figure~\ref{fig:imagenet_pretrained_dendogram}, higher $\lambda$ leads to a coarser clustering structure.

\paragraph{\textsc{cca}.} Due to the connection with ridge-$\cca$, our $\gulp$ distance is related to (unregularized) $\cca$ when $\lambda = 0$. Specifically, defining $C:= \Sigma_\phi^{-1}\Sigma_{\phi\psi}\Sigma_\psi^{-1}\Sigma_{\psi\phi}$, the mean-squared-\cca{} similarity measure is given by (see \cite[Def. 10.2]{Eat07}):
$$
\rho_{\cca}(\phi, \psi):=\frac{\tr(C)}{k}=1-\frac1{2k}\E\big[( \phi(X)^\top \Sigma_\phi^{-1} \phi(X') -  \psi(X)^\top \Sigma_\psi^{-1}\psi(X') )^2\big]\,,
$$
where $X$ is an independent copy of $X'$; the last identity can be checked directly. From Lemma~\ref{lem:gulp3} it can be seen that our \gulp{} distance $d_0(\phi, \psi)$ with $\lambda=0$ is a linear transformation of $\rho_{\cca}$.

It can be checked that  $\rho_{\cca}$ takes values in $[0,1]$, which has led researchers to simply propose $1-\rho_{\cca}$ as a dissimilarity measure. Interestingly, this choice turns out to produce a valid (squared) metric, i.e., a dissimilarity measure that satisfies the triangle inequality. Indeed, we get that
\begin{align*}
    d^2_{\cca}(\phi, \psi)=1-\rho_{\cca}(\phi, \psi)
    &=\frac{1}{2k}\E\big[(K(\tilde \phi(X),\tilde \phi(X'))-K(\tilde \psi(X),\tilde \psi(X')) )^2\big]
\end{align*}
where $K(u,v)= u^\top v$ is the linear kernel over $\R^d$ and $\tilde \phi:= \Sigma_\phi^{-1/2} \phi$ (where $\tilde \psi$ and $\tilde \phi$ are the whitened versions of $\psi$ and $\phi$ respectively). 
These identities have two consequences: (i) we see from Lemma~\ref{lem:gulp3} that $d_{\cca{}}$ corresponds to  the \gulp{} distance with $\lambda=0$ up to a scaling factor and (ii) $d_{\cca{}}$ is a valid pseudometric on the space of representations, since we just exhibited an isometry $T: \tilde f  \mapsto K(\tilde f(X), \tilde f(X'))$ with $L^2(P_X^{\otimes 2})$. We show in Appendix~\ref{app:distance-proofs} that $d_{\cca{}}(\phi, \psi)=0$ iff $\psi(X)=A\phi(X)$ a.s. for some matrix $A$. Note that the invariance of $\rho_{\cca{}}$ to linear transformations was previously known and criticized in~\cite{KorNorLee19} as arguably too strong.

\paragraph{\textsc{cka}.} 
In fact, thanks to the additional structure of the Hilbert space $L^2(P_X^{\otimes 2})$, the $d_{\cca}$ distance comes with an inner product
$$
\langle T(\tilde \phi), T(\tilde \psi) \rangle_{\cca}:=\frac{1}{2k}\E[K(\tilde \phi(X),\tilde \phi(X'))K(\tilde \psi(X),\tilde \psi(X'))]\,
$$
This observation allows us to connect \cca{} with \cka{}, another measure of similarity between distributions that is  borrowed from classical literature on kernel methods~\cite{cristianini2001kernel,cortes2012algorithms} and that was recently made popular by~\cite{KorNorLee19}. Under our normalization assumptions, \cka{} is a measure of similarity given by
\begin{align*}
    \rho_{\cka}(\phi, \psi)&= \frac{\E[K(\phi(X),\phi(X'))K(\psi(X), \psi(X'))]}{\sqrt{\E[K(\phi(X),\phi(X'))^2]\E[K(\psi(X),\psi(X'))^2]}}\\
    &=\frac{\langle T(\phi) , T(\psi) \rangle_{\cca{}}}{\|T(\phi)\|_{\cca{}}\|T(\psi)\|_{\cca{}}}=\cos\left(\measuredangle( T(\phi), T(\psi))\right)\,,
\end{align*}
where $\|T\|_{\cca{}}^2=\langle T, T\rangle_{\cca{}}$ and $\measuredangle$ denotes the angle in the geometry induced by $\langle\cdot, \cdot \rangle_{\cca{}}$\,. In turn, $d^2_{\cka{}}$ is chosen as $d^2_{\cka{}}=1-\rho_{\cka{}}$, which does not yield a pseudometric. This observation highlights  two major differences between \cca{} and \cka{}: the first measures inner products and works with whitened representations, while the second measures angles and works with raw representations. As illustrated in the experimental section~\ref{sec:experiments} as well as in~\cite{ding2021grounding}, this additional whitening step appears to be detrimental to the overall qualities of this distance measure.

The fact that \gulp{} with $\lambda=0$ recovers $d_{\cca{}}$ (i.e. $d^2_0=2kd^2_{\cca{}}$) is illustrated in Figure~\ref{fig:imagenet-relationships}. As shown, although \gulp{} has a roughly monotone relationship with \cka{}, they remain quite different.

\begin{figure}
    \centering
    \includegraphics[trim={1cm 1cm 1cm 1.05cm},clip,scale=0.28]{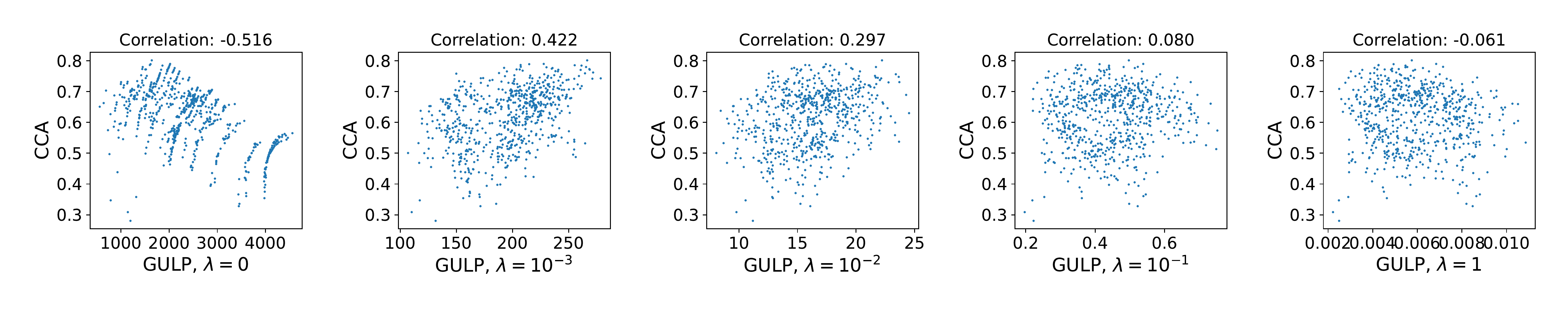}
    \includegraphics[trim={1cm 1cm 1cm 1.05cm},clip,scale=0.28]{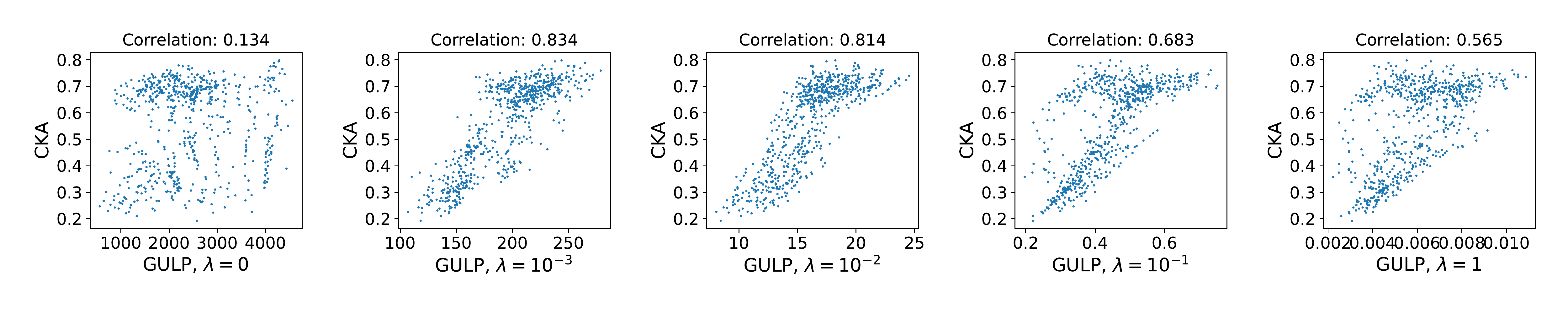}
    \caption{Empirical relationship between distances. Each point in the scatter-plot corresponds to a pair of ImageNet representations; the $x$-coordinate is the $\gulp$ distance, and the $y$-coordinate is the $\cca$ or $\cka$ distance. Although $\cca$ and $\gulp$  for $\lambda = 0$ are related, their relationship is not linear since the representations' dimensionalities differ. Although $\cka$ and $\gulp$ are related for large $\lambda$, their relationship is not linear due to the difference in normalization. Appendix~\ref{app:relationships} contains more details and comparisons, including a surprisingly strong correlation between $\gulp$ and $\procrustes$ for some values of $\lambda$.} 
    \label{fig:imagenet-relationships} 
\end{figure}
\paragraph{\procrustes{}.} The relationship between \gulp{} and \procrustes{} is not as clean as in the previous comparisons, but we include it for completeness. In the limit of infinite samples, the Procrustes distance as derived by \cite{schonemann1966generalized} is
\begin{align*}
d_{\text{Procrustes}} &= \tr(\Sigma_\phi) + \tr(\Sigma_\psi) - 2 \tr\big( (\Sigma_{\phi\psi}\Sigma_{\phi\psi}^{\top})^{1/2} \big).
\end{align*}
Our normalization implies $\tr(\Sigma_\phi)=\tr(\Sigma_\psi)=k$. However, the term $ \tr\big( (\Sigma_{\phi\psi}\Sigma_{\phi\psi}^{\top})^{1/2} \big)$ (which is equal to the nuclear norm $ || \Sigma_{\phi\psi}||_*$) is not directly comparable to the preceding distances.

\section{Plug-in estimation of $\gulp$}\label{sec:plug-in}

In practice, the distribution $P_X$ of $X$ is unknown, so we cannot compute the population version of \gulp{} exactly. Instead, we have access to a sample $X_1, \ldots, X_n \simiid P_X$. In all of the experiments of this paper, we approximate $\gulp$ with the following plug-in estimator:
\begin{align*}
    \boxed{\begin{aligned}\hat{d}^2_{\lambda, n}(\phi,\psi) &=  \tr(\hat{\Sigma}_{\phi}^{-\lambda} \hat{\Sigma}_{\phi} \Sigma_{\phi}^{-\lambda} \hat{\Sigma}_{\phi})  + \tr(\hat{\Sigma}_{\psi}^{-\lambda} \hat{\Sigma}_{\psi} \Sigma_{\psi}^{-\lambda} \hat{\Sigma}_{\psi})  - 2\tr(\hat{\Sigma}_{\phi}^{-\lambda} \hat{\Sigma}_{\phi \psi} \hat{\Sigma}_{\psi}^{-\lambda} \hat{\Sigma}_{\phi \psi}^\top),
\end{aligned}}
\end{align*}
where \begin{align*}
\hat{\Sigma}_{\phi} = \frac{1}{n} \sum_{i=1}^n \phi(X_i) \phi(X_i)^{\top}, \quad \hat{\Sigma}_{\psi} = \frac{1}{n} \sum_{i=1}^n \psi(X_i) \psi(X_i)^{\top},\quad \mbox{and} \quad \hat{\Sigma}_{\phi\psi} = \frac{1}{n} \sum_{i=1}^n \phi(X_i) \psi(X_i)^{\top}
\end{align*}
are the empirical covariance and cross-covariance matrices, and
\begin{align*}
\hat{\Sigma}_{\phi}^{-\lambda} = (\hat{\Sigma}_{\phi} + \lambda I)^{-1}, \quad \mbox{and} \quad \hat{\Sigma}_{\psi}^{-\lambda} = (\hat{\Sigma}_{\psi} + \lambda I)^{-1}
\end{align*} 
are the empirical inverse regularized covariance matrices. To justify our use of the plug-in estimator, we prove concentration around the population $\gulp$ distance as $n$ goes to infinity. 
\begin{Theorem}\label{thm:plug-in-concentration}
Assume that $\|\phi(X)\|^2, \|\psi(X)\|^2 \leq 1$ almost surely. Then, for any  $\lambda \in (0,1)$, $\delta > 0$, with probability at least $1 - \delta$ the plug-in estimator $\hat{d}^2_{\lambda,n}$  satisfies
\begin{align*}
\left|\hat{d}_{\lambda,n}^2(\phi,\psi) - d_{\lambda}^2(\phi,\psi)\right| \lesssim \frac{1}{\lambda^3} \sqrt{\frac{\log ((k+l)/\delta)}{n}}\,.
\end{align*}
\end{Theorem}
We defer the proof of this theorem to Appendix~\ref{app:plug-in-proof}.  At a high-level, we first show that the inverse regularized covariance matrices, $(\Sigma_{\phi} + \lambda I)^{-1}$ and $(\Sigma_{\psi} + \lambda I)^{-1}$, are well-approximated in operator norm, so the expectation of the plug-in estimator is close to the population distance. We then apply McDiarmid's inequality to show that the plug-in estimator concentrates around its expectation. Note that the boundedness conditions on the representations are here to simplify technical arguments by appealing simply to McDiarmid's inequality; these can be presumably be relaxed to weaker tail conditions at the cost of more involved arguments.

Figure~\ref{fig:convergence} supports our theoretical result by showing  convergence on pairs of networks on the ImageNet dataset. See Appendix~\ref{app:convergence} for more details.

\section{Experiments}\label{sec:experiments}

We evaluate our distance in a variety of empirical settings, comparing to $\cca{}$, $\cka{}$, the classical $\procrustes{}$ method from shape analysis, and a variant of \cca{} known as projection-weighted \cca{} (\pwcca{}); see \cite[Sec. 2]{ding2021grounding} for definitions. 

\subsection{$\gulp{}$ captures generalization performance by linear predictors}\label{sec:predict}

The $\gulp{}$ distance is motivated by how differently linear predictors using the representations $\phi$ and $\psi$ generalize. In this section, we demonstrate that $\gulp$ indeed captures downstream generalization performance by linear predictors. We consider the representation maps $\phi_1,\ldots,\phi_m$ given by $m = 37$ pretrained image classification architectures on the ImageNet dataset $P_X$ (see Appendix~\ref{app:network_experiments}). For each pair of representations, %
we estimate the $\cka$, $\cca$, $\pwcca$, and $\gulp$ distances, using the plug-in estimators on 10,000 images, sufficient to guarantee good convergence (see Figure \ref{fig:convergence}).

We then draw $n =$ 5,000 images from the dataset $X_1,\ldots,X_n \sim P_X$, and assign a random label $Y_k \sim \cN(0,1)$ to each one. For each representation $i \in [m]$, we fit a $\lambda$-regularized least-squares linear regression to the training data $\{(X_k,Y_k)\}_{k \in [n]}$, which gives a coefficient vector $\beta_{\lambda,i}$. Finally, for each $1 \leq i < j \leq m$, we estimate the distance $\tau_{ij} = \E_{X \sim P_X}[(\beta_{\lambda,i}^{\top} \phi_i(X) - \beta_{\lambda,j}^{\top} \phi_j(X))^2]$ between the predictions with representations $\phi_i$ and $\phi_j$, by taking the empirical average over $3000$ samples in a test set. In Figure~\ref{fig:generalization}, we plot Spearman's $\rho$ rank correlation between $\tau$ and each of the distances $\gulp$, $\cka$, $\cca$, $\pwcca$, viewed as vectors with $\binom{m}{2}$ entries, one for each pair of networks. Notice that for each $\lambda$, the distance that attains the best correlation is the $\gulp$ distance with that $\lambda$. This indicates that while \gulp{} is a measure of distance that holds uniformly over prediction tasks, it retains its meaning in the context of a single prediction task.

\begin{figure}
\centering
\begin{minipage}{0.38\textwidth}
    \centering
    \includegraphics[width=\textwidth]{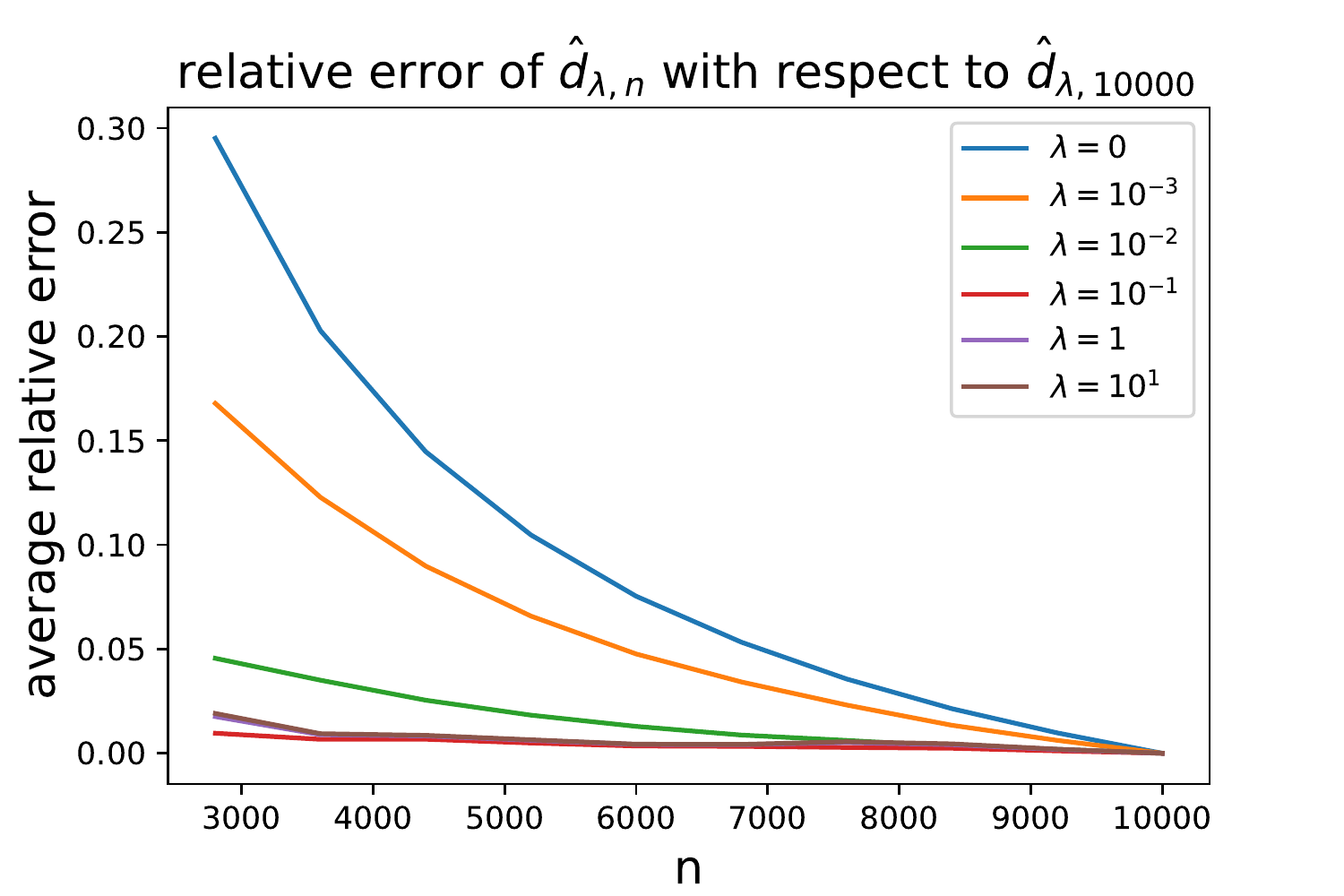}
    \caption{Convergence of plug-in esimator as $n \to \infty$. We plot relative error $|\hat{d}^2_{\lambda,n} - d^2_{\lambda,10000}| / d^2_{\lambda,10000}$ averaged over pairs of ImageNet DNNs.
    }
    \label{fig:convergence}
\end{minipage}%
\quad
\begin{minipage}{0.58\textwidth}
    \centering
    \includegraphics[width=.48\textwidth]{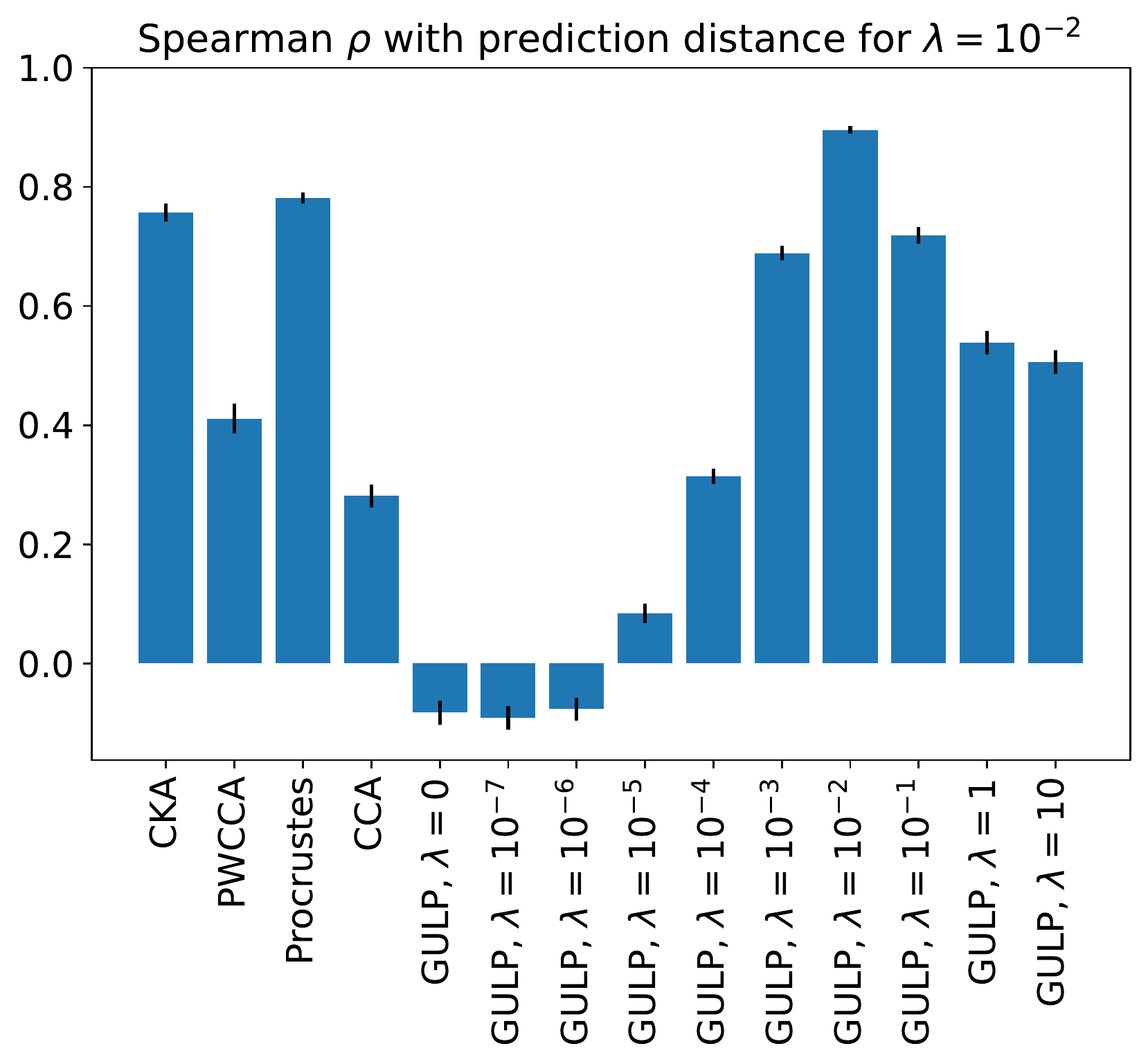}
    \includegraphics[width=.48\textwidth]{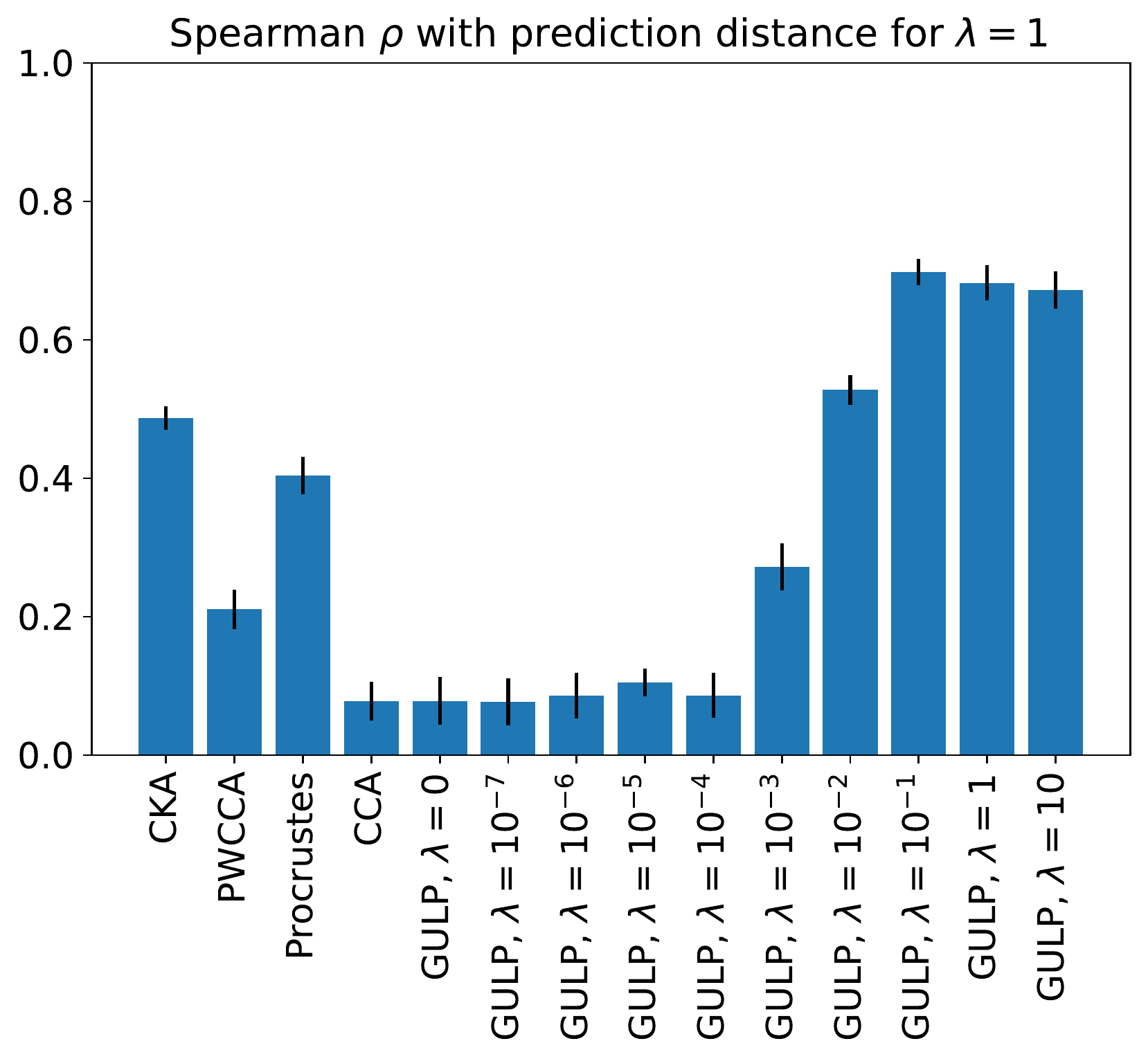}
    \caption{$\gulp$ captures generalization of linear predictors. We plot Spearman's $\rho$ between the differences in predictions by $\lambda$-regularized linear regression, and the different distances. Results are averaged over 10 trials.}
    \label{fig:generalization}
    \end{minipage}
\end{figure}

\subsection{$\gulp{}$ distances cluster together networks with similar architectures}\label{subsec:network_experiments}

We are interested in how \gulp{} can be used to compare networks of different architectures trained on the same task. We begin by comparing fully-connected ReLU networks of widths ranging from 100 to 1,000 and depths ranging from 1 to 10, trained on the MNIST handwritten digit database. Every architecture is retrained four times from different initializations (see Appendix~\ref{app:experimental_setup}). We input all MNIST training set images into each network, save their representations at the final hidden layer, and compute \cka{}, \procrustes, and \gulp{} distances between all pairs of representations.

Figure~\ref{fig:mnist_width_depth_small} shows Multi-Dimensional Scaling (MDS) embeddings of the distances between all MNIST networks, color coded by width and depth. For $\gulp$ with $\lambda = 10^{-6}$, the networks are largely organized according to rank of the feature matrix: networks of large width and small depth, ones whose representations have the largest rank, are the most different, as evidenced by the halo of points in the MDS plots. This outcome confirms that \cca{} simply measures rank~\cite{KorNorLee19}. However, for \procrustes{} and \gulp{} with $\lambda = 10^{-2}$, networks become clustered by their depth, as evidenced by the striations in the MDS embeddings (plot colored by depth). Furthermore, networks of the same depth look most similar at large widths, as shown by the red centerline in the MDS embedding (plot colored by width), implying that as width increases networks converge to a shared limiting representation. Finally, \gulp{} with $\lambda = 1$ closely resembles \cka{} and roughly organizes networks by depth. A takeaway is that \gulp{} with $\lambda = 1$  resembles \procrustes{} and \cka{}, and captures intrinsic characteristics such as width and depth.

\begin{figure}
\centering
\includegraphics[width=0.9\linewidth]{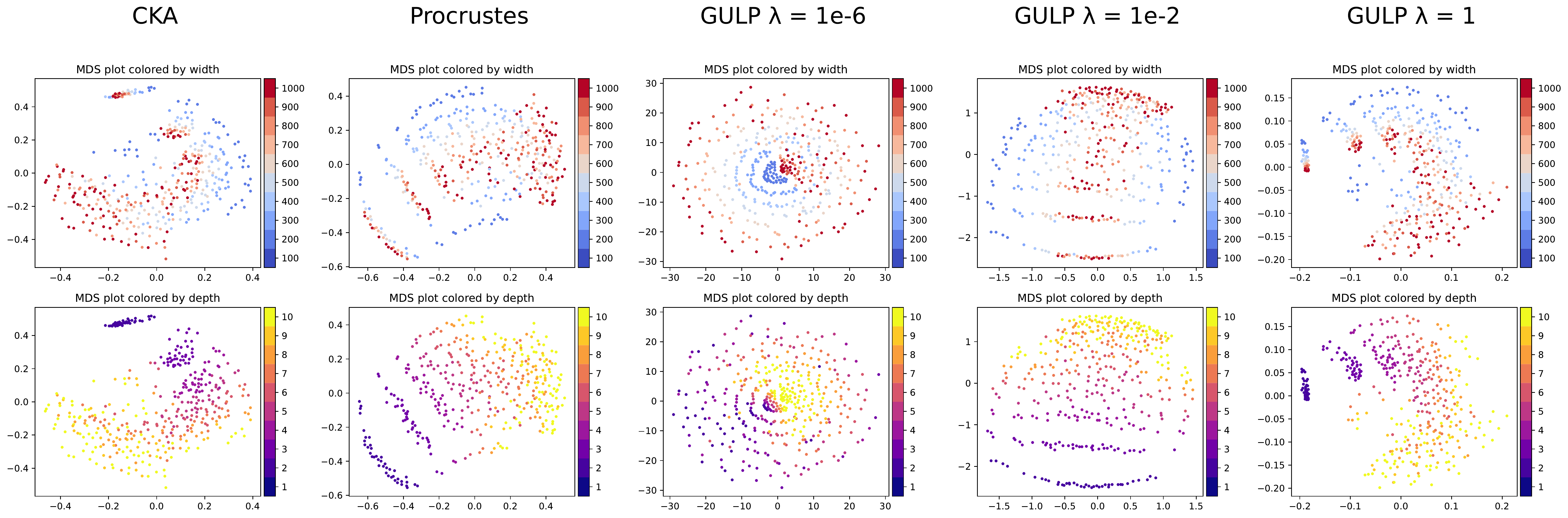}
\caption{Two dimensional MDS embedding plots of fully-connected ReLU networks colored by architecture width (top) and depth (bottom). Networks are fully-trained on MNIST and penultimate layer representations are constructed from 60,000 input train images.}
\label{fig:mnist_width_depth_small}
\end{figure}
Next, we show how distances between penultimate layer representations allow us to cluster pretrained networks with more complex architectures and, in turn, draw comparisons between them. To that end, we study 37 state-of-the-art models on the ImageNet Object Localization Challenge, of which the four major groups are ResNets, EfficientNets, ConvNeXts, and MobileNets (see Appendix~\ref{app:experimental_setup}).

We compute the baseline %
distances between every pair of representations using 10,000 training images, and visualize them using a two-dimensional t-SNE embedding in Figure~\ref{fig:imagenet_pretrained_dendogram}. Below each embedding plot we show the dendogram resulting from a hierarchical clustering of the networks based on their distances. As seen from the embeddings, when $\lambda$ increases, the \gulp{} distance separates the ResNet architectures (blue) from the EfficientNet and ConvNeXt convolutional networks (orange and red). Compared to other distances, \gulp{} with large $\lambda$ is able to more compactly cluster ResNets and convolutional networks separately. In Appendix~\ref{app:network_experiments} we further quantify the compactness of clusterings under each distance metric by computing the standard deviation of distances within each cluster.

\begin{figure}
\centering
\includegraphics[width=\linewidth]{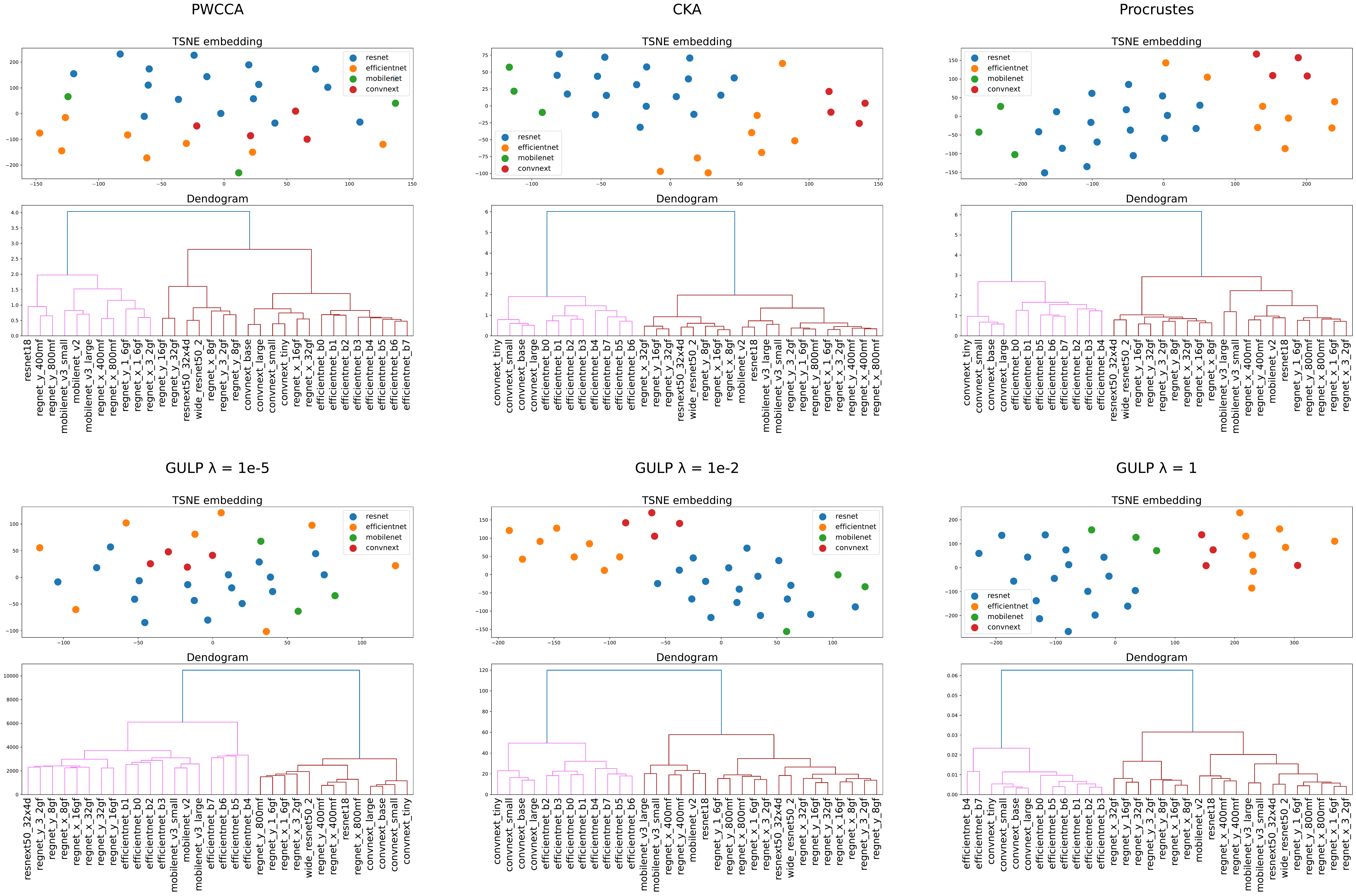}%
\caption{Embeddings of \pwcca{}, \cka{}, \procrustes{}, and \gulp{} distances between the last hidden layer representations of 36 pretrained ImageNet models (top) along with their hierarchical clusterings (bottom). All distance metrics separate ResNet architectures (brown dendogram leaves) from the rest of the ConvNeXt and EfficientNet architectures (pink dendogram leaves). \gulp{} at $\lambda=1$ is the most effective distance at separating ResNets from the remaining architectures.}
\label{fig:imagenet_pretrained_dendogram}
\end{figure}

\subsection{Network representations converge in $\gulp{}$ distance during training}

So far, we have used \gulp{} to compare static networks taken as a blackbox representation maps. Now we use \gulp{} to examine how representation maps evolve over the course of training. To that end, we independently train 16 Resnet18 architectures on the CIFAR10 dataset \cite{krizhevsky2009learning} for 50 epochs. Figure~\ref{fig:dists_during_training_one_plot} tracks the  distance (averaged over all  network pairs) at each epoch.

As shown, other distances change very little or even briefly increase over the course of training. For \gulp{} with small $\lambda$, the previous sections have demonstrated that our distance captures fine-grained differences between representations; here too, it accentuates differences in representations mid-training (visible around epoch 25). %
However, as $\lambda$ increases, the \gulp{} distance differentiates less between representations, and smoothly decreases over the course of training, thus indicating that it captures intrinsic properties of the representations rather than artifacts due to random seeds.

\begin{figure}
    \centering
     \includegraphics[width=0.8\linewidth]{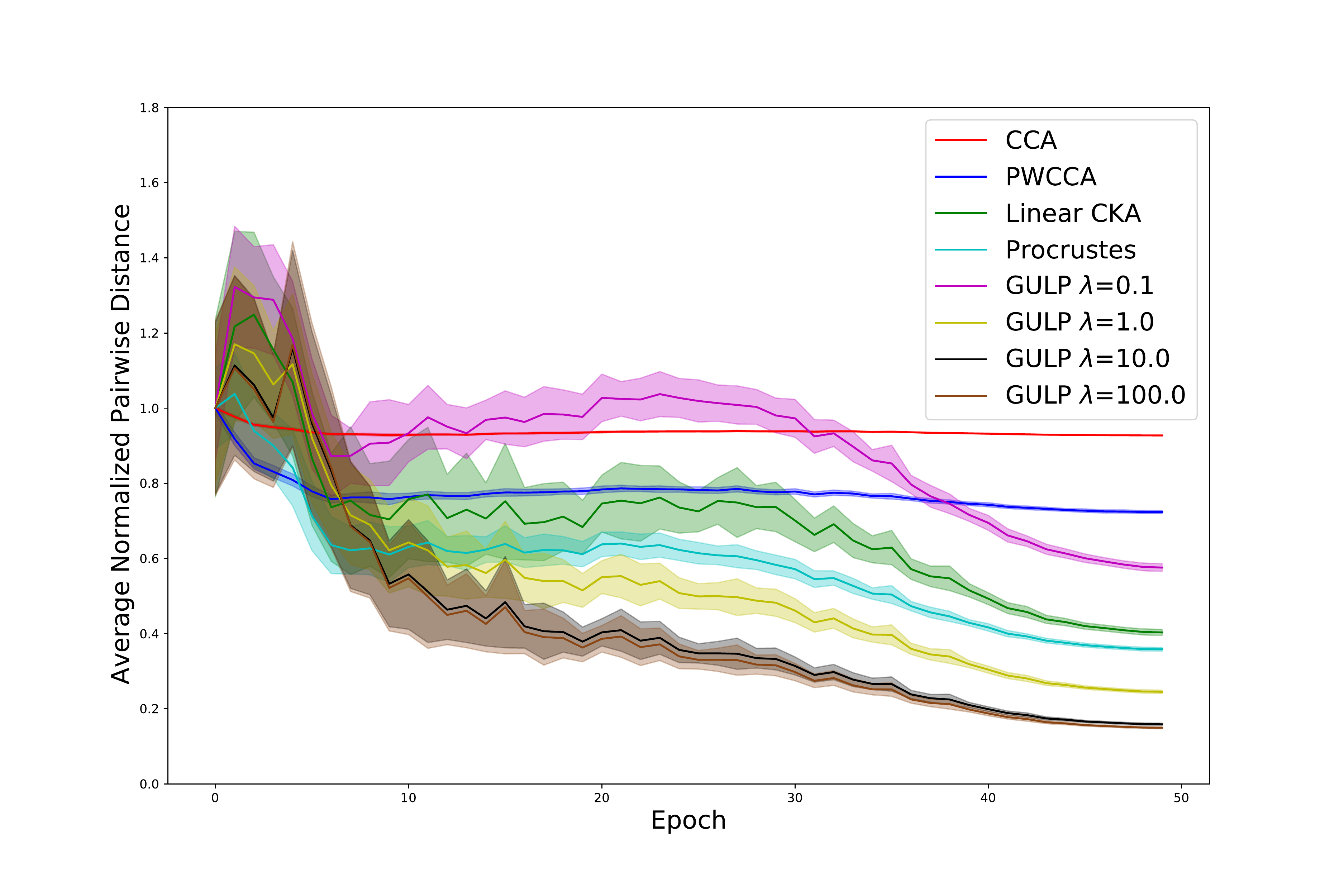}
    \caption{Empirical distances between representations of 16 independently trained ResNet18 architectures during training, computed using 3,000 samples and averaged over all ${16 \choose 2}$ pairs. Distances are scaled by their average value at iteration 0.}
    \label{fig:dists_during_training_one_plot}
\end{figure}

\subsection{Sensitivity versus specificity of $\gulp{}$}
In Appendix~\ref{app:ding-experiments}, we reproduce the experiments of \cite{ding2021grounding}. Our distance compares favorably to baselines %
and correlates with measures of a DNN's functional behavior. %
It achieves the specificity of \cca{} and \pwcca{} to random initializations, and improves the sensitivity of \cka{} to out-of-distribution performance.

\section{Conclusion}\label{sec:conclusion}
In this paper, we have defined a family of distances for comparing learned representations in terms of their worst-case performance gap over all $\lambda$-regularized regression tasks. We proved convergence of the finite-sample estimator of this distance, quantified its relationship to existing notions such as \cca{}, ridge-\cca{}, and \cka{}, and demonstrated promising performance in a variety of empirical settings, including the ability to distinguish between network architectures and to capture performance differences on regression tasks.

Further studying extensions beyond linear transfer learning could provide a rich direction for future work. In fact, preliminary  experiments reported in Appendix~\ref{app:logistic} indicate that, compared to section~\ref{sec:predict}, \gulp{} fails to predict generalization performance when the downstream task shifts from linear to logistic prediction. This suggests extending \gulp{} to a uniform bound over other downstream predictive tasks, such as logistic regression, multi-class classification, or kernel ridge regression. Although \gulp{} under kernel ridge regression has a closed form using the kernel trick\footnote{This can be easily derived from vanilla ridge regression as studied in this paper (see Appendix~\ref{app:kernel_ridge_reg}), and is related to ``kernel ridge CCA'' \cite{kuss2003the}.}, \gulp{} for logistic regression does not have a closed form. This brings additional computational questions of interest that are beyond the scope of this work. Finally, it could be interesting to consider the application of \gulp{} to knowledge distillation, or alternatively to consider adding a ridge regularization term to probing methods (inspired by \gulp{}).

\begin{ack}
EB is supported by an Apple AI/ML Fellowship, and the National Science Foundation Graduate Research Fellowship under Grant No. 1745302.
HL is supported by the Fannie and John Hertz Foundation and the National Science Foundation Graduate Research Fellowship under Grant No. 1745302. GS is supported by the National Science Foundation Graduate Research Fellowship under Grant No. 1745302. PR supported by NSF awards IIS-1838071, DMS- 2022448, and CCF-2106377.
\end{ack}

\bibliographystyle{alphaabbr}
\bibliography{ref.bib}

\appendix

\section{Deferred proofs}

\subsection{Alternate characterizations of $\gulp$, proofs of Proposition~\ref{prop:closed-form} and Lemma~\ref{lem:gulp3}}\label{app:closed-form}

We provide proofs of the two alternative characterizations to the $\gulp$ distance that were claimed in the main text.
\begin{proof}[Proof of Lemma~\ref{lem:gulp3}]
Fix a distribution of $(X,Y)$, and let $\eta(x)=\E[Y|X=x]$ be the regression function. Since we are using squared error, with $\phi$ features the best linear predictor is $\beta_{\lambda}$ that solves
$$
\beta_{\lambda}=\Sigma_{\phi}^{-\lambda} \E[Y\phi(X)]= \Sigma_{\phi}^{-\lambda} \E[\eta(X)\phi(X)]=\Sigma_{\phi}^{-\lambda}\int \eta(x)\phi(x) \ud P_X(x)
$$
where $P_X$ is the marginal distribution of $X$.  Similarly
$$
\gamma_{\lambda}=\Sigma_{\psi}^{-\lambda}\int \eta(x)\psi(x) \ud P_X(x)
$$

In particular, for a given distribution of $(X,Y)$, the distance between the best linear predictors is 
$$
\E(\beta_{\lambda}^\top \phi(X)-\gamma_{\lambda}^\top \psi(X))^2.
$$
We rewrite this in terms of $\eta$:
\begin{align*}
  \E(\beta_{\lambda}^\top \phi(X)-\gamma_{\lambda}^\top \psi(X))^2&=\E\left( \int \eta(x) \left[\phi(X)^\top \Sigma_{\phi}^{-\lambda}\phi(x) - \psi(X)^\top \Sigma_{\psi}^{-\lambda}\psi(x)\right]\ud P_X(x) \right)^2  \\
  &=\E \langle \eta, \phi(X)^\top \Sigma_{\phi}^{-\lambda}\phi(\boldsymbol{\cdot}) - \psi(X)^\top \Sigma_{\psi}^{-\lambda}\psi(\boldsymbol{\cdot})\rangle_{L^2(P_X)}^2
\end{align*}
Therefore to sup out the distribution over $Y$, we take the sup of $\eta$ such that $\|\eta\|_{L^2(P_X)} \le 1$. It yields the claim of Lemma~\ref{lem:gulp3}.
\begin{align*}
  d_{\lambda}^2(\phi, \psi) &:= \sup_{\|\eta\|_{L^2(P_X)} \le 1}\E \langle \eta, \cdots \rangle_{L^2(P_X)}^2 \\
  &=\E\|\phi(X)^\top \Sigma_\phi^{-\lambda} \phi(\boldsymbol{\cdot}) - \psi(X)^\top \Sigma_\psi^{-\lambda}\psi(\boldsymbol{\cdot})\|_{L^2(P_X)}^2  \\
  &=\E (\phi(X)^\top \Sigma_\phi^{-\lambda}\phi(X') - \psi(X)^\top \Sigma_\psi^{-\lambda}\psi(X'))^2
\end{align*}
where $X,X'\sim P_X$ are independent.
\end{proof}

Using Lemma~\ref{lem:gulp3}, we can easily prove Proposition~\ref{prop:closed-form}.
\begin{proof}[Proof of Proposition~\ref{prop:closed-form}]
Start with the characterization in Lemma~\ref{lem:gulp3}, expand the square and use the cylicity and linearity of the trace:
\begin{align*}
    d^2_\lambda(\phi, \psi) &= \E (\phi(X)^\top \Sigma_\phi^{-\lambda}\phi(X')\phi(X')^\top \Sigma_\phi^{-\lambda}\phi(X)) \\
                            &\quad + \E (\psi(X)^\top \Sigma_\psi^{-\lambda}\psi(X')\psi(X')^\top \Sigma_\psi^{-\lambda}\psi(X)) \\
                            &\quad - 2 \E (\phi(X)^\top \Sigma_\phi^{-\lambda}\phi(X')\psi(X')^\top \Sigma_\psi^{-\lambda}\psi(X)) \\
                            &= \tr \E ( \Sigma_\phi^{-\lambda}\phi(X')\phi(X')^\top \Sigma_\phi^{-\lambda}\phi(X)\phi(X)^\top) \\
                            &\quad + \tr \E ( \Sigma_\psi^{-\lambda}\psi(X')\psi(X')^\top \Sigma_\psi^{-\lambda}\psi(X)\psi(X)^\top) \\
                            &\quad - 2 \tr \E (\Sigma_\phi^{-\lambda}\phi(X')\psi(X')^\top \Sigma_\psi^{-\lambda}\psi(X)\phi(X)^\top ) \\
    &= \tr(\Sigma_{\phi}^{-\lambda} \Sigma_{\phi} \Sigma_{\phi}^{-\lambda} \Sigma_{\phi})
+ \tr(\Sigma_{\psi}^{-\lambda} \Sigma_{\psi} \Sigma_{\psi}^{-\lambda} \Sigma_{\psi}) - 2\tr(\Sigma_{\phi}^{-\lambda} \Sigma_{\phi \psi}\Sigma_{\psi}^{-\lambda} \Sigma_{\phi \psi}^\top).
 \end{align*}
\end{proof}

\subsection{$\gulp{}$ is a distance, proof of Theorem~\ref{thm:gulp_dist}}\label{app:distance-proofs}

We complete the proof of Theorem~\ref{thm:gulp_dist} by characterizing when the $\gulp{}$ distance is zero in the following lemma.

\begin{lem}[Characterization for when $\gulp{}$ is zero, for $\lambda > 0$]\label{lem:when-gulp-is-zero}
For any $\lambda > 0$, the two representation maps $\phi : \R^d \to \R^k,\psi : \R^d \to \R^l$ have zero $\gulp$ distance, $d_{\lambda}(\phi,\psi) = 0$, if and only if $k = l$  andthere exists an orthogonal transformation $U \in \mathbb{R}$ such that $\phi(X) = U \psi(X)$ a.s.
\end{lem}
\begin{proof}[Proof of Lemma~\ref{lem:when-gulp-is-zero}]
In the main text it was shown that if $\phi$ and $\psi$ are related by an orthogonal transformation, then $d_{\lambda}(\phi,\psi) = 0$. It remains to prove the converse direction, which is more involved. Define $\tilde{\phi}(x) = (\Sigma_{\phi} + \lambda I)^{-1/2} \phi(x)$ and $\tilde{\psi}(x) = (\Sigma_{\psi} + \lambda I)^{-1/2} \psi(x)$. We make the following claim, whose proof we defer:
\begin{claim}\label{claim:whitened-orthogonal}
Let $\lambda > 0$ and suppose $d_{\lambda}(\phi,\psi) = 0$. Then $k = l$ and there is an orthogonal transformation $U \in \R^{k \times k}$ such that $\tilde{\phi}(X) = U \tilde{\psi}(X)$ almost surely.
\end{claim}
Let $U \in \R^{k \times k}$ be the orthogonal transformation guaranteed by the above claim. We can write 
\begin{align*}
\Sigma_{\phi} &= \E[\phi(X) \phi(X)^{\top}] \\
&= (\Sigma_{\phi} + \lambda I)^{1/2} U (\Sigma_{\psi} + \lambda I)^{-1/2} \E[\psi(X) \psi(X)^{\top}] (\Sigma_{\psi} + \lambda I)^{-1/2} U^{\top} (\Sigma_{\phi} + \lambda I)^{1/2} \\
&= (\Sigma_{\phi} + \lambda I)^{1/2} U (\Sigma_{\psi} + \lambda I)^{-1/2} \Sigma_{\psi} (\Sigma_{\psi} + \lambda I)^{-1/2} U^{\top} (\Sigma_{\phi} + \lambda I)^{1/2}.
\end{align*}
Since $\Sigma_{\phi}$ and $(\Sigma_{\phi} + \lambda I)^{1/2}$ commute, and similarly for $\Sigma_{\psi}$ and $(\Sigma_{\psi} + \lambda I)^{1/2}$, we have
\begin{align*}
\Sigma_{\phi} (\Sigma_{\phi} + \lambda I)^{-1} =  U \Sigma_{\psi} (\Sigma_{\psi} + \lambda I)^{-1} U^{\top}.
\end{align*}
Write the SVDs $\Sigma_{\phi} = V_{\phi} D_{\phi} V_{\phi}^{\top}$ and $\Sigma_{\psi} = V_{\psi} D_{\psi} V_{\psi}^{\top}$. Then
\begin{align}\label{eq:sylvesters-long}
D_{\phi} ( D_{\phi} + \lambda I)^{-1} V_{\phi}^{\top} U V_{\psi}  =  V_{\phi}^{\top} U V_{\psi} D_{\psi} (D_{\psi} + \lambda I)^{-1}.
\end{align}
Define the diagonal matrices $\Lambda_{\phi} = D_{\phi} ( D_{\phi} + \lambda I)^{-1}$ and $D_{\psi} (D_{\psi} + \lambda I)^{-1}$, and define the orthogonal matrix $M = V_{\phi}^{\top} U V_{\psi}$. Equation \eqref{eq:sylvesters-long} is a homogeneous Sylvester equation:
\begin{align*}
\Lambda_{\phi} M = M \Lambda_{\psi}.
\end{align*}
Therefore $(\Lambda_{\psi})_{ii} = (\Lambda_{\phi})_{jj}$ if $M_{ij} \neq 0$. Since $f : \R_+ \to [0, 1]$ defined by $f(x) = \frac{x}{x + \lambda}$ is invertible, this implies that $(D_{\phi})_{ii} = (D_{\psi})_{jj}$ if $M_{ij} \neq 0$. From this it follows that
\begin{align*}
( D_{\phi} + \lambda I)^{-1/2} M (D_{\psi} + \lambda I)^{1/2} = M.
\end{align*}
Plugging in $M$ and rearranging, we obtain
\begin{align*}
U^{\top} V_{\phi}^{\top} ( D_{\phi} + \lambda I)^{-1/2} V_{\phi} U = V_{\psi} (D_{\psi} + \lambda I)^{-1/2} V_{\psi}^{\top},
\end{align*}
which simplifies to
\begin{align*}
U^{\top} (\Sigma_{\phi} + \lambda I)^{-1/2} U = (\Sigma_{\psi} + \lambda I)^{-1/2}.
\end{align*}
By combining this with the guarantee from Claim~\ref{claim:whitened-orthogonal} that $\phi(X) = (\Sigma_{\phi} + \lambda I)^{1/2} U (\Sigma_{\psi} + \lambda I)^{-1/2} \psi(X)$ almost surely, we obtain
\begin{align*}
\phi(X) = U \psi(X),
\end{align*}
almost surely. This shows the converse direction of the theorem.
\end{proof}

We conclude with a proof of the claim.
\begin{proof}[Proof of Claim~\ref{claim:whitened-orthogonal}]
Let $(X_1,\ldots,X_{n},\ldots)$ be an infinite sequence of i.i.d copies of $X$. For each $n$, let $$A_n = [\tilde{\phi}(X_1),\ldots,\tilde{\phi}(X_n)] \in \R^{k \times n}, \quad B_n = [\tilde{\psi}(X_1),\ldots,\tilde{\psi}(X_n)] \in \R^{l \times n}.$$ Since $d_{\lambda}(\phi,\psi) = 0$, by the characterization of $\gulp$ in Lemma~\ref{lem:gulp3} we have $\tilde{\phi}(X)^{\top} \tilde{\phi}(X') = \tilde{\psi}^{\top}(X) \tilde{\psi}(X')$ almost surely, so $A_n^{\top} A_n = B_n^{\top} B_n$ almost surely. Suppose without loss of generality that $l \leq k$. Then by Theorem~7.3.11 of \cite{horn2012matrix}, we can construct a semi-orthogonal $U_n \in \R^{l \times k}$ such that $A_n = U_n B_n$ almost surely. Define the event $$E_1 = \{A_n = U_n B_n \mbox{ for all } n \geq 1\}.$$ Taking a union bound over countably many $n$, we see that $E_1$ holds almost surely.

Define $W = \mathrm{span}\{\tilde{\psi}(X_i)\}_{i=1}^{\infty}$. We claim that there is a deterministic vector space $V \subseteq \R^l$ such that $W = V$ almost surely. Let $W'$ be an independent copy of $W$. Then $W \stackrel{d}{=} W + W'$. For any $i \in \{0,\ldots,k\}$, $$\PP[\mathrm{rank}(W) \leq i] = \PP[\mathrm{rank}(W + W') \leq i] \leq \PP[\mathrm{rank}(W) \leq i] - \PP[\mathrm{rank}(W) \leq i, \mbox{ and } W' \not\subseteq W].$$ We conclude that $\PP[\mathrm{rank}(W) \leq i, \mbox{ and } W' \not\subseteq W] = 0$ for all $i$, so $\PP[W' \not\subseteq W] = 0$ for the two independent copies. Therefore $W$ is deterministic, and equals $V$ almost surely.

Let $N = \sup \{n+1 : \mathrm{span}\{\tilde{\psi}(X_1),\ldots,\tilde{\psi}(X_n)\} = \R^l\} \cup \{1\}$. Define the event that $N$ is finite, $$E_2 = \{N < \infty\}.$$ Since we have shown that $\mathrm{span}\{\tilde{\psi}(X_i)\}_{i=1}^{\infty} = V$ almost surely, it follows that $E_2$ holds almost surely.

We now prove that the semi-orthogonal random matrix $U_N \in \R^{k \times l}$ satisfies our conditions.
Under the almost-sure events $E_1$ and $E_2$, we can write $\tilde{\psi}(X_{N+1}) = \sum_{i=1}^N \lambda_i \tilde{\psi}(X_i)$, and it holds that
\begin{align*}
\tilde{\phi}(X_{N+1}) &= U_{N+1} \tilde{\psi}(X_{N+1}) = \sum_{i=1}^N \lambda_i U_{N+1} \tilde{\psi}(X_i) = \sum_{i=1}^N \lambda_i \tilde{\phi}(X_i) = \sum_{i=1}^N \lambda_i U_N \tilde{\psi}(X_i) = U_N \tilde{\psi}(X_{N+1}).
\end{align*}
Since events $E_1$ and $E_2$ hold almost surely, and $X_{N+1}$ is independent of $N$ and $X_1,\ldots,X_N$, 
\begin{align*}\PP[\tilde{\phi}(X) = U_N \tilde{\psi}(X)] &= \PP[\tilde{\phi}(X_{N+1}) = U_N \tilde{\psi}(X_{N+1})] = 1.
\end{align*}
So we conclude that there is a deterministic semi-orthogonal matrix $U \in \R^{k \times l}$ such that $\tilde{\phi}(X) = U \tilde{\psi}(X)$ almost surely. Finally, recall that we have assumed that $\Sigma_{\phi}$ and $\Sigma_{\psi}$ are invertible. Therefore $k = \rank(\Sigma_{\phi}) = \rank(\Sigma_{\tilde{\phi}}) \leq \min(\rank(U), \rank(\Sigma_{\tilde{\psi}})) = \min(\rank(U), \rank(\Sigma_{\psi})) = \min(\rank(U), l)$. We conclude that $k = l$, and $U \in \R^{k \times k}$ is an orthogonal transformation.
\end{proof}

For $\lambda = 0$, we also characterize when the $\gulp$ distance is zero. Since $\gulp$ corresponds to the $\cca$ distance, with slightly different normalization, this is also a characterization of when the $\cca$ distance is zero.
\begin{lem}[Characterization for when $\gulp{}$ is zero, for $\lambda = 0$]\label{lem:when-cca-is-zero}
If $\lambda = 0$, the two representation maps $\phi : \R^d \to \R^k$ and $\psi : \R^d \to \R^l$ have zero $\gulp$ distance, $d_0(\phi,\psi) = 0$, if and only if $k = l$ and there exists an invertible linear transformation $M \in \mathbb{R}^{k \times k}$ such that $\phi(X) = M \psi(X)$ a.s.
\end{lem}
\begin{proof}
For the ``easy'' direction, suppose that $k = l$ and $\phi = M \psi$ for an invertible $M \in \R^{k \times k}$. Then $\Sigma_{\phi} = M \Sigma_{\psi} M^{\top}$ and $\Sigma_{\phi \psi} = M \Sigma_{\psi}$. Using the characterization of $\gulp$ from Proposition~\ref{prop:closed-form}, we obtain
\begin{align*}
d_0^2(\phi,\psi) &= \tr(\Sigma_{\phi}^{-1} \Sigma_{\phi} \Sigma_{\phi}^{-1} \Sigma_{\phi}) + \tr(\Sigma_{\psi}^{-1} \Sigma_{\psi} \Sigma_{\psi}^{-1} \Sigma_{\psi}) - 2 \tr(\Sigma_{\phi}^{-1} \Sigma_{\phi\psi} \Sigma_{\psi}^{-1} \Sigma_{\phi \psi}^{\top}) \\
&= \tr(I_k) + \tr(I_k) - 2 \tr((M^{-1})^{\top} \Sigma_{\psi}^{-1} M^{-1} M \Sigma_{\psi} \Sigma_{\psi}^{-1} \Sigma_{\psi} (M^{-1})^{\top}) \\
&= k + k - 2 \tr(I_k) \\
&= 0.
\end{align*}
For the converse direction, we construct the representations $\tilde{\phi} = \Sigma^{-1/2}_{\phi} \phi$ and $\tilde{\psi} = \Sigma^{-1/2}_{\psi} \psi$. By the characterization of $\gulp$ in Lemma~\ref{lem:gulp3}, the condition $d_0(\phi,\psi) = 0$ implies that $\tilde{\phi}(X)^{\top}\tilde{\phi}(X') = \tilde{\psi}(X)^{\top}\tilde{\psi}(X')$, almost surely over independent $X,X' \sim P_X$. Therefore, analogous reasoning to Claim~\ref{claim:whitened-orthogonal} applies, and implies that $k = l$ and that there is an orthogonal transformation $U$ such that $\tilde{\phi}(X) = U \tilde{\psi}$ almost surely. So $\phi(X) = \Sigma_{\phi}^{1/2} U \Sigma_{\psi}^{-1/2} \psi(X)$, almost surely.
\end{proof}

\subsection{Convergence of plug-in estimator, proof of Theorem~\ref{thm:plug-in-concentration}}\label{app:plug-in-proof}

In order to prove Theorem~\ref{thm:plug-in-concentration}, we first show the following lemma.

\begin{lem}\label{lem:plug-in-concentration}
There is a universal constant $C > 0$, such that for any $B$ such that $\|\phi(X)\|^2, \|\psi(X)\|^2 \leq B$ almost surely, and for any $\lambda > 0$, the plug-in estimator $\hat{d}^2_{\lambda,n}$ converges to the population distance $d_{\lambda}^2$, with the following guarantee for any $t > 0$ and any number of samples $n > 0$,
\begin{align*}
\PP[|\hat{d}^2_{\lambda,n}(\phi,\psi) - d_{\lambda}^2(\phi,\psi)| \geq t + 4B^2/(n \lambda^2)] \leq \exp(-C n t^2 \lambda^4 /B^4) + (k + l)\exp(-C n t^2 \lambda^6 / B^6).
\end{align*}
\end{lem}
\begin{proof}
By the expanding the square and using cyclicity and linearity of the trace, similarly to the proof of Proposition~\ref{prop:closed-form}, the plug-in estimator can alternatively be written as:
\begin{align}\label{eq:d-hat}
\hat{d}^2_{\lambda,n}(\phi,\psi) = \frac{1}{n^2} \sum_{i,j=1}^n ( \phi(X_i)^{\top} (\hat{\Sigma}_{\phi} + \lambda I)^{-1} \phi(X_j) - \psi(X_i)^{\top} (\hat{\Sigma}_{\psi} + \lambda I)^{-1}) \psi(X_j))^2.
\end{align}
For the analysis, also define the plug-in estimator, but with the true covariance matrices, 
\begin{align}\label{eq:d-lambda-tilde}\tilde{d}^2_{\lambda,n}(\phi,\psi) = \frac{1}{n^2} \sum_{i,j=1}^n ( \phi(X_i)^{\top} (\Sigma_{\phi} + \lambda I)^{-1} \phi(X_j) - \psi(X_i)^{\top} (\Sigma_{\psi} + \lambda I)^{-1}) \psi(X_j))^2.\end{align}
We bound the error between the plug-in estimator and the true distance by the triangle inequality:
\begin{align}\label{ineq:triangle-plug-in}
|\hat{d}^2_{\lambda,n}(\phi,\psi) - d_{\lambda}^2(\phi,\psi)| &\leq \underbrace{|\hat{d}^2_{\lambda,n}(\phi,\psi) - \tilde{d}^2_{\lambda,n}(\phi,\psi)|}_{\text{Term 1}} + \underbrace{|\tilde{d}^2_{\lambda,n}(\phi,\psi) - d_{\lambda}^2(\phi,\psi)|}_{\text{Term 2}}.
\end{align}

We bound Term 1 and Term 2 separately, stating our bounds in the following claims.
\begin{claim}[Bound on Term 1]\label{claim:term-1}
Under the conditions of Lemma~\ref{lem:plug-in-concentration}, for any $t > 0$,
\begin{align*}
\PP[|\hat{d}^2_{\lambda,n}(\phi,\psi) - \tilde{d}^2_{\lambda,n}(\phi,\psi)| \geq  t] \leq (k + l)e^{-n t^2 \lambda^6 / (2048 B^6)}
\end{align*}
\end{claim}
\begin{proof}For any $i,j \in [n]$, define $\hat{T}_{ij,\phi} = \phi(X_i)^{\top} (\hat{\Sigma}_{\phi} + \lambda I)^{-1} \phi(X_j)$ and $T_{ij,\phi} = \phi(X_i)^{\top} \Sigma_{\phi} + \lambda I)^{-1} \phi(X_j)$. We have
\begin{align*}
|\hat{T}_{ij,\phi} - T_{ij,\phi}|
&\leq B \|(\hat{\Sigma}_{\phi} + \lambda I)^{-1} - (\Sigma_{\phi} + \lambda I)^{-1}\|,
\end{align*}
and
\begin{align*}
|\hat{T}_{ij,\phi}|, |T_{ij,\phi}| \leq B \|(\hat{\Sigma}_{\phi} + \lambda I)^{-1}\| \leq B / \lambda.
\end{align*}
Analogous definitions and inequalities hold if we replace $\phi$ by $\psi$.
Therefore,
\begin{align*}
|\hat{d}^2_{\lambda,n}(\phi,\psi) - &\tilde{d}^2_{\lambda,n}(\phi,\psi)| \\
&= |\frac{1}{n^2} \sum_{i,j=1}^n (\hat{T}_{ij,\phi} - \hat{T}_{ij,\psi})^2 - (T_{ij,\phi} - T_{ij,\psi})^2| \\
&= |\frac{1}{n^2} \sum_{i,j=1}^n (\hat{T}_{ij,\phi} - \hat{T}_{ij,\psi} - T_{ij,\phi} + T_{ij,\psi})(\hat{T}_{ij,\phi} - \hat{T}_{ij,\psi} + T_{ij,\phi} - T_{ij,\psi})| \\
&\leq 4B^2(\|(\hat{\Sigma}_{\phi} + \lambda I)^{-1} - (\Sigma_{\phi} + \lambda I)^{-1}\| + \|(\hat{\Sigma}_{\psi} + \lambda I)^{-1} - (\Sigma_{\psi} + \lambda I)^{-1}\|) / \lambda.
\end{align*}
So the bound on Term 1 follows from combining with the following technical claim:
\begin{claim}\label{claim:plug-in-technical-matrix-hoeffding}
For any $t > 0$,
\begin{align}\label{eq:sigma-phi-lambda-I-inv-operator}
\PP[\|(\hat{\Sigma}_{\phi} + \lambda I)^{-1} - (\Sigma_{\phi} + \lambda I)^{-1}\| \geq t] \leq k e^{-n t^2 \lambda^4 / (32B^2)}.
\end{align}
\begin{align}\label{eq:sigma-psi-lambda-I-inv-operator}
\PP[\|(\hat{\Sigma}_{\psi} + \lambda I)^{-1} - (\Sigma_{\psi} + \lambda I)^{-1}\| \geq t] \leq l e^{-n t^2 \lambda^4 / (32B^2)}.
\end{align}
\end{claim}
\end{proof}

\begin{proof}[Proof of Claim~\ref{claim:plug-in-technical-matrix-hoeffding}]
We prove the claim for $\phi$, since the reasoning for $\psi$ is analogous. First, let us prove that $\hat{\Sigma}_{\phi}$ concentrates around $\Sigma_{\phi}$ in operator norm. For each $i \in [n]$, let $Z_i = \frac{1}{n}\left(\phi(X_i) \phi(X_i)^{\top} - \Sigma_{\phi}\right)$, which is self-adjoint, satisfies $\E[Z_i] = 0$ and has operator norm bounded by $\|Z_i^2\| \leq \frac{1}{n^2}\left(2\|\phi(X_i) \phi(X_i)^{\top}\|^2 + 2\|\Sigma_{\phi}\|^2\right) \leq 4B^2/n^2$ almost surely. So applying the matrix Hoeffding inequality (Theorem~1.3 of \cite{tropp2012user}) to $\hat{\Sigma}_{\phi} = \sum_{i=1}^n Z_i$, we have, for any $t > 0$,
\begin{align*}
\PP[\|\hat{\Sigma}_{\phi}- \Sigma_{\phi}\| \geq t] \leq k e^{-t^2 n / (32 B^2)}.
\end{align*}
Now let us show that $(\hat{\Sigma}_{\phi} + \lambda I)^{-1}$ concentrates to $(\Sigma_{\phi} + \lambda I)^{-1}$ in operator norm. Since $0 \lesssim \hat{\Sigma}_{\phi}, \Sigma_{\phi}$, for any $v \in \R^{k}$, we have
\begin{align*}
\|((\hat{\Sigma}_{\phi} + \lambda I)^{-1} - (\Sigma_{\phi} + \lambda I)^{-1})v\| &\leq \frac{1}{\lambda} \|(I - (\hat{\Sigma}_{\phi} + \lambda I)(\Sigma_{\phi} + \lambda I)^{-1})v\| \\
&= \frac{1}{\lambda} \|((\hat{\Sigma}_{\phi} - \Sigma_{\phi})(\Sigma_{\phi} + \lambda I)^{-1})v\| \\
&\leq \frac{1}{\lambda^2} \|\hat{\Sigma}_{\phi} - \Sigma_{\phi}\|\|v\|.
\end{align*}
\end{proof}

We now bound the second term in \eqref{ineq:triangle-plug-in}.

\begin{claim}[Bound on Term 2]\label{claim:term-2}
Under the conditions of Lemma~\ref{lem:plug-in-concentration}, for any $t > 0$,
\begin{align*}
\PP[|\tilde{d}^2_{\lambda,n}(\phi,\psi) - d_{\lambda}^2(\phi,\psi)| \geq 4B^2 / (n \lambda^2) + t] \leq \exp(-t^2 \lambda^4 n / ( 8 B^4)).
\end{align*}
\end{claim}
\begin{proof}
Write $\tilde{d}^2_{\lambda,n}(\phi,\psi) = \sum_{i,j=1}^n s_{ij}$, where $$s_{ij} = \frac{1}{n^2} ( \phi(X_i)^{\top} (\Sigma_{\phi} + \lambda I)^{-1} \phi(X_j) - \psi(X_i)^{\top} (\Sigma_{\psi} + \lambda I)^{-1}) \psi(X_j))^2$$ is the $i,j$ term in the sum. Since $\|(\hat{\Sigma}_{\phi} + \lambda I)^{-1}\|, \|(\hat{\Sigma}_{\psi} + \lambda I)^{-1}\| \leq 1/\lambda$, and $\|\phi(X_i)\|^2, \|\psi(X_i)\|^2 \leq B$, we have almost surely $$|s_{ij}| \leq \frac{4B^2}{n^2 \lambda^2}.$$ Furthermore, term $s_{ij}$ only depends on $X_i$ and $X_j$. Therefore, by McDiarmid's inequality,
\begin{align}\label{eq:plug-in-helper-1}
\PP[|\tilde{d}^2_{\lambda,n}(\phi,\psi) - \E[\tilde{d}^2_{\lambda,n}(\phi,\psi)]| \geq t] \leq \exp(-t^2 \lambda^4 n / ( 8 B^4)),
\end{align}
where we have used that $|\sum_{j=1}^n s_{ij}| \leq 4B^2 / (n \lambda^2)$ for each $i$.
Finally, we bound the difference between $\tilde{d}^2_{\lambda,n}$ and $d_{\lambda}^2$ in expectation over the samples. Notice that if $i \neq j$ we have $\E[s_{ij}] = d_{\lambda}^2(\phi,\psi) / n^2$. So the only terms that can add bias are the diagonal terms $s_{ii}$, so
\begin{align}\label{eq:plug-in-helper-2}
|d_{\lambda}^2(\phi,\psi) - \E[\tilde{d}^2_{\lambda,n}(\phi,\psi)]| \leq \sum_{i=1}^n |s_{ii}| \leq 4B^2 / (n \lambda^2)
\end{align}
Combining \eqref{eq:plug-in-helper-1} and \eqref{eq:plug-in-helper-2} proves the claim.
\end{proof}

Combining Claims~\ref{claim:term-1} and \ref{claim:term-2} with the triangle inequality \eqref{ineq:triangle-plug-in} proves Lemma~\ref{lem:plug-in-concentration}.

\end{proof}

Theorem~\ref{thm:plug-in-concentration} is now a simple consequence of Lemma~\ref{lem:plug-in-concentration}.

\begin{proof}[Proof of Theorem~\ref{thm:plug-in-concentration}]
Under the conditions of Theorem~\ref{thm:plug-in-concentration}, we have $\|\phi(X)\|^2, \|\psi(X)\|^2 \leq 1$ almost surely and $\lambda \in (0,1)$. For any $t > 0$, Lemma~\ref{lem:plug-in-concentration} implies 
\begin{align*}
\PP[|\hat{d}^2_{\lambda,n}(\phi,\psi) - d_{\lambda}^2(\phi,\psi)| \geq t + 4/(n \lambda^2)] \leq \exp(-C n t^2 \lambda^4) + (k + l)\exp(-C n t^2 \lambda^6).
\end{align*}
Let $0 < \delta \leq 1$ and let $t = \frac{2}{C\lambda^3} \sqrt{\frac{\log((k+l)/\delta)}{n}}$. Then 
$$\PP[|\hat{d}_{\lambda,n}^2(\phi,\psi) - d_{\lambda}^2(\phi,\psi)| \geq t + 4/(n\lambda^2)] < \delta / 2 + \delta / 2 = \delta\,.$$
Finally, since $\lambda \in (0,1)$ we have $$\frac{1}{\lambda^3} \sqrt{\frac{\log((k+l)/\delta)}{n}} \gtrsim t + 4 / (n\lambda^2),$$ which proves the theorem.
\end{proof}

\subsection{Transfer learning distance under kernel ridge regression} \label{app:kernel_ridge_reg}
Consider comparing the predictors output by kernel ridge regression with some kernel $K(x,y)=\langle\tau(x),\tau(y) \rangle$, applied to different representations. This corresponds to the case $\mathcal{F}=\{f_\beta(\cdot): \: f_{\beta}(x) = \beta^{\top} \tau(x)\}$ and $r(f_{\beta}) = ||\beta||_2^2 $. Although $\tau$ may be high or even infinite dimensional, we now show that computing $\gulp{}$ under this $\mathcal{F}$ requires only access to $K(\cdot, \cdot)$, and not $\tau$ directly. 

This is equivalent to defining new representations $\phi' = \tau \circ \phi$ and $\psi' = \tau \circ \phi$, and computing $d_\phi(\phi', \psi')$. However, $\tau$ may be high or even infinite-dimensional; traditionally in kernel ridge regression, one only wishes to compute $K(\cdot, \cdot)$ but never $\tau$ explicitly. Here, we show that  $d_\lambda(\phi, \psi)$ is computable in terms of only inner products $\langle \phi(x), \phi(y) \rangle$ and $\langle \psi(x), \psi(y) \rangle$, or put differently, that $d_\lambda(\phi, \psi)$ can be written in terms of only the kernel functions associated with $\phi$ and $\psi$. By applying this result to $\phi'$ and $\psi'$, this implies we only need to access $\langle \phi'(x), \phi'(y) \rangle = \langle \tau(\phi(x)), \tau(\phi(y)) \rangle = K(\phi(x), \phi(y))$.

Recall that $d_\lambda(\phi, \psi)^2 = \tr((\Sigma_{\phi} + \lambda I)^{-1} \Sigma_{\phi} (\Sigma_{\phi} + \lambda I)^{-1} \Sigma_{\phi})  + \tr((\Sigma_{\psi} + \lambda I)^{-1} \Sigma_{\psi} (\Sigma_{\psi} + \lambda I)^{-1} \Sigma_{\psi})  - 2\tr((\Sigma_{\phi} + \lambda I)^{-1} \Sigma_{\phi \psi} (\Sigma_{\psi} + \lambda I)^{-1} \Sigma_{\phi \psi}^\top)$.
We prove the result for the finite sample case discussed in \ref{sec:plug-in}, where we approximate $\Sigma_\phi = VV^{\top}$, $\Sigma_\psi = WW^{\top}$. Here, $V$ consists of all the samples $\phi(x)$, with number of columns equal to the number of samples. 
By the kernel trick, $(\Sigma_\phi + \lambda I)^{-1}\Sigma_\phi = (VV^{\top}+\lambda I)^{-1}VV^{\top} = V(V^{\top}V+\lambda I)^{-1}V^{\top}$. Thus:
\begin{align*}
    \tr((\Sigma_{\phi} + \lambda I)^{-1} \Sigma_{\phi} (\Sigma_{\phi} + \lambda I)^{-1} \Sigma_{\phi}) &= \tr(V(V^{\top}V+\lambda I)^{-1}V^{\top}V(V^{\top}V+\lambda I)^{-1}V^{\top}) \\
    &= \tr((V^{\top}V+\lambda I)^{-1}V^{\top}V(V^{\top}V+\lambda I)^{-1}V^{\top}V)
\end{align*}
This term is expressible in terms of only $(V^{\top}V)_{ij}$, which only depends on $\langle \phi(x_i),\phi(x_j)\rangle$ for samples $x_i$ and $x_j$. Similar reasoning holds for the term $\tr((\Sigma_{\psi} + \lambda I)^{-1} \Sigma_{\psi} (\Sigma_{\psi} + \lambda I)^{-1} \Sigma_{\psi})$. Finally, consider the cross-term:
\begin{align*}
    \tr((\Sigma_{\phi} + \lambda I)^{-1} \Sigma_{\phi \psi} (\Sigma_{\psi} + \lambda I)^{-1} \Sigma_{\phi \psi}^\top) &= \tr((VV^{\top}+\lambda I)^{-1}VW^{\top}(WW^{\top}+\lambda I)^{-1}WV^{\top}) \\
    &= \tr(V(V^{\top}V+\lambda I)^{-1}W^{\top} W(W^{\top}W+\lambda I)^{-1}) \\
    &= \tr((V^{\top}V+\lambda I)^{-1}V^{\top}V(W^{\top}W+\lambda I)^{-1}W^{\top}W)
\end{align*}
Again, this term is expressible only in terms of $V^{\top}V$ and $W^{\top}W$.

\section{Supplementary experiments}

\subsection{Experimental Setup}\label{app:experimental_setup}
Here we briefly describe all of the network architectures used in this paper as well as the procedure for training them. All experiments were run on Nvidia Volta V100 GPUs. 

\paragraph{Networks on MNIST} For the MNIST handwritten digit database~\cite{deng2012mnist}, we initialize 400 fully-connected networks with ReLU activations. Each networks accepts a flattened $28 \times 28$ image (784 grayscale pixels) as input and outputs at its last layer a vector of 10 probabilities for a given digit 1-10. The number of hidden layers in the networks range from 1 to 10 and the widths of all hidden layers are constant and range from 100 to 1000 in multiples of 100. Each model architecture with a fixed width and depth is randomly initialized 4 separate times with uniform Kaiming initialization~\cite{he2015delving} and zero bias. Every network is trained for 50 epochs and a batch size of 100 on all 60,000 images of the MNIST train set using the Adam optimizer~\cite{kingma2014adam} with a learning rate of $10^{-4}$.

\paragraph{Networks on ImageNet} For the ImageNet Object Localization Challenge~\cite{krizhevsky2012imagenet}, we use 37 state-of-the-art models downloaded both in untrained and pretrained form from the PyTorch database of models\footnote{\href{https://pytorch.org/vision/stable/models.html\#classification}{https://pytorch.org/vision/stable/models.html\#classification}}. All models can be separated into the following classes
\begin{itemize}
    \item \underline{ResNets}: regnet\_x\_16gf, regnet\_x\_1\_6gf, regnet\_x\_32gf, regnetx\_3\_2\_gf, regnet\_x\_400mf, regnet\_x\_800mf, regnet\_x\_8gf, regnet\_y\_16gf, regnet\_y\_1\_6gf, regnet\_y\_32gf, regnet\_y\_3\_2gf, regnet\_y\_400mf, regnet\_y\_800mf, regnet\_y\_8gf, resnet18, resnext50\_32x4d, wide\_resnet50\_2
    
    \item \underline{EfficientNets}: efficientnet\_b0, efficientnet\_b1, efficientnet\_b2, efficientnet\_b3, efficientnet\_b4, efficientnet\_b5, efficientnet\_b6, efficientnet\_b7
    
    \item \underline{MobileNets}: mobilenet\_v2, mobilenet\_v3\_small, mobilenet\_v3\_large
    
    \item \underline{ConvNeXts}: convnext\_base, convnext\_tiny, convnext\_small, convnext\_large
    
    \item \underline{Miscellaneous}: alexnet, googlenet, inception, mnasnet, vgg16
\end{itemize}
All models accept 3-channel RGB images of size $224 \times 224$ (i.e. total dimension $3 \times 224 \times 224$). We normalize the 1,281,119 images in the train set of ImageNet to have mean $(0.485, 0.456, 0.406)$ and standard deviation $(0.229, 0.224, 0.225)$ in each RGB channel. Every models embeds the images into a latent space with dimension ranging from 400 to 4096 depending on the architecture.

\paragraph{Networks on CIFAR}
For CIFAR \cite{krizhevsky2009learning}, we train 16 ResNet18 architectures from independent, random initializations for 50 epochs each using the FFCV library \cite{leclerc2022ffcv}. They were trained with batch size 512, learning rate $0.5$ on a cyclic schedule, momentum parameter $0.9$, and with weight decay parameter $5e-4$. 

\subsection{Relationship of $\gulp{}$ to other distances}\label{app:relationships}

\paragraph{Embeddings of ImageNet} Figure~\ref{fig:imagenet-relationships} of the main text compares the $\cka$, $\cca$, and $\gulp$ distances between pairs of representations of 37 ImageNet representations, estimated from 10,000 samples. In Figure~\ref{fig:imagenet-relationships-extended}, we extend the comparison to $\pwcca$ and $\procrustes$. We note that at certain $\lambda$, our distance has near-linear relationships with $\procrustes$ and $\cka$.

\begin{figure}[h!]
    \centering
    \includegraphics[scale=0.28]{images/cca_vs_gulp_extended_imagenet.pdf}
    \includegraphics[scale=0.28]{images/cka_vs_gulp_extended_imagenet.pdf}
    \includegraphics[scale=0.28]{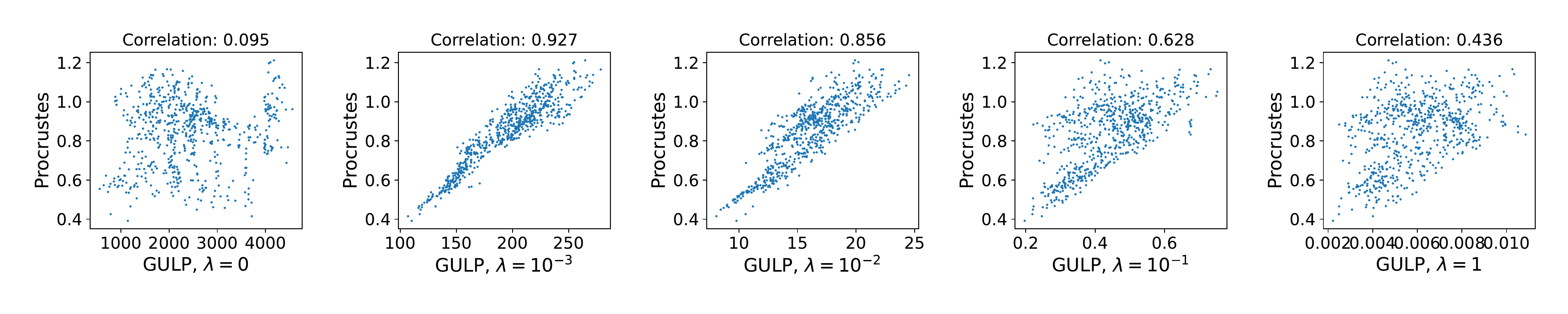}
    \includegraphics[scale=0.28]{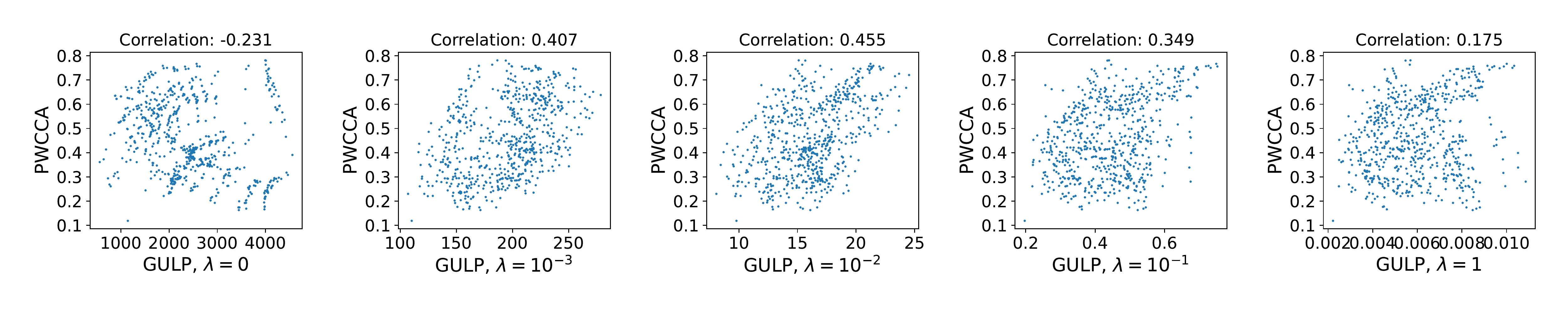}
    \caption{Scatter plots showing relationships between network distances on ImageNet. Each point is a pair of ImageNet representations, and the $x$ and $y$ coordinates correspond to two distances that are being compared. There is a surprising near-linear relationship between $\procrustes$ and $\gulp$ for intermediate $\lambda$. The title of each plot shows the Pearson correlation coefficient.}
    \label{fig:imagenet-relationships-extended}
\end{figure}

\paragraph{Embeddings of MNIST} In Figure~\ref{fig:mnist-relationships-extended}, we repeat the same experiment for MNIST embeddings with trained fully-connected networks of depths in the range from 1 to 10, and widths in $\{200, 400, 600, 800, 1000\}$.

\begin{figure}
    \centering
    \includegraphics[scale=0.28]{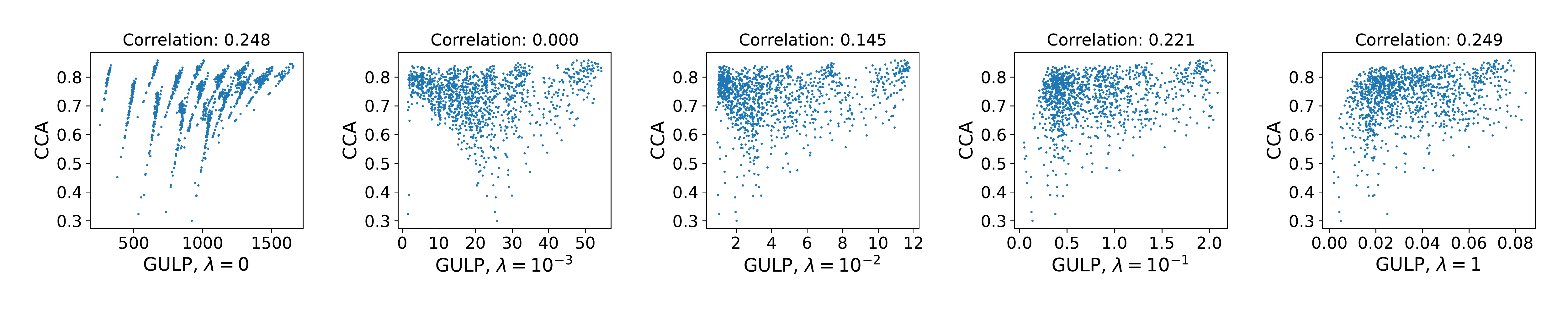}
    \includegraphics[scale=0.28]{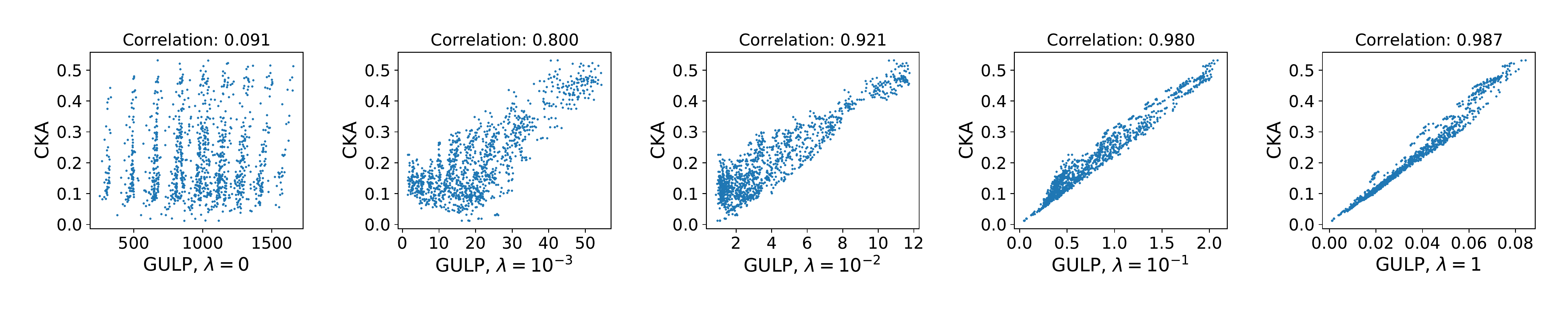}
    \includegraphics[scale=0.28]{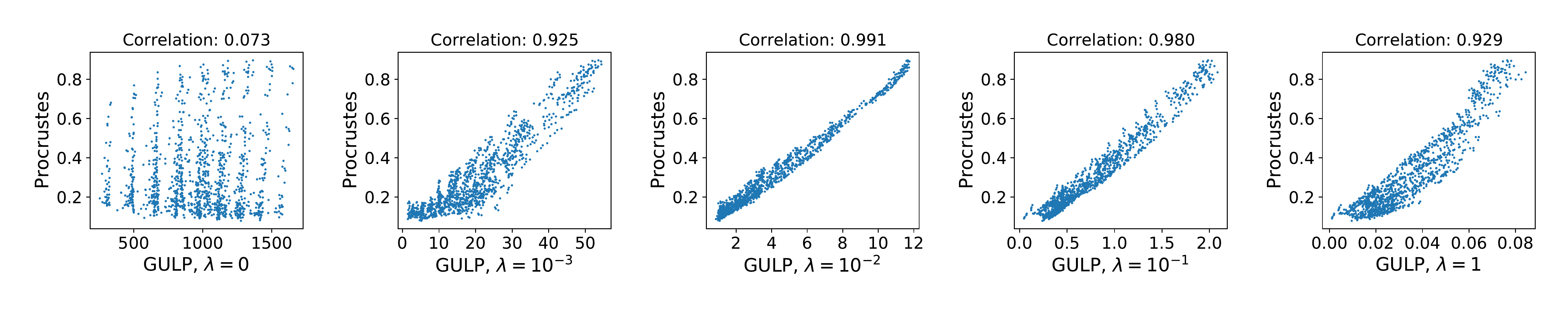}
    \includegraphics[scale=0.28]{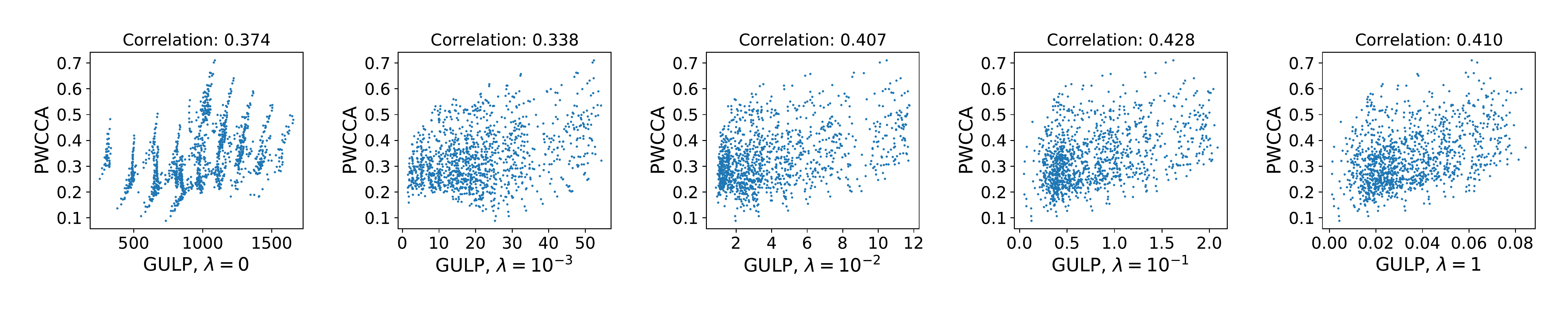}
    \caption{
    Scatter plots showing relationships between network distances of fully-connected network representations on MNIST. For $\lambda = 0$, there is no straight-line relationship with CCA, since the dimensions of the representations differ, and the normalization of CCA is different from that of $\gulp$ because it depends the representation dimension. Each point is a pair of MNIST representations, and the $x$ and $y$ coordinates correspond to two distances that are being compared. The near-linear relationship between CKA and GULP is quite evident for large $\lambda$, as it turns out that all of the kernels are closer to having the same normalization than in the case of the ImageNet dataset. Furthermore, there is a surprising near-linear relationship between CKA and $\procrustes$ for intermediate $\lambda$. The title of each plot shows the Pearson correlation coefficient.}
    \label{fig:mnist-relationships-extended}
\end{figure}

\subsection{Convergence of the plug-in estimator}\label{app:convergence}

In Figure~\ref{fig:convergence}, we estimated the distances between $\binom{15}{2} = 105$ pairs of ImageNet networks with the plug-in estimator as we increased the number of samples $n$. We plotted the average relative error to the 10000-sample estimate. We supplement this result with Figure~\ref{fig:convergence-rel-error-indep}, which shows that for $n \geq 2000$, two independent estimates of $\gulp$ have average relative error smaller than 2\%. Therefore, if there is error in the plug-in estimator it is mainly due to bias, apart from roughly 2\% relative error. Since the convergence in \ref{fig:convergence} indicates that the plug-in estimator is unbiased, this reinforces our claim that the plug-in estimator concentrates quickly around the true distance.

\begin{figure}
    \centering
    \includegraphics[scale=0.5]{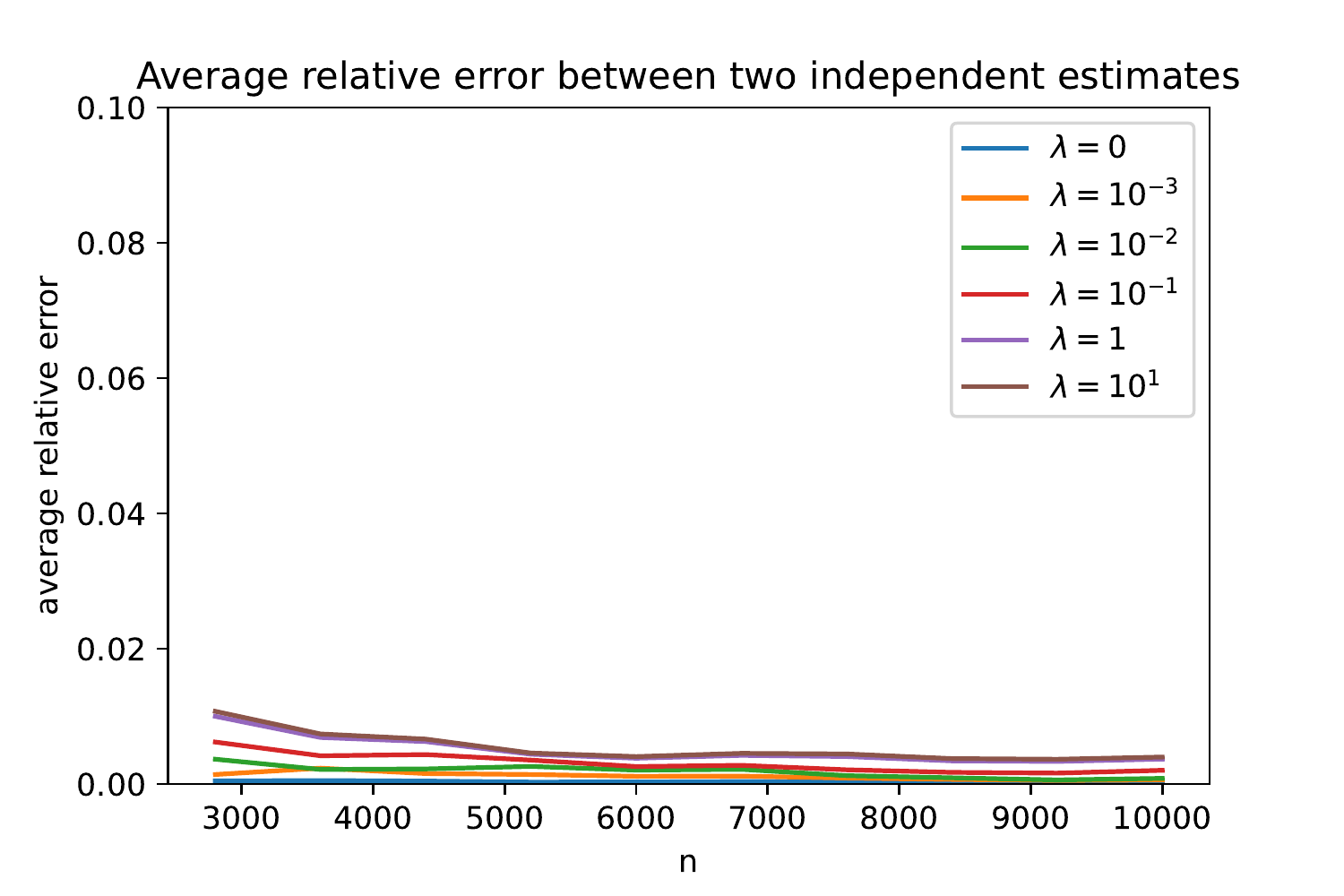}
    \caption{Relative error $|\hat{d}_{\lambda,n}^{(1)} - \hat{d}_{\lambda,n}^{(2)}| / (\hat{d}_{\lambda,n}^{(1)} + \hat{d}_{\lambda,n}^{(2)})$ between plug-in estimator on two trials $\hat{d}^{(1)}_{\lambda,n}$ and $\hat{d}^{(2)}_{\lambda,n}$ with independent samples. We have averaged across the 66 pairs of ImageNet networks. For $\lambda = 0$, due to numerical precision issues we do not plot the relative error in the estimate for $n \leq 2000$.}
    \label{fig:convergence-rel-error-indep}
\end{figure}

\paragraph{Runtime} The 12 ImageNet networks for these plots were alexnet\_pretrained\_rep,  convnext\_small\_pretrained\_rep, efficientnet\_b0\_pretrained\_rep, efficientnet\_b3\_pretrained\_rep,
efficientnet\_b6\_pretrained\_rep, inception\_pretrained\_rep,
mobilenet\_v3\_large\_pretrained\_rep,
regnet\_x\_1\_6gf\_pretrained\_rep,
regnet\_x\_400mf\_pretrained\_rep, regnet\_y\_16gf\_pretrained\_rep,
regnet\_y\_3\_2gf\_pretrained\_rep, regnet\_y\_8gf\_pretrained\_rep, subsampled from the 37 models at our disposal so as to reduce the computational burden. Generating these plots took 11 minutes with an Nvidia Volta V100 GPU. The computational cost is due to the fact that distances are computed for a range of increasing number of samples $n$, on 66 pairs of networks and two independent trials.

\subsection{$\gulp{}$ captures generalization performance by linear predictors}

Here we supplement the experiments of Section \ref{sec:predict}, which show how the $\gulp$ distance captures generalization performance by linear predictors. We provide an experiment on the UTKFace dataset \cite{zhifei2017cvpr} using the age of a face as the regression label, instead of using a random label. We consider the representation maps $\phi_1,\ldots,\phi_m$ given by $m = 37$ pretrained Imagenet image classification architectures, applied to the UTKFace dataset $P_X$. For each pair of representations, we compute the $\cka$, $\cca$, $\pwcca$, and $\gulp$ distances with the plug-in estimator on 10,000 images. We then draw $n = 5000$ data points $(X_i,Y_i) \sim P_X$, where $X_i$ is the face image and $Y_i$ is the corresponding age. The remaining experiment details are the same as in Section~\ref{sec:predict}. For each representation $i \in [m]$ we fit a $\lambda$-regularized least-squares linear regression to the training data $\{(X_k,Y_k)\}_{k \in [n]}$, yielding a coefficient vector $\beta_{\lambda,i}$. Finally, for each $1 \leq i \leq j \leq m$, we compute the distance $\tau_{ij}$ between predictions as an empirical average over 3000 samples in a testset. In Figure~\ref{fig:generalization_utk_face}, we plot the Spearman $\rho$ correlations between the prediction distances $\tau_{ij}$ and the different distances between representations (similarly to Figure~\ref{fig:generalization}). We run one trial, since the labels are no longer random. The $\gulp{}$ distance again performs favorably compared to other methods. For linear regression with $\lambda = 1$ and $\lambda = 10^{-6}$, the $\gulp{}$ distance with $\lambda = 1$ and $\lambda = 10^{-6}$, respectively vastly outperform previously-proposed distances in terms of predicting generalization. For linear regression with $\lambda = 10^{-4}$ and $\lambda = 10^{-2}$, $\gulp{}$ with $\lambda = 10^{-2}$ predicts the generalization performance on par with the $\cka$ and $\procrustes$ distances. Notice that unlike the experiment with random labels, the best $\lambda$ for $\gulp$ does not exactly match the $\lambda$ used in the linear regression task, but instead is close to it.

\begin{figure}
\centering
    \centering
    \includegraphics[width=.48\textwidth]{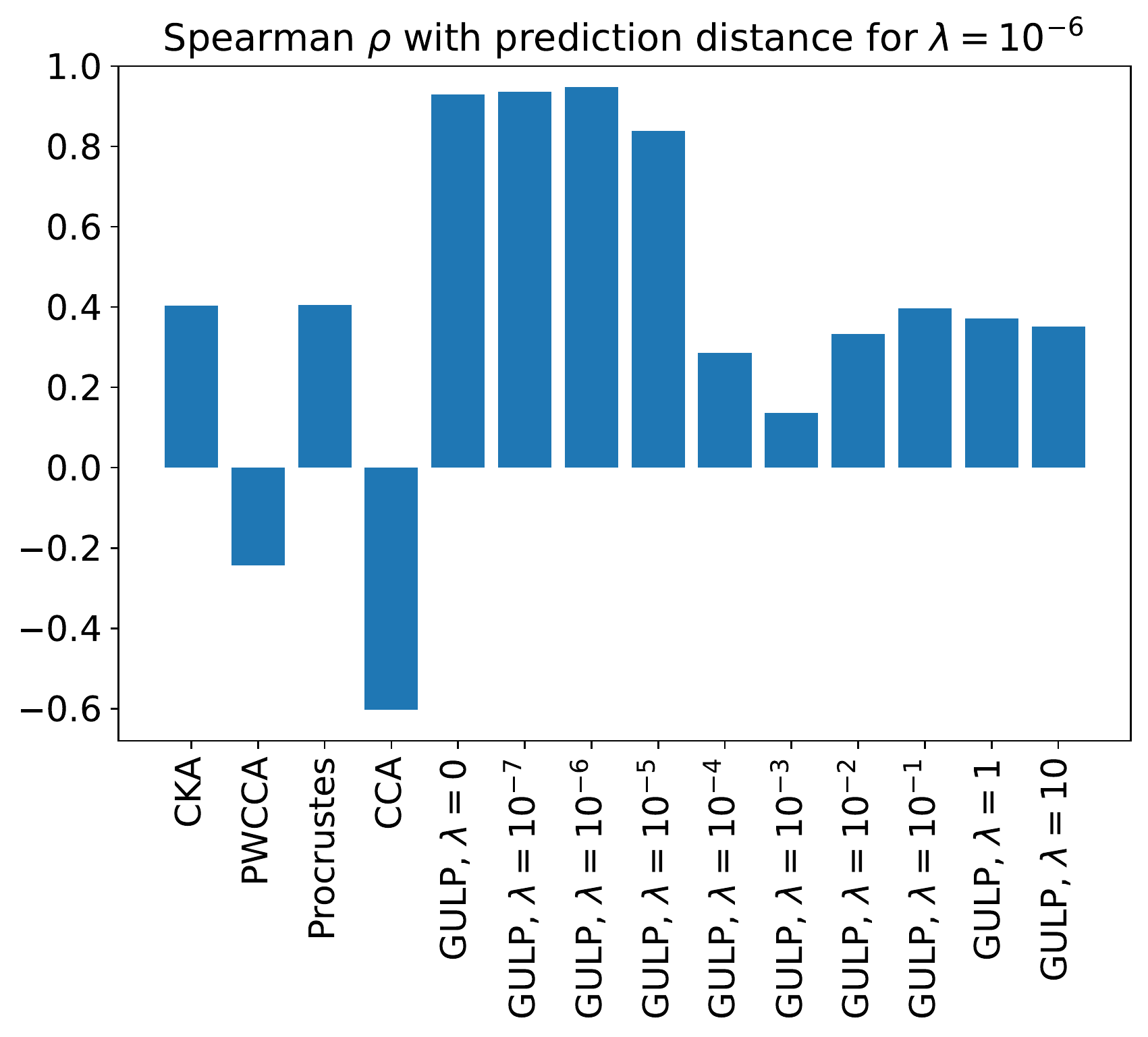}
    \includegraphics[width=.48\textwidth]{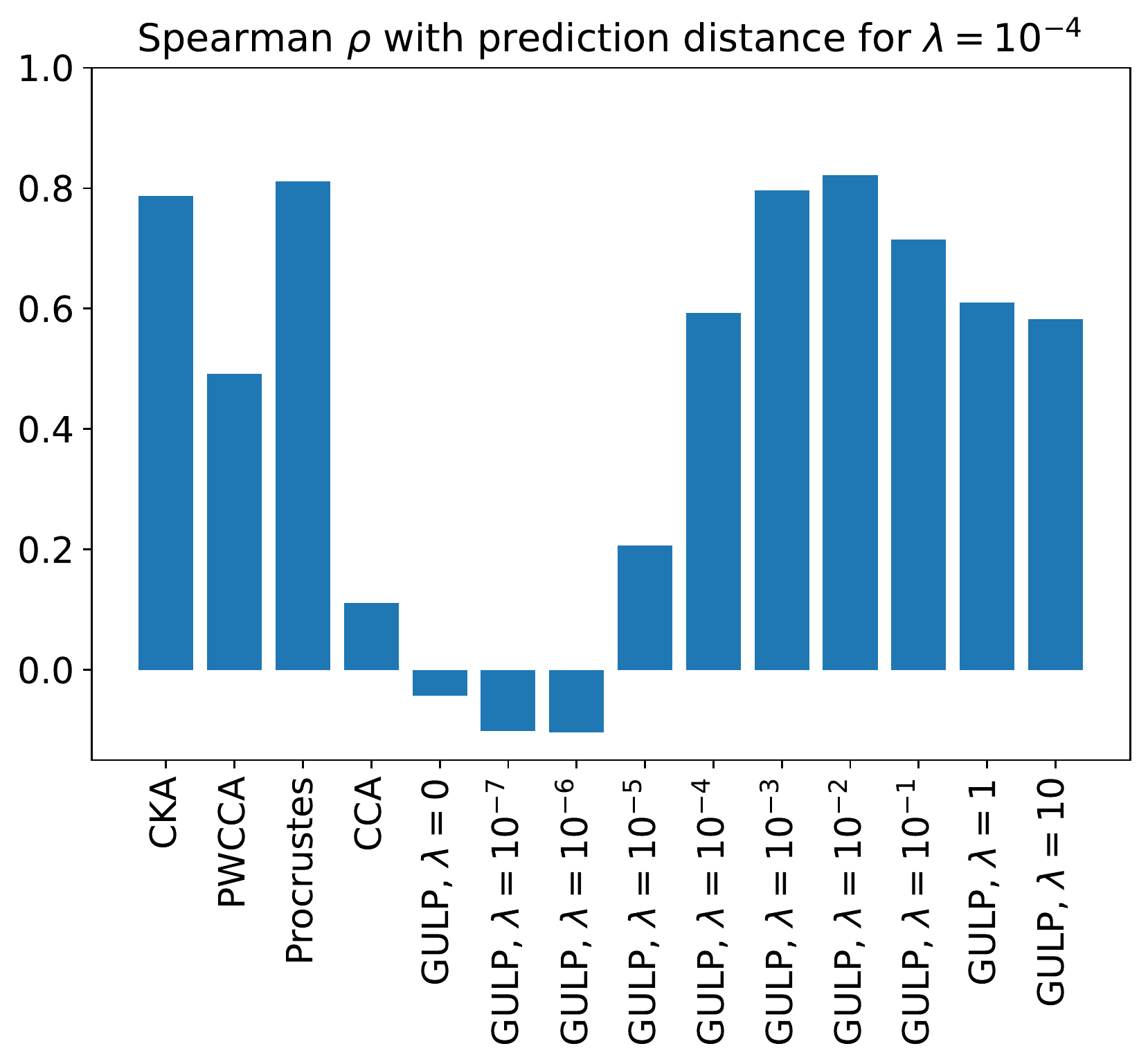}
    \includegraphics[width=.48\textwidth]{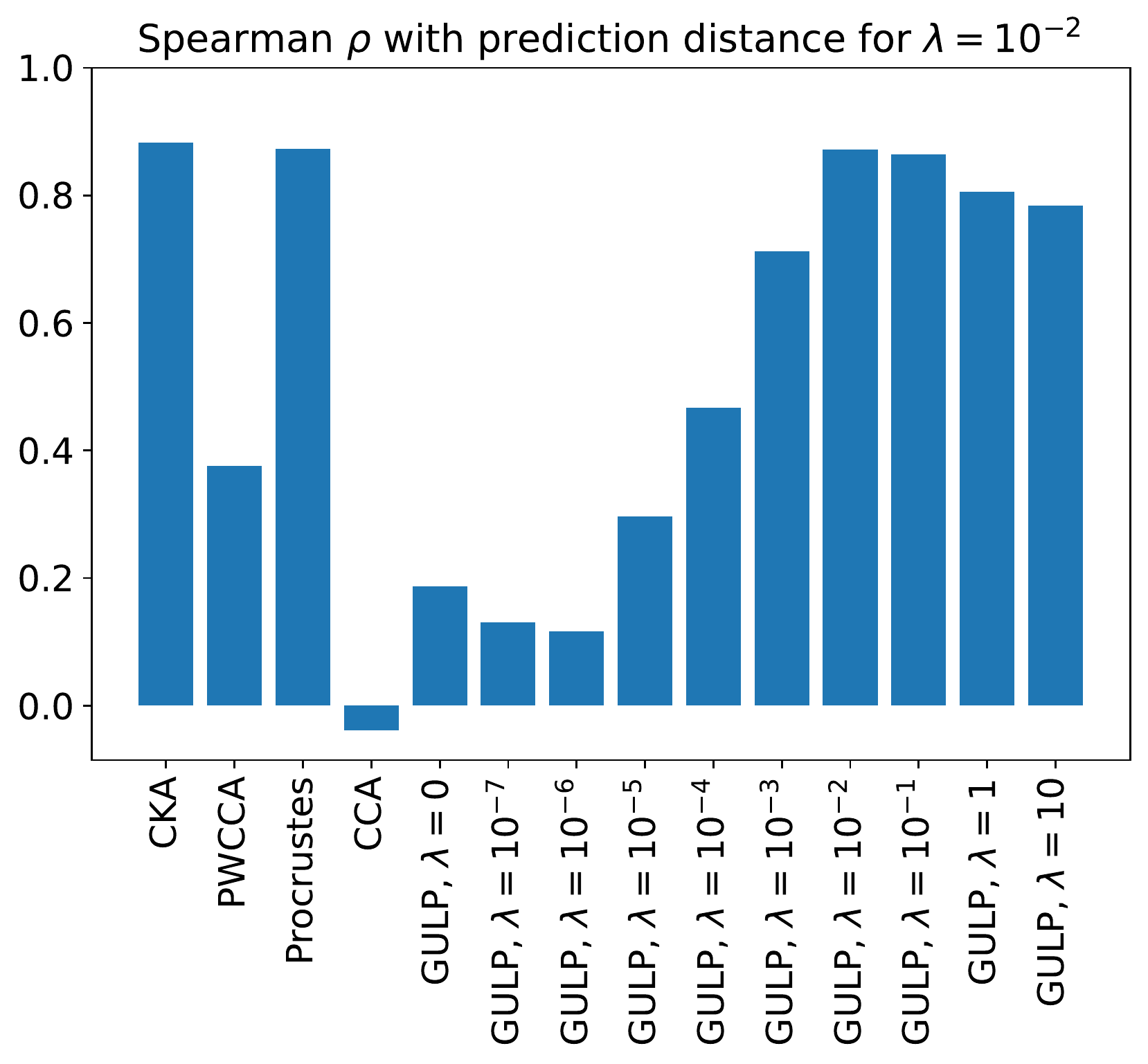}
    \includegraphics[width=.48\textwidth]{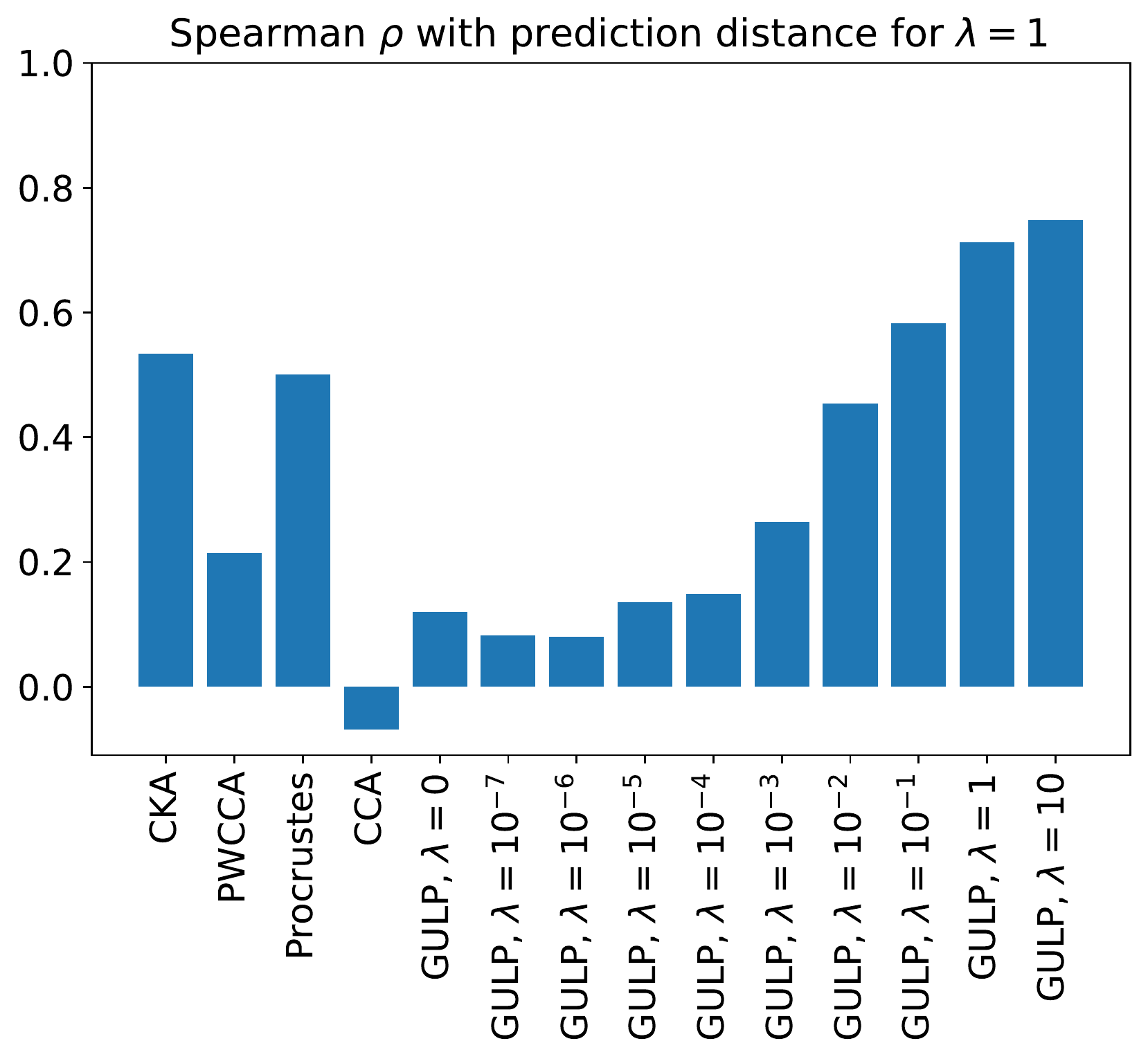}
    \caption{$\gulp$ captures generalization of linear predictors. We plot Spearman's $\rho$ between the differences in predictions by $\lambda$-regularized linear regression, and the different distances.}
    \label{fig:generalization_utk_face}
\end{figure}

\subsection{$\gulp{}$ distances cluster together networks with similar architectures}\label{app:network_experiments}

Here we elaborate further on the experiments described in Section~\ref{subsec:network_experiments} on embeddings of MNIST networks. As described previously, we generate four independent copies of fully-connected ReLU networks with depths ranging from 1-10 and widths ranging from 100-1000. Network depth refers to the number of hidden layers in a model and network width refers to the width of each hidden layer. All networks are fully-trained on MNIST, and their last hidden layer representations are computed on 60,000 input images from the train set. For every pair of widths and depths $(w_1, d_1)$ and $(w_2, d_2)$, there are four trained networks with dimensions $(w_1, d_1)$ and four trained networks with dimensions $(w_2, d_2)$. For a given metric, we compute $4 \cdot (3-1) = 12$ distances between the penultimate layer representations of these networks and average them. This gives us the average distance between the penultimate layer representations of a network with dimensions $(w_1, d_1)$ and a network with dimensions $(w_2, d_2)$. In Figure~\ref{fig:mnist_width_depth} (left) we show the average \pwcca{}, \cka{}, \procrustes{}, and \gulp{} distances between each pair of width-depth architectures for varying $\lambda$. We also display the MDS embeddings of all $4 \times 10 \times 10$ networks colored by width and depth (center and right).

In Figure~\ref{fig:cifar_width_depth} we perform a very similar experiment to the one above with networks trained on CIFAR10 instead of MNIST. We generate five independent copies of fully-connected ReLU networks with depths ranging from 1-5 and widths ranging from 200-1,000. All networks are fully-trained on 60,000 images of the CIFAR10 train set by SGD in the maximal-update parametrization \cite{yang2020feature}, where for a width-one network our hyperparameters would be learning rate $\eta = 0.1$ and we would initialize weights and biases as Gaussian with standard deviation 1. The distances between their penultimate layer representations are computed using 10,000 randomly selected CIFAR10 images. Figure~\ref{fig:cifar_width_depth} shows the average \pwcca{}, \cka{}, \procrustes{}, and \gulp{} distances between each pair of width-depth architectures and show the MDS embeddings of all $5 \times 5 \times 5$ networks colored by width and depth (center and right).

Now we describe in more detail how various distance metrics cluster state-of-the-art network architectures on the ImageNet Object Localization Challenge. In Figure~\ref{fig:imagenet_pretrained_variance} (left) we compute the \cca{}, \pwcca{}, \cka{}, \procrustes{}, and \gulp{} distances for five groups of networks: 17 ResNets, 8 EfficientNets, 4 ConvNeXts, and 3 MobileNets. These 32 networks are fully-trained on ImageNet and are given the same 10,000 input training images to form their last hidden layer representations. As discussed in Section~\ref{subsec:network_experiments}, all distance metrics separate ResNet architectures (blue) from the EfficientNet and ConvNeXt convolutional networks (orange and red) with \gulp{} at $\lambda=1$ achieving the best separation between these two clusters. %
To further quantify the compactness of the clusterings given by these distance metrics, we compute a standard deviation ratio for each of the five network classes. Given a distance metric, this ratio is computed as the sum of squared distances between all 36 networks divided by the sum of squared distances between networks in each class:
\begin{equation}\label{eq:std_ratio}
    \text{standard deviation ratio for class } k = \Big(\frac{1}{n(n-1)}\sum_{1 \leq i \neq j \leq n}d_{ij}^2 \Big/ \frac{1}{|\cC_k|(|\cC_k|-1)}\sum_{i \neq j \in \cC_k}d_{ij}^2\Big)^\frac{1}{2}
\end{equation}
where $n = 36$ and $C_k \subset \{1, \dots, n\}$ is the subset of networks in class $k = 1, \dots, 5$. Note that a ratio of 1 implies that the size of the cluster is equal to the average distance between any two ImageNet networks. In Figure~\ref{fig:imagenet_pretrained_variance} (right) we plot the standard deviation ratio for each of the five network classes. As expected, the ratios under the \gulp{} distance increase for large $\lambda$ and the residual and convolutional network architectures become well separated at $\lambda = 1$. The \cca{}, \pwcca{}, \cka{}, and \procrustes{} distances do not achieve the same level of separation between different network architectures but are similar to the \gulp{} distance at $\lambda = 10^{-2}$.

Now we study distances between the same ImageNet models when they are untrained and are at random initialization. Again there are 32 untrained networks consisting of 17 ResNets, 8 EfficientNets, 4 ConvNeXts, and 3 MobileNets. Each of the untrained networks is randomly initialized ten separate times and is given the same 10,000 input training images from ImageNet. We compute the \cka{}, \procrustes{}, and \gulp{} distances between their penultimate layer representations which are displayed in Figure~\ref{fig:imagenet_untrained_variance} (left). The distances between these networks are visualized using a two-dimensional t-SNE embedding and the standard deviation ratio~\eqref{eq:std_ratio} of each of the four groups is calculated [Figure~\ref{fig:imagenet_untrained_variance} (center and right)]. Under all distance metrics we see that the ResNets (blue), EfficientNets (orange), and ConvNeXts (red) all form their own clusters. As evidenced by the standard deviation ratios, the ConvNeXt networks under the $\gulp{}$ distance form a tighter cluster as $\lambda$ increases. Both \cka{} and \gulp{} with $\lambda = 1$ achieve the most compact clusterings of ResNets, EfficientNets, and ConvNeXts.

In Figure~\ref{fig:imagenet_variance_networks} for several distance metrics we display the standard deviation ratios for the five network groups before and after training. On untrained and pretrained networks, \cka{} and \procrustes{} are competitive with \gulp{} at clustering ResNet, EfficientNet, and ConvNeXt architectures. However on ConvNeXt models, for untrained networks \gulp{} achieves the highest standard deviation ratio with large $\lambda$ and for pretrained networks it achieves the highest standard deviation ratio at intermediate values of $\lambda$.

\begin{figure}
\centering
\includegraphics[width=\linewidth]{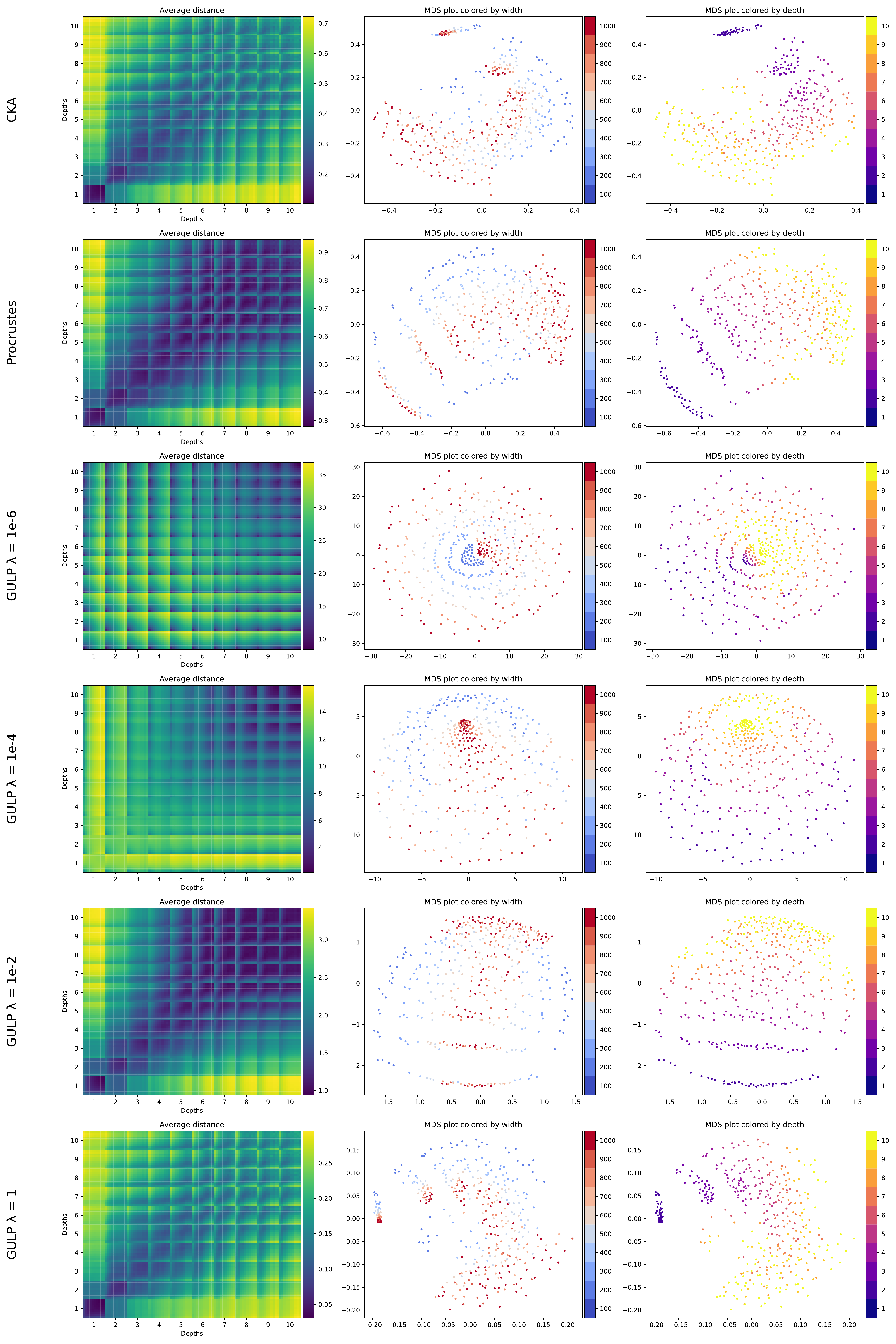}
\caption{Average \cka{}, \procrustes{}, and \gulp{} distance between last hidden layer representations of two fully-connected ReLU networks with a given width and depth (left). Networks are fully-trained on MNIST and penultimate layer representations are constructed from 60,000 input train images. Ordering of networks along rows and columns of distance matrices has outer indices as network depths 1-10 and inner indices as network widths 100-1000. Two dimensional MDS embedding plots (center and right) of all networks colored by architecture width and depth.}
\label{fig:mnist_width_depth}
\end{figure}

\begin{figure}
\centering
\includegraphics[width=\linewidth]{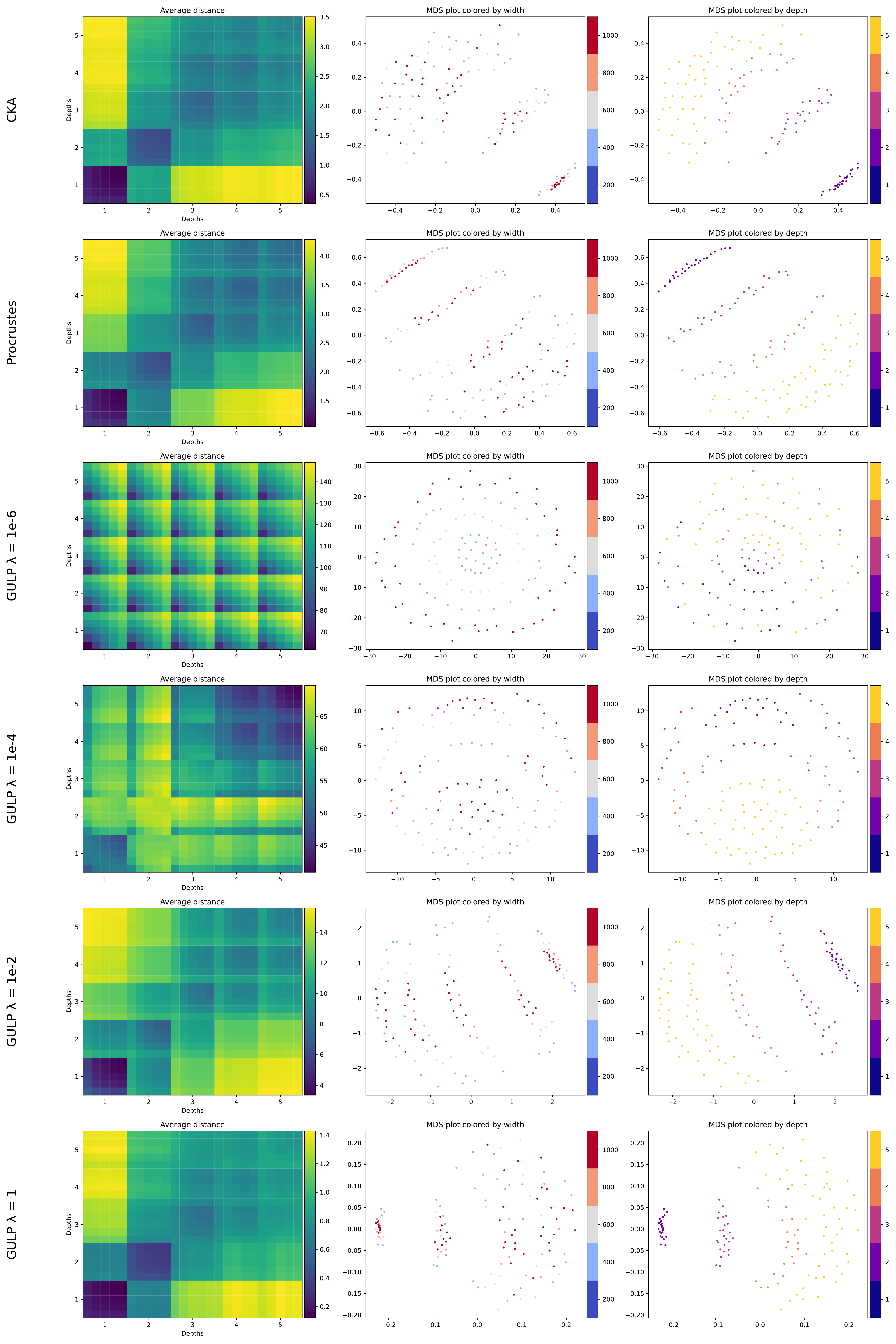}
\caption{Average \cka{}, \procrustes{}, and \gulp{} distance between last hidden layer representations of two fully-connected ReLU networks with a given width and depth (left). Networks are fully-trained on CIFAR and penultimate layer representations are constructed from 10,000 input train images. Ordering of networks along rows and columns of distance matrices has outer indices as network depths 1-5 and inner indices as network widths 200-1000. Two dimensional MDS embedding plots (center and right) of all networks colored by architecture width and depth.}
\label{fig:cifar_width_depth}
\end{figure}

\begin{figure}
\centering
\includegraphics[width=0.7\linewidth]{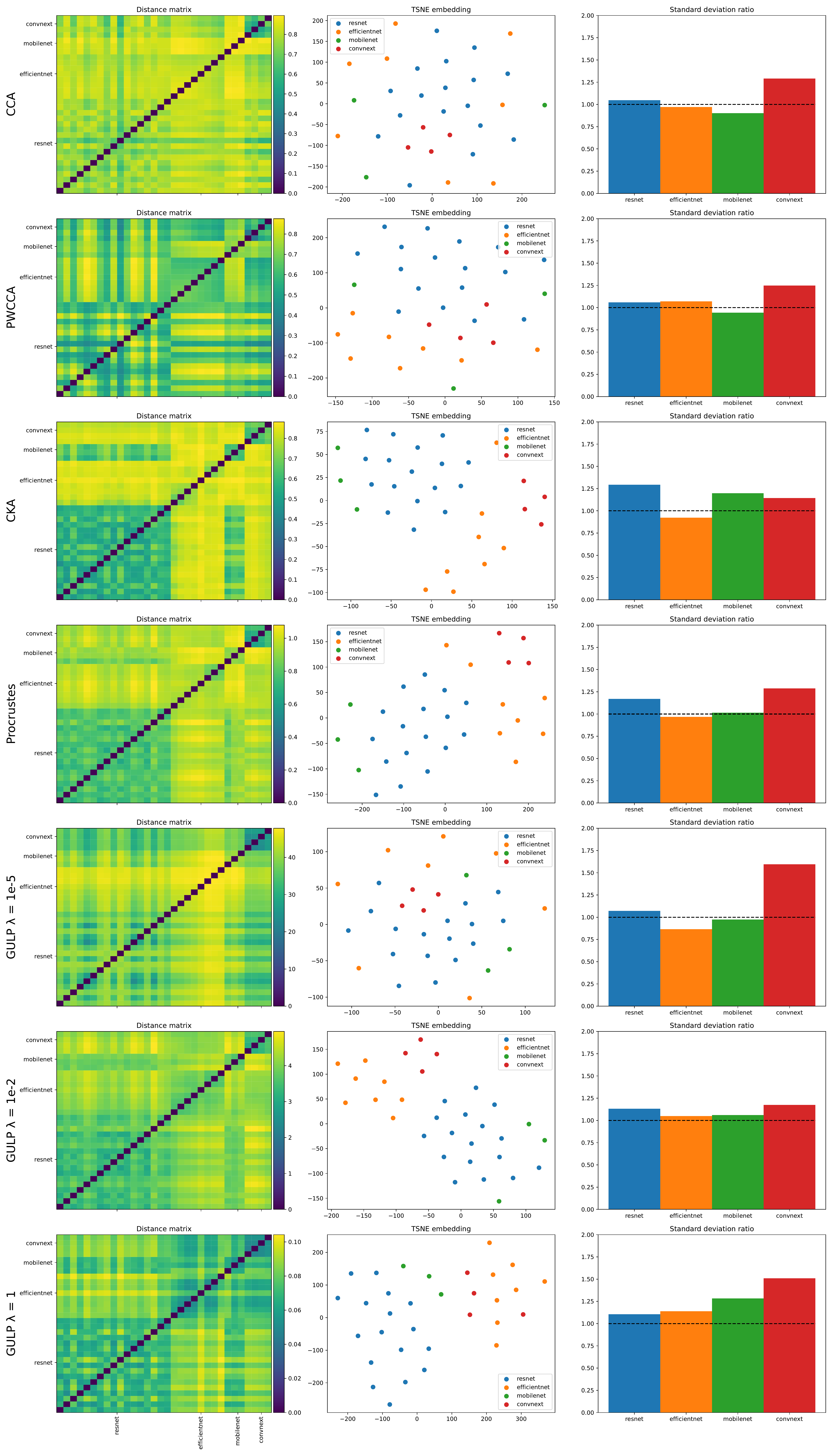}
\caption{\cca{}, \cka{}, \procrustes{}, and \gulp{} distances between last hidden layer representations of 36 pretrained ImageNet networks. Representations are formed by passing 10,000 train images from ImageNet into each network. For five groups of pretrained networks (ResNet, EfficientNet, MobileNet and ConvNeXt), we compute their distance matrices (left) and two-dimensional t-SNE embeddings (center). Separation of the five network groups is quantified by their standard deviation ratios which measure the the standard deviation of the distance across all networks divided by the standard deviation of the distance in a given group. \gulp{}, \cka{}, and \procrustes{} successfully separate all four network types from each other.}
\label{fig:imagenet_pretrained_variance}
\end{figure}

\begin{figure}
\centering
\includegraphics[width=0.7\linewidth]{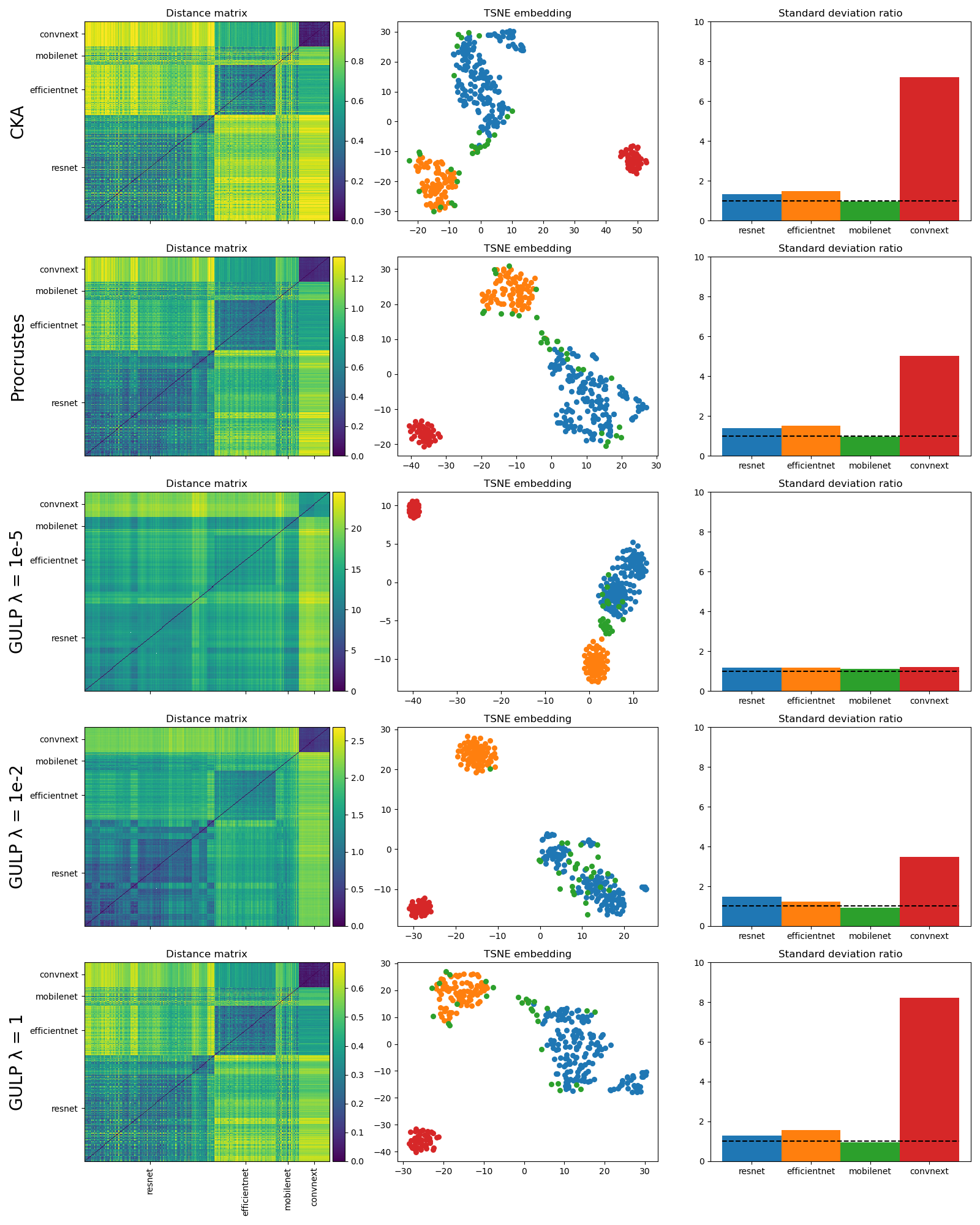}
\caption{\cka{}, \procrustes{}, and \gulp{} distances between penultimate layer representations of 32 untrained ImageNet networks where each network model is randomly intialized 10 times. Representations are formed by passing 10,000 train images from ImageNet into each network. For four groups of pretrained networks (ResNet, EfficientNet, MobileNet, ConvNeXt), we compute their distance matrices (left) and two-dimensional t-SNE embeddings (center). Separation of the four network groups is similarly quantified by their standard deviation ratios which measure the the standard deviation of the distance across all networks divided by the standard deviation of the distance in a given group. Under all distance metrics ResNets, EfficientNets, and ConvNeXts are clustered separately with \cka{} and \gulp{} at $\lambda = 1$ forming the most compact clusters.}
\label{fig:imagenet_untrained_variance}
\end{figure}

\begin{figure}
\centering
\includegraphics[width=\linewidth]{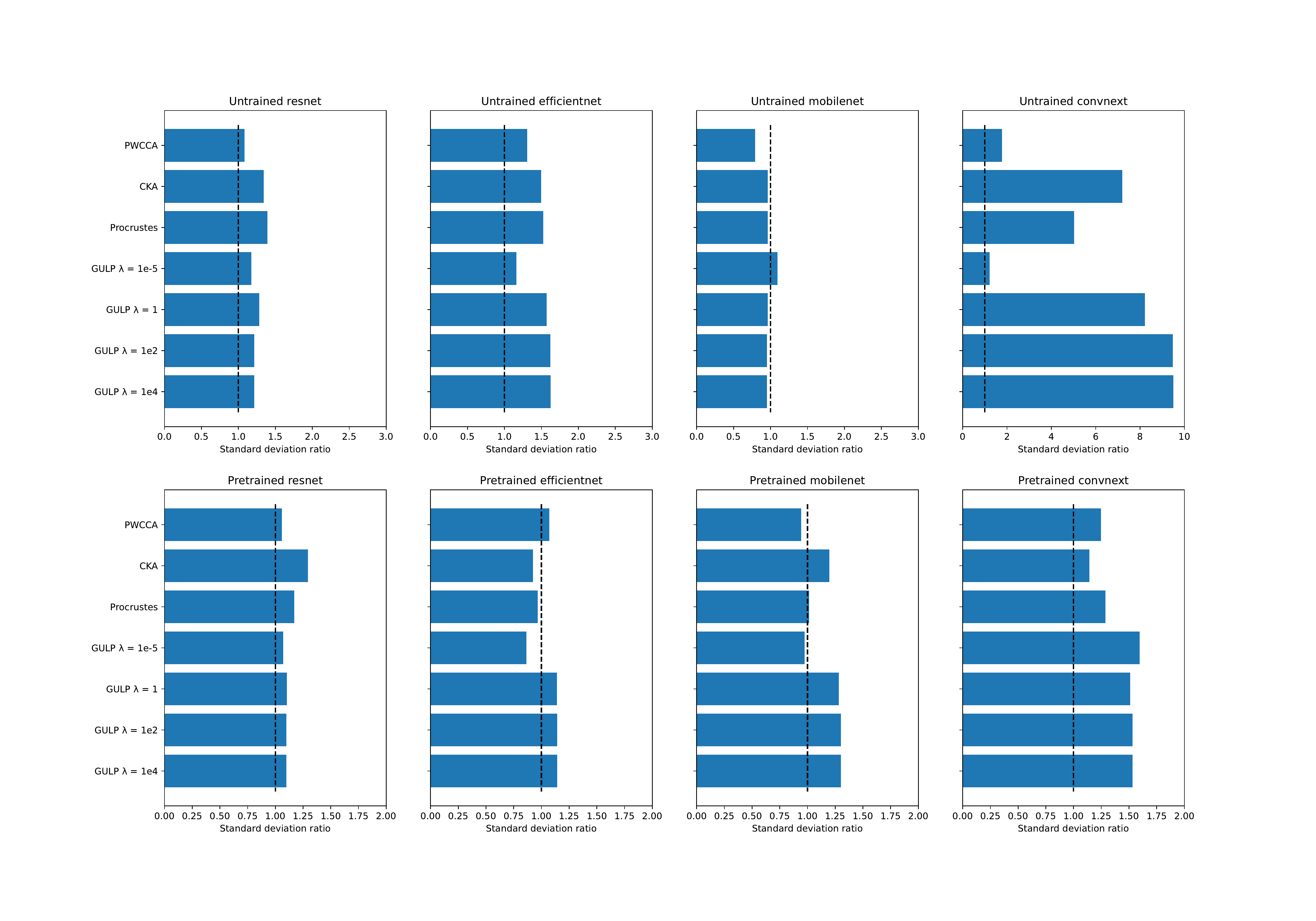}
\caption{Standard deviation ratio of distances for five groups of architectures (ResNet, EfficientNet, MobileNet, and ConvNeXt) both for untrained and pretrained networks.}
\label{fig:imagenet_variance_networks}
\end{figure}

\subsection{\gulp{} does not strongly depend on input data distribution}
Here we test how the \gulp{} distance between network architectures depends on the distribution of the input data $X$ from which the last hidden layer representations are computed. In Figure~\ref{fig:teaser} we showed a t-SNE embedding of the \gulp{} distance ($\lambda=10^{-2}$) between the last hidden layer representations of 37 networks pretrained on ImageNet. These penultimate layer representations were computed by passing 10,000 images from the ImageNet train set into each network. In Figure~\ref{fig:imagenet_mnist} we repeat this experiment and generate a t-SNE embedding of the \gulp{} distance ($\lambda=10^{-2}$) between ImageNet networks where each network is passed in 10,000 images from the MNIST train set. In order to input MNIST grayscale images into these networks, we convert them to RGB images where each channel has a copy of the same image and is centered and normalized as described in Section~\ref{app:experimental_setup}. Even though all 37 networks were trained on the ImageNet train set, \gulp{} is able to separately cluster EfficientNet, ResNet, and ConvNeXt architectures from their last hidden layer representations of MNIST images. In Figure~\ref{fig:imagenet_utkface} we show yet another example of this phenomenon, where \gulp{} properly clusters ImageNet architectures when their last hidden layer representations are constructed from 10,000 face input images taken from the UTKFace train dataset \cite{zhifei2017cvpr}. This shows that in practice the \gulp{} distance consistently captures the same relationships between network architectures and does not strongly depend on the input data distribution used to build the network representations.

\begin{figure}
  \centering
    \includegraphics[width=\textwidth]{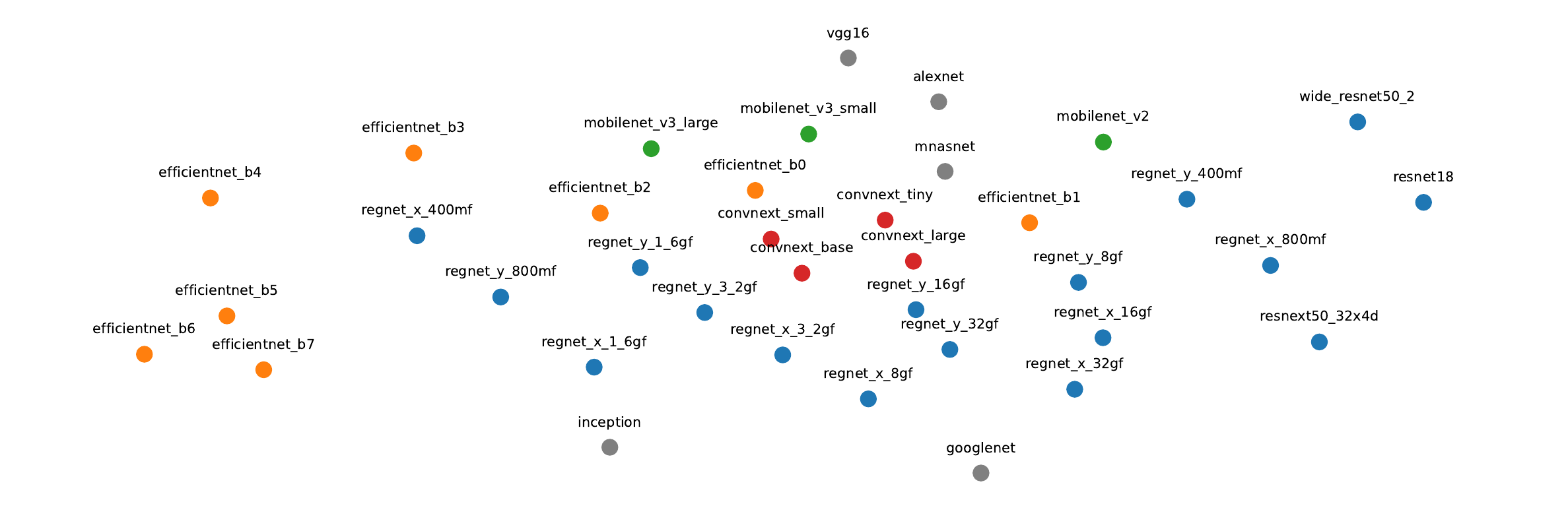}
        \caption{t-SNE embedding of penultimate layer representations of pretrained ImageNet networks with \gulp{} distance ($\lambda=10^{-2}$), colored by architecture type (gray denotes architectures that do not belong to a family). For each network pretrained on ImageNet we input MNIST images and compute their last hidden layer representations. Even though these ImageNet networks were not trained on MNIST data, the \gulp{} distance is able to cluster their penultimate layer representations and consistently forms groups of MobileNet, EfficientNet, ResNet, and ConvNeXt architectures. This indicates that the \gulp{} metric does not depend strongly on the data distribution which networks are trained on.}
    \label{fig:imagenet_mnist}
\end{figure}

\begin{figure}
  \centering
    \includegraphics[width=\textwidth]{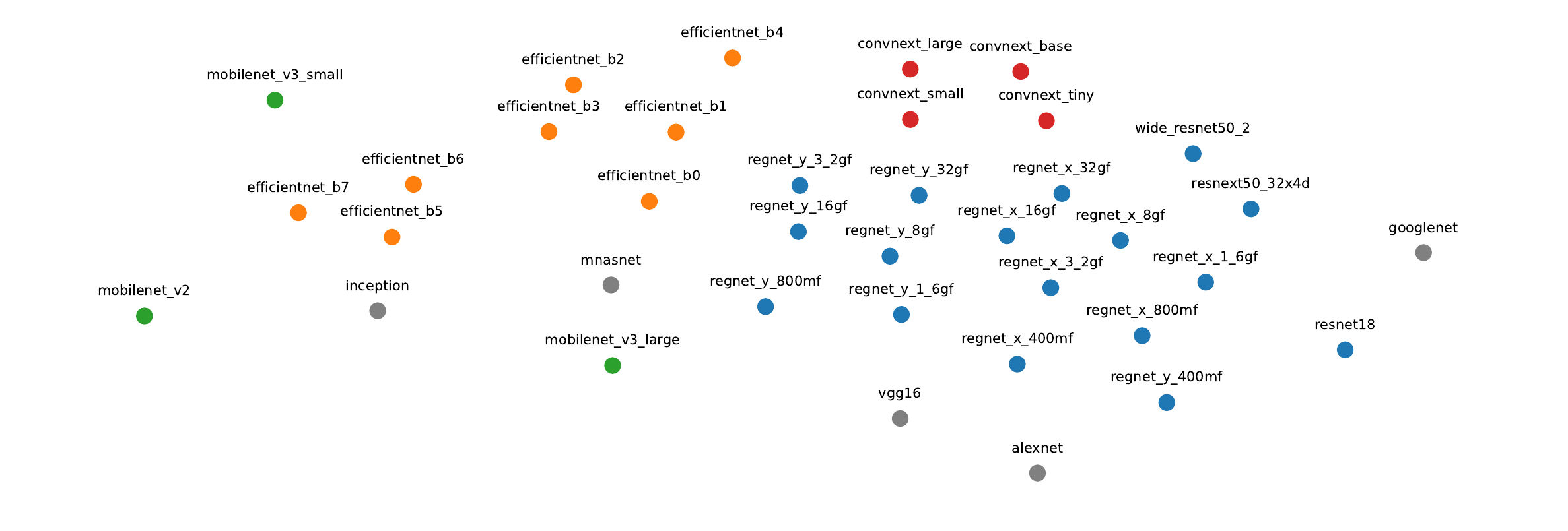}
        \caption{t-SNE embedding of penultimate layer representations of pretrained ImageNet networks with \gulp{} distance ($\lambda=10^{-2}$), colored by architecture type (gray denotes architectures that do not belong to a family). Contrary to Figure~\ref{fig:teaser}, here for each network pretrained on ImageNet we input 10,000 face images from the UTKFace train dataset and compute their last hidden layer representations. Even though these ImageNet networks were not trained on UTKFace data, the \gulp{} distance is able to cluster their last hidden layer representations and consistently forms groups of MobileNet, EfficientNet, ResNet, and ConvNeXt architectures. This in conjunction with Figure~\ref{fig:imagenet_mnist} shows that the \gulp{} metric is not overly sensitive to the input data distribution from which network representations are constructed.}
    \label{fig:imagenet_utkface}
\end{figure}

\subsection{Network representations converge in $\gulp{}$ distance during training}
Here, we repeat Figure \ref{fig:dists_during_training_one_plot}, but plot each distance separately and with a greater variety of regularization values $\lambda$ (see Figure \ref{fig:dists_during_training}).

\begin{figure}
    \centering
    \includegraphics[trim=0 600 0 0, clip, width=0.9\linewidth]{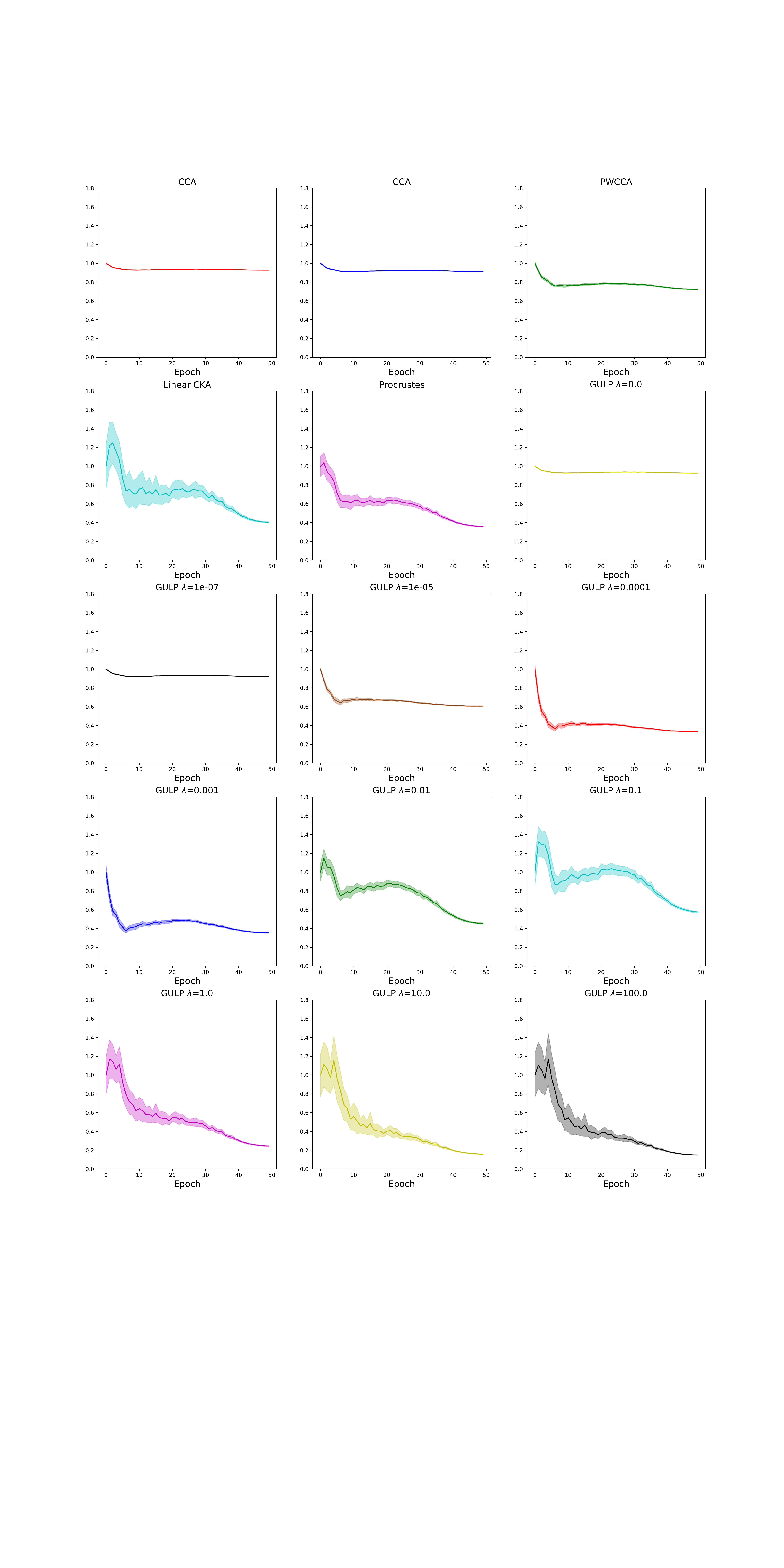}
    \caption{The empirical distances between penultimate layer representations of 16 independently trained ResNet18 architectures during training, computed using $3,000$ samples and averaged over all pairs. Distances are scaled by their average value at iteration $0$ for the sake of comparison between metrics.}
    \label{fig:dists_during_training}
\end{figure}

\subsection{$\gulp{}$ distance at intermediate network layers}
Throughout this paper, we have primarily used \gulp{} to compare neural networks using their last hidden layer representations. Here we study how the \gulp{} distance compares intermediate hidden layers of neural networks. Namely, we take 10 NLP BERT base models from Zhong et al.~\cite{zhong2021larger} which are pretrained with different random initializations on sentences from the Multigenre Natural Language Inference (MNLI) dataset~\cite{williams2017broad}. Each model has 12 hidden layers and we save the representations at every hidden layer on 3,857 MNLI input train samples. In Figure~\ref{fig:bert_layer_embedding} we plot the distance matrices for \gulp{} at varying values of $\lambda$ between every pair of hidden layers across 10 BERT networks. We also plot the tSNE, MDS, and UMAP embeddings with each colored line representing one of the 10 BERT models. In each embedding plot, earlier layers are drawn as points with a dark hue while layers closer to the end of the network are represented by points with a faded color. As expected, for each of the BERT model the \gulp{} distances arrange their hidden layers linearly in order from their input layer to their output layer. When $\lambda$ is small, the earlier layers of all 10 networks are grouped together while the later layers have large \gulp{} distances between all 10 models. As $\lambda$ increases, the later layers of all 10 models also become grouped together and \gulp{} arranges all BERT models linearly in the order of their hidden layers. Therefore, tuning the $\lambda$ parameter in \gulp{} allows us to make distinctions between earlier and later layers of a network architecture.

\begin{figure}
    \centering
    \includegraphics[width=0.9\linewidth]{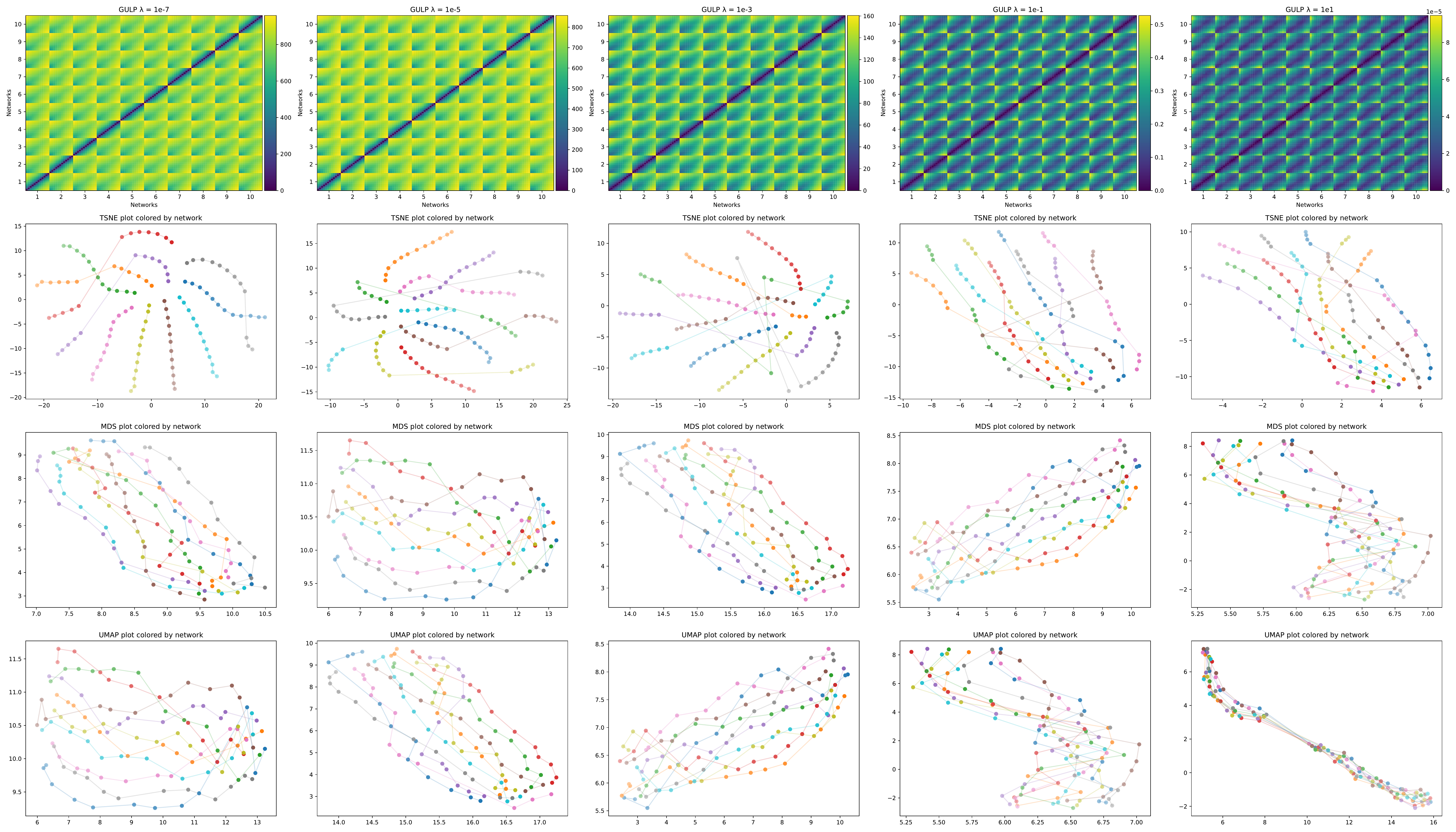}
    \caption{Top row shows \gulp{} distance matrices between 12 hidden layers of 10 fully-trained NLP BERT base models with different random initializations. Representations at every hidden layer are constructed from 3,857 MNLI input train samples which are then used to compute the \gulp{} distance between every pair of layers across the 10 models. Distance matrices are embedded using tSNE, MDS, and UMAP where each colored line represents one of the 10 BERT models. Earlier layers are drawn as dark saturated points while layers close to the output of the network are drawn as faded points. For each of the 10 BERT networks, \gulp{} finds a one-dimensional embedding of its layers which respects their ordering. Across all BERT models, \gulp{} with small $\lambda$ groups together the earlier layers of the 10 network architectures but assigns large distances between the later layers. This is particularly emphasized in the top left tSNE embedding. As $\lambda$ increases, the later layers of all 10 models also become grouped together until all BERT networks are linearly aligned in the order of their hidden layers.}
    \label{fig:bert_layer_embedding}
\end{figure}

\subsection{Specificity versus sensitivity of $\gulp{}$}\label{app:ding-experiments}
Here we run three benchmark experiments of~\cite{ding2021grounding} to compare the sensitivity and specificity of our \gulp{} distance to \cca{}, \pwcca{}, \cka{}, and \procrustes{}.

In the first experiment, we take 10 BERT base models from Zhong et al.~\cite{zhong2021larger} which are pretrained with different random initializations on sentences from the Multigenre Natural Language Inference (MNLI) dataset~\cite{williams2017broad}. All BERT base models have 12 hidden layers of transformer blocks with dimension 768~\cite{devlin2018bert}. For each of the 10 networks, at each of the 12 layers we save the representations on 3,857 MNLI input train samples. We compute the probing accuracies of all 120 representations on the Question-answering Natural Language Inference dataset (QNLI)~\cite{wang2018glue} and the Stanford Sentiment Tree Bank Task (SST-2)~\cite{socher2013recursive}. For a given dataset (QNLI and SST-2), we find the representation $X^* \in \mathbb{R}^{768 \times 3857}$ which has the best probing accuracy and we compare the accuracies of all 120 representations to it. For every representation $X \in \mathbb{R}^{768 \times 3857}$, the difference in probing accuracy from the best representation $X^*$ is correlated with the distance between between the two representions $d(X, X^*)$ under a given distance metric (\cca{}, \cka{}, \procrustes{}, etc.). In Figure~\ref{fig:ding_layer_exp} we display Spearman's $\rho$ and Kendall's $\tau$ rank correlations of the \cca{}, \pwcca{}, \procrustes{}, \cka{}, and \gulp{} distances against the probing accuracy differences between two representations. On the QNLI dataset we see in Figure~\ref{fig:ding_layer_exp} (left) that \gulp{} with large $\lambda$ outperforms all other metrics including \cka{} and achieves the largest rank correlations with statistically significant $p$-values that are below 0.05. Similar results are obtained on the SST-2 dataset as seen in Figure~\ref{fig:ding_layer_exp} (right). This shows that the \gulp{} distance with large $\lambda$ has better specificity (is less sensitive) to random initializations of a network as this has less of an effect on its correlation with probing accuracy compared to the other metrics.

In the second experiment, we study 50 BERT base models from McCoy et al.~\cite{mccoy2019berts} which are trained on MNLI and finetuned for classification with different finetuning seeds at initialization. Similar to the experiment above, we compute 600 representations of the 50 BERT models at each of the 12 layers using 3,857 MNLI input train samples. We are interested in studying how distances between these representations correlate with their out-of-distribution (OOD) performance on a different task. Namely, as our measure of OOD performance we compute each representation's accuracy on the “Lexical Heuristic (Non-entailment)'' subset of the HANS dataset~\cite{mccoy2019right}. As before, we choose the best representation $X^*$ with the lowest OOD accuracy. Then for every representation $X$ the difference in OOD accuracy from the best representation $X^*$ is correlated with the distance between between the two representions $d(X, X^*)$ under a given distance metric. Spearman's $\rho$ and Kendall's $\tau$ rank correlations of the \cca{}, \pwcca{}, \procrustes{}, \cka{}, and \gulp{} distances are shown in Figure~\ref{fig:ding_feather}. Note that \cca{}, \pwcca{}, \procrustes{}, and \gulp{} with small $\lambda$ have the largest correlation with OOD accuracy. Since the BERT model representations were constructed on in-distribution MNLI data, this implies that these distance metrics can detect differences between OOD accuracy of different models without access to OOD data.

Lastly, for the third experiment we study 100 BERT medium models taken from Zhong et al.~\cite{zhong2021larger} which are fully-trained on the MNLI dataset with 10 pretraining seeds and further finetuned on MNLI with 10 different finetuning seeds by Ding et al.~\cite{ding2021grounding}. Each BERT medium model has 8 hidden layers of width 512~\cite{devlin2018bert}. We study the OOD accuracy of these models on the antonymy stress test and the numerical stress test defined in Naik et al.~\cite{naik2018stress}. As with the previous experiments, we compute 800 representations of the 100 BERT models at each of the 8 layers using 3,857 MNLI input train samples. For every representation $X$ the difference in OOD accuracy from the best representation $X^*$ is correlated with the distance between between the two representions $d(X, X^*)$ under a given distance metric. Spearman's $\rho$ and Kendall's $\tau$ rank correlations of the \cca{}, \pwcca{}, \procrustes{}, \cka{}, and \gulp{} distances are shown in Figure~\ref{fig:ding_pretrain_finetune}. As shown in the original experiments by Ding et al.~\cite{ding2021grounding}, none of the distance metrics show a large rank correlation with the OOD accuracy for either of the stress tests and the associated $p$-values are not significant at the 0.05 level except for \gulp{} with $\lambda > 10^{-2}$.

In summary, these benchmark experiments show that the \gulp{} distance exhibits specificity (is not sensitive) to random initializations of a network as shown in Figure~\ref{fig:ding_layer_exp} and this become particularly apparent at large $\lambda$. Additionally, it is sensitive to the out-of-distribution accuracy of a model as supported by Figure~\ref{fig:ding_feather} where it improves upon the performance of \cca{}, \pwcca{}, and \procrustes{}.

\begin{figure}
\centering
\begin{subfigure}[b]{.45\linewidth}
\includegraphics[width=\linewidth]{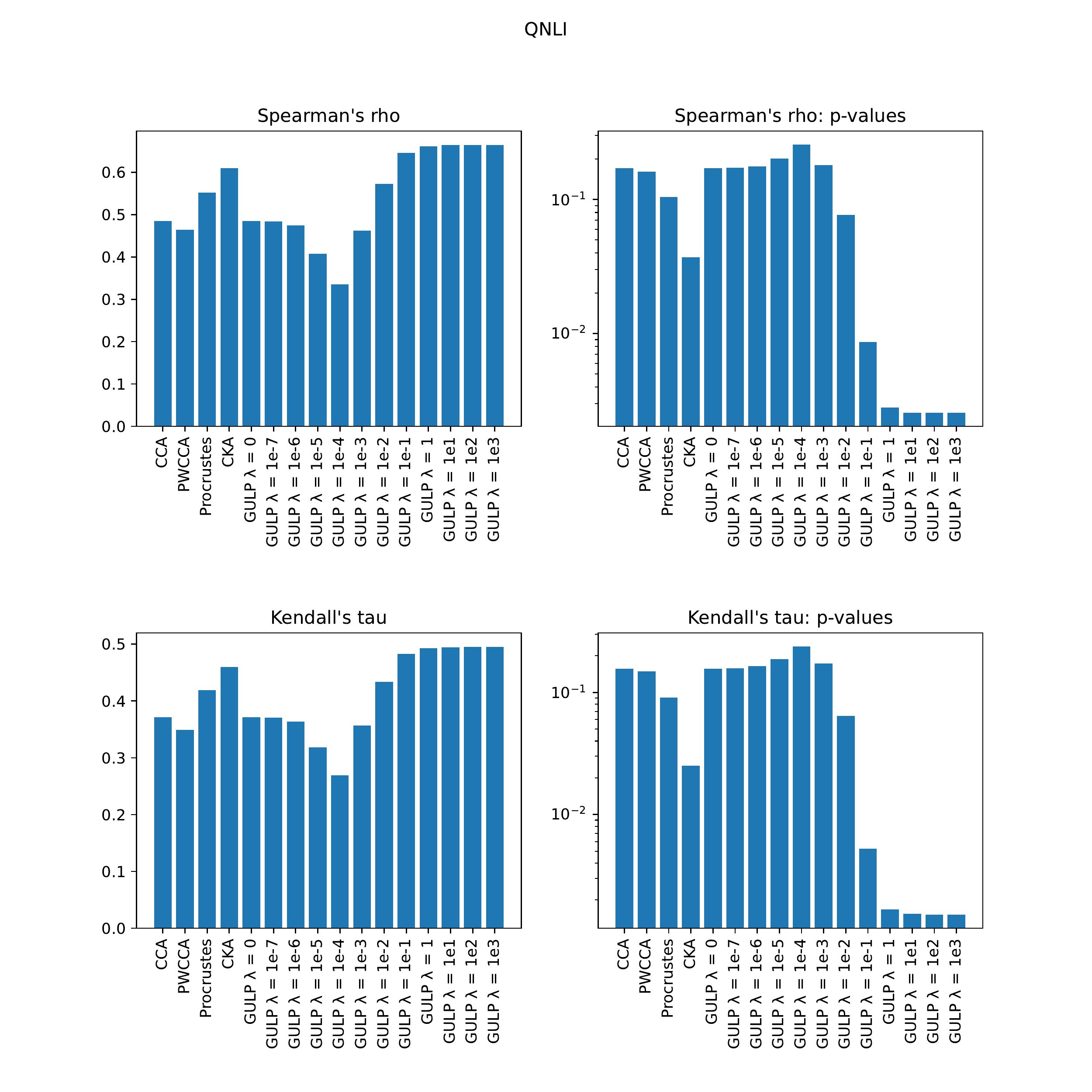}
\end{subfigure}
\begin{subfigure}[b]{.45\linewidth}
\includegraphics[width=\linewidth]{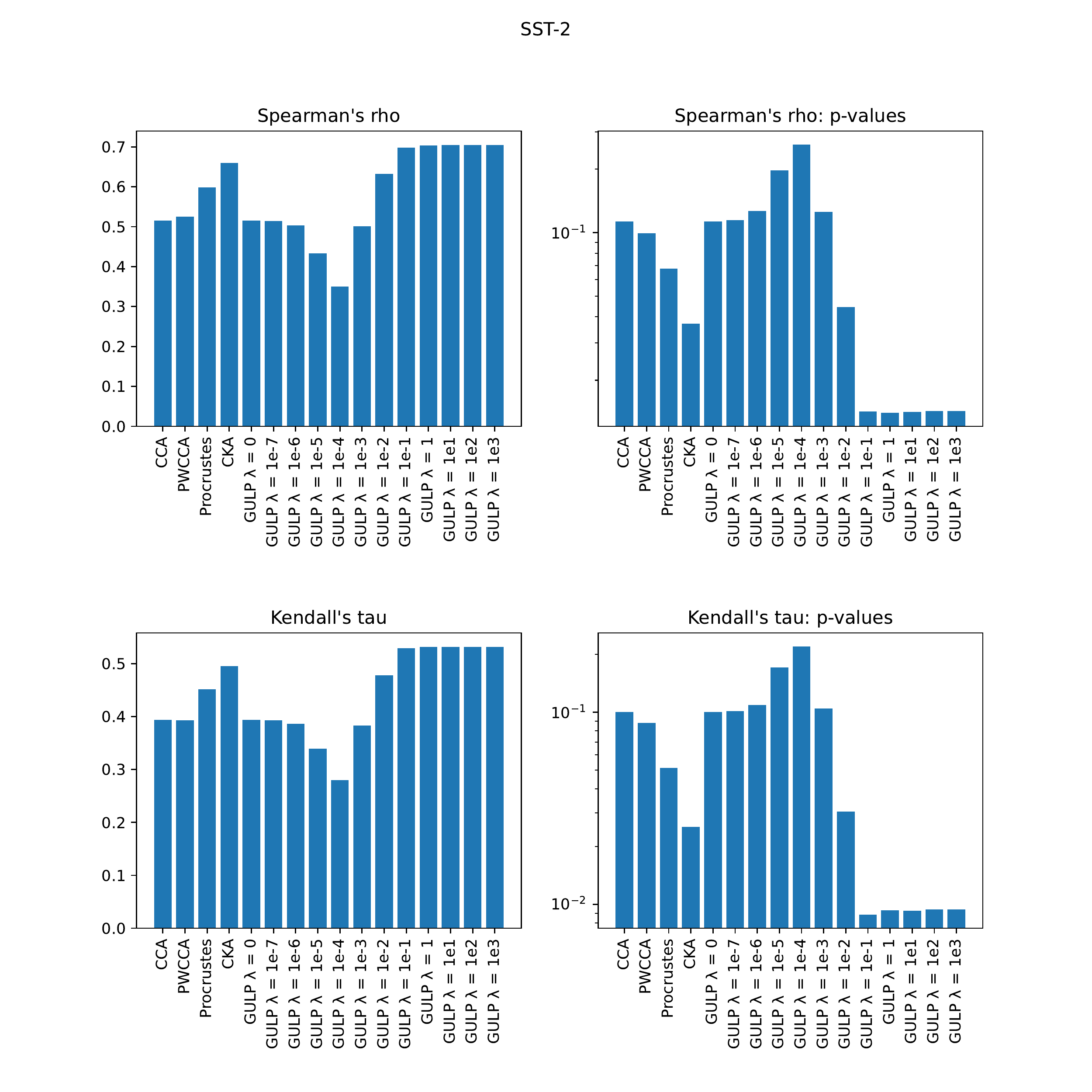}
\end{subfigure}
\caption{Spearman's $\rho$ and Kendall's $\tau$ rank correlations and associated $p$-values for difference of probing accuracy between two representations vs. distance between two representations. Representations are constructed from 12 layers of 10 BERT base models using 3,857 MNLI input train samples. Rank correlations are computed with probing accuracy on the QNLI and SST-2 datasets (left and right).}
\label{fig:ding_layer_exp}
\end{figure}

\begin{figure}
\centering
\includegraphics[width=\linewidth]{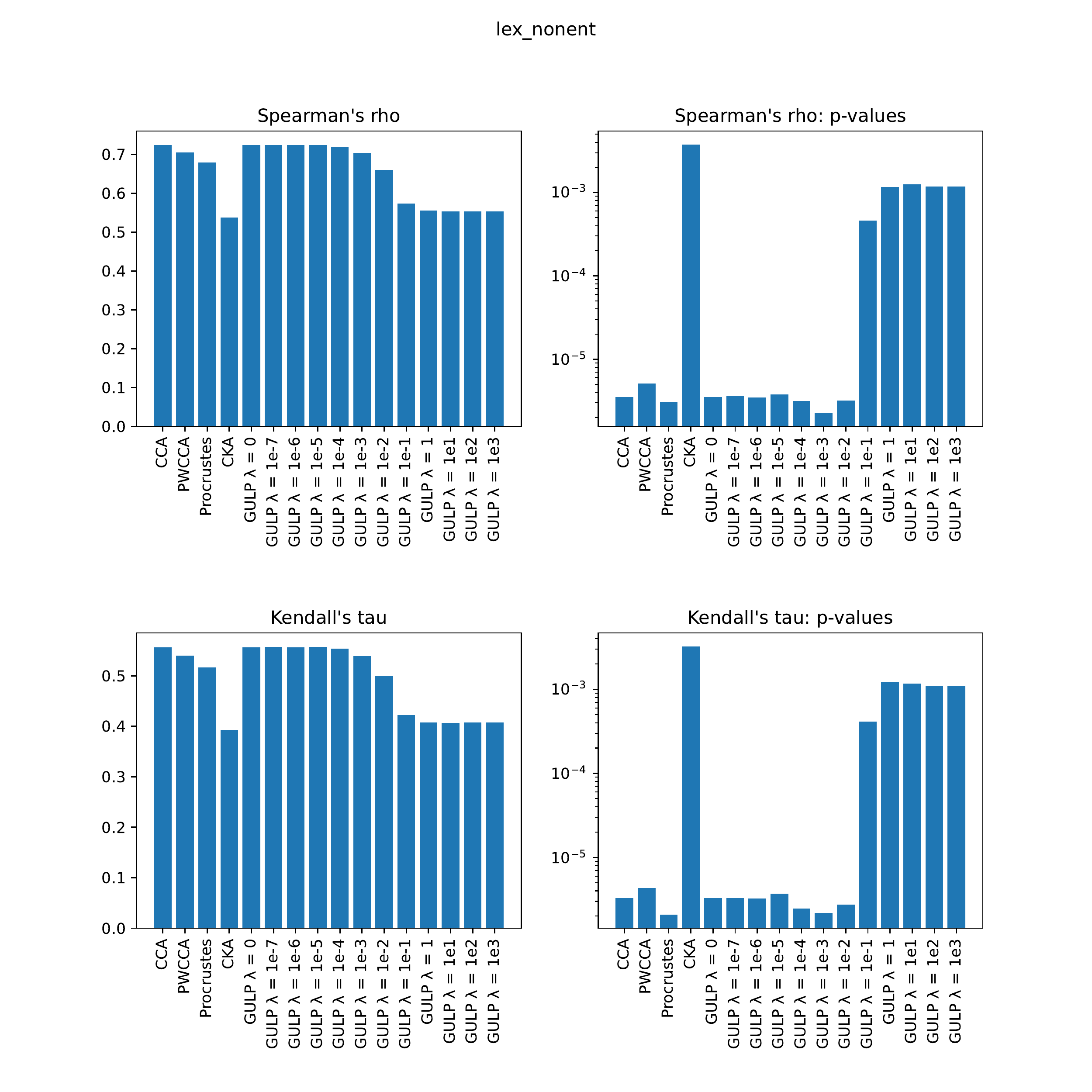}
\caption{Spearman's $\rho$ and Kendall's $\tau$ rank correlations and associated $p$-values for difference of OOD accuracy between two representations vs. distance between two representations. Representations are constructed from 12 layers of 50 BERT base models using 3,857 MNLI input train samples. The BERT base models are finetuned for classification and the OOD accuracy is computed on the “Lexical Heuristic (Non-entailment)'' subset of the HANS dataset.}
\label{fig:ding_feather}
\end{figure}

\begin{figure}
\centering
\begin{subfigure}[b]{.45\linewidth}
\includegraphics[width=\linewidth]{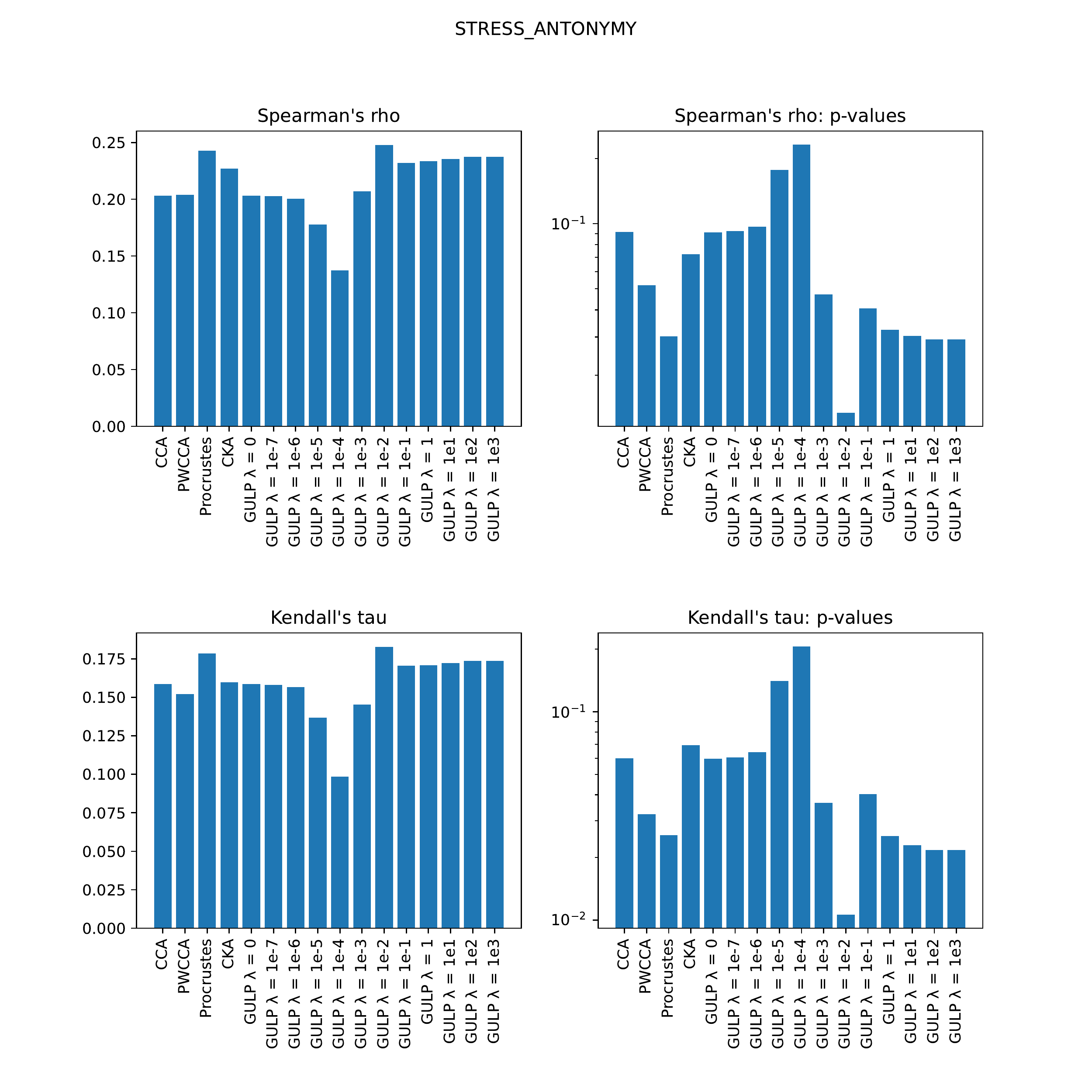}
\end{subfigure}
\begin{subfigure}[b]{.45\linewidth}
\includegraphics[width=\linewidth]{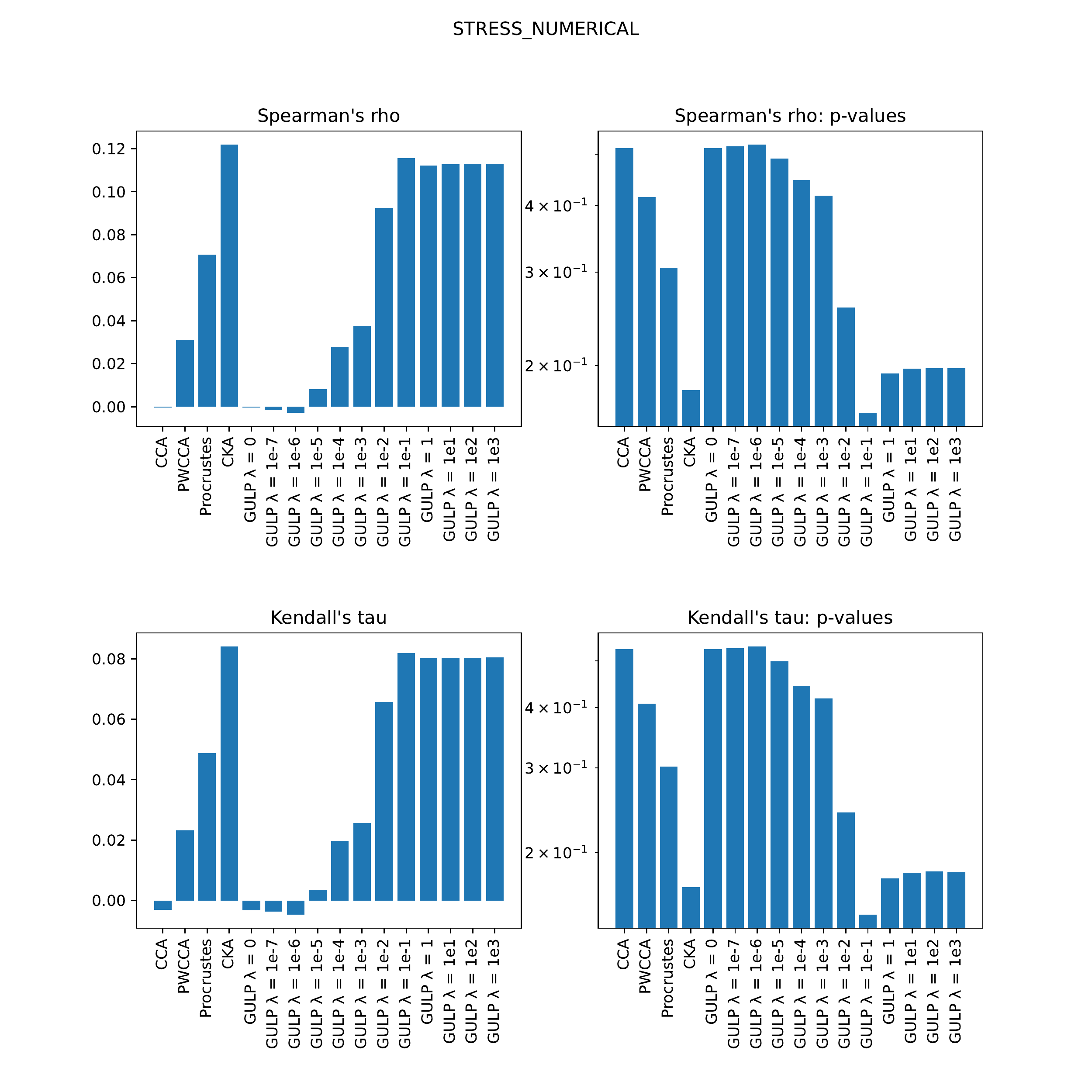}
\end{subfigure}
\caption{Spearman's $\rho$ and Kendall's $\tau$ rank correlations and associated $p$-values for difference of OOD accuracy between two representations vs. distance between two representations. Representations are constructed from 8 layers of 100 BERT medium models using 3,857 MNLI input train samples. The BERT base models are trained from a combination of 10 pretraining and 10 finetuning seeds and the OOD accuracy of each model is measured on the antonymy stress and the numerical stress tests.}
\label{fig:ding_pretrain_finetune}
\end{figure}

\subsection{\gulp{} distances do \emph{not} especially capture generalization on logistic regression}\label{app:logistic}
In this section, we provide Figure~\ref{fig:generalization_log_reg}, which replicates the experiment of Figure \ref{fig:generalization}, but where the downstream transfer learning task is binary logistic regression instead of ridge regression. We assign labels of $0$ and $1$ with equal probability, and compute the resultant test prediction accuracy averaged over $3000$ samples. We find (perhaps unsurprisingly) that \gulp{}, as defined for ridge regression, does not capture downstream generalization better than baselines on logistic regression tasks. This motivates the extension of \gulp{} to logistic regression in future work.

\begin{figure}
\centering
    \includegraphics[width=.5\textwidth]{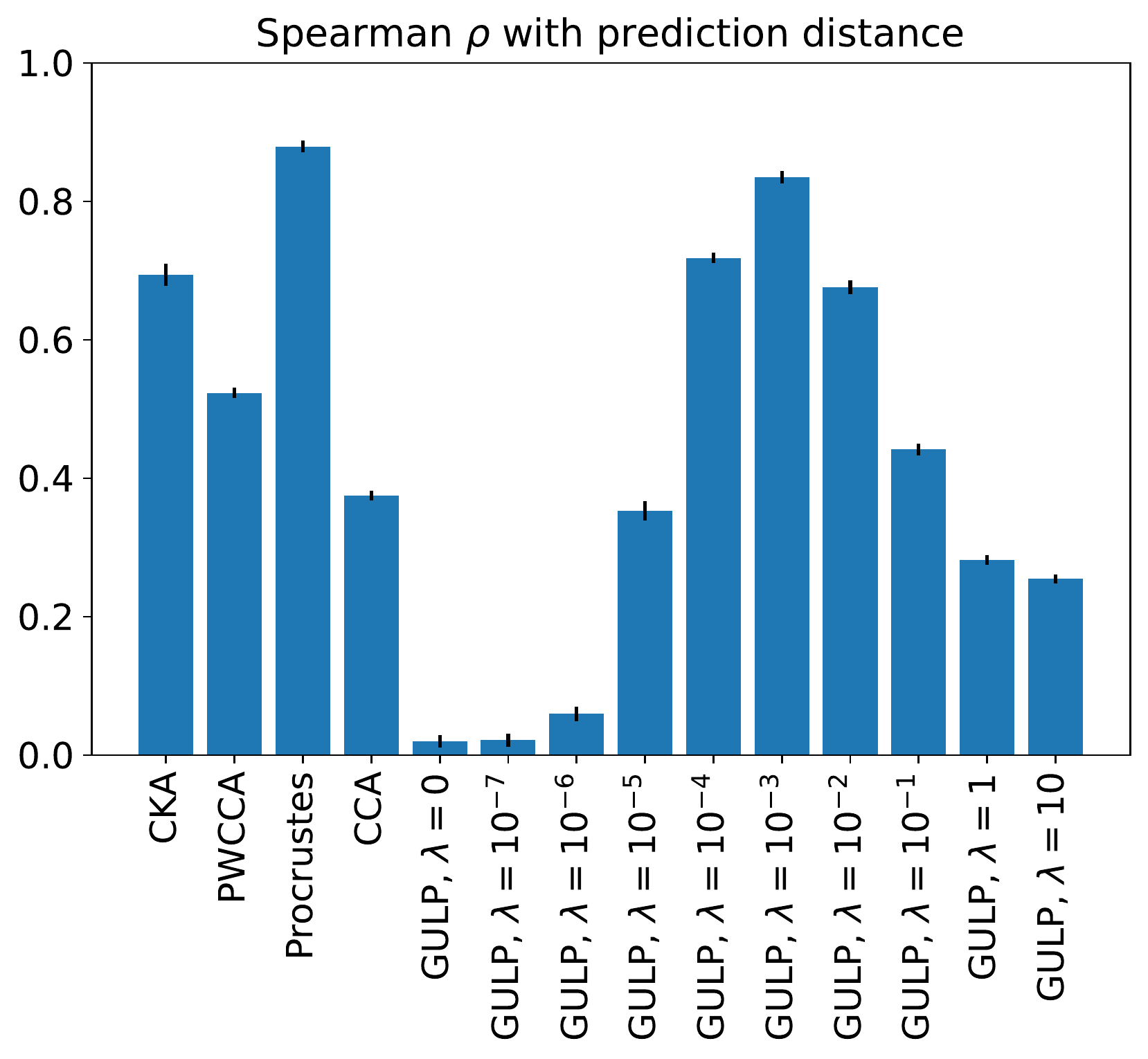}
    \caption{$\gulp$ does not capture generalization of the predictors output by logistic regression. We plot Spearman's $\rho$ between the differences in predictions by $\lambda$-regularized linear regression, and the different distances. Results are averaged over 10 trials.}
    \label{fig:generalization_log_reg}
\end{figure}

\end{document}